\documentclass{article}
\usepackage[title]{appendix} %appendix for article
\usepackage[utf8]{inputenc}
\usepackage{comment}
\usepackage{caption} %for two-line caption
\usepackage[square,sort,comma,numbers]{natbib}
\usepackage{graphicx}
\usepackage{amsmath}
\usepackage{amssymb}
\usepackage{balance}
\usepackage{indentfirst}
\usepackage[T1]{fontenc}
\usepackage{geometry}
\usepackage[normalem]{ulem}
\usepackage{tikz}
\usepackage{ytableau}
\usepackage{amsfonts}
\usepackage{tikz}
\usepackage{scrextend}
\usepackage{color,soul}
\usepackage{xcolor}
\usepackage{float}
\usepackage{comment}
\usetikzlibrary{matrix,calc, fit,backgrounds,positioning}
\usepackage{algorithm} %algorithm
\usepackage{algpseudocode} %algorithm
\usepackage{amsthm}
\newtheorem{theorem}{Theorem}
\newtheorem{corollary}[theorem]{Corollary}
\newtheorem{lemma}[theorem]{Lemma}
\newtheorem{definition}{Definition}
\newtheorem{proposition}[theorem]{Proposition}

\usepackage{bm}   %bold math symbol

\definecolor{bblue}{HTML}{4F81BD}
\definecolor{rred}{HTML}{C0504D}
\definecolor{ggreen}{HTML}{9BBB59}
\definecolor{ppurple}{HTML}{9F4C7C}

\usepackage{authblk}

\usepackage{pgf} %display JPEG images

\usepackage{soul} %strike through text

\title{A Centralized Planning and Distributed Execution Method for Shape Filling with Homogeneous Mobile Robots
}

\author{Shuqing Liu, Rong Su, Karl H.Johansson
\thanks{Shuqing Liu (email: shuqing001@e.ntu.edu.sg) and Rong Su (email: rsu@ntu.edu.sg) are with School of Electrical and Electronic Engineering, 
Nanyang Technological University, Singapore. Karl H.Johansson (email: kallej@kth.se) is with Division of Decision and Control Systems, KTH Royal Institute of Technology, Stockholm, Sweden.
}
}

\date{} % leave empty
\begin{document} % goes here
\maketitle
\begin{abstract}
    \noindent
    \textbf{Abstract:}
    The pattern formation task is commonly seen in a multi-robot system. In this paper, we study the problem of forming complex shapes with functionally limited mobile robots, which have to rely on other robots to precisely locate themselves. The goal is to decide whether a given shape can be filled by a given set of robots; in case the answer is yes, to complete a shape formation process as fast as possible with a minimum amount of communication. Traditional approaches either require global coordinates for each robot or are prone to failure when attempting to form complex shapes beyond the capability of given approaches - the latter calls for a decision procedure that can tell whether a target shape can be formed before the actual shape-forming process starts. In this paper, we develop a method that does not require global coordinate information during the execution process and can effectively decide whether it is feasible to form the desired shape. The latter is achieved via a planning procedure that is capable of handling a variety of complex shapes, in particular, those with holes, and assigning a simple piece of scheduling information to each robot, facilitating subsequent distributed execution, which does not rely on the coordinates of all robots but only those of neighboring ones. The effectiveness of our shape-forming approach is vividly illustrated in several simulation case studies. 
    
\end{abstract}

%Not dividable
%\newcommand{\ndiv}{\hspace{-4pt}\not|\hspace{2pt}}
%Underline
%\setul{0.7ex}{0.2ex} %{space}{thickness}
%\setulcolor{red}

\section{Introduction}
Multi-Robot Systems (MRS), as suggested by the name, involves the study of groups of heterogeneous or homogeneous autonomous robots tackling a wide range of tasks with varying characteristics in the same environment \cite{madridano_trajectory_2021}. Among the various tasks assigned to a group of autonomous robots, the deployment task has received significant attention since it is usually used to set up the initial configuration for other downstream tasks such as exploration, surveillance and cooperative manipulation. The pattern formation task, as a variant of the robot deployment tasks, aims to assemble a collection of robots into proper positions that resemble a specific pattern when viewed globally \cite{bayindir_review_2015}. Multiple research studies have been performed on the pattern formation using homogeneous autonomous robots.

In \cite{stoy_how_2006}, Stoy and Kasper proposed a shape formation algorithm utilizing a porous scaffold and established two types of gradients: the wander gradient, which is generated from wandering robots outside the desired shape, and the hole gradient, which is generated from the unfilled spaces inside the desired shape. A wandering robot descends the hole gradient and ascends the wander gradient, resulting in the robots inside the shape moving away from the boundary and propagating the unfilled spaces toward the boundary. In such a way, this algorithm is able to reconfigure a random shape into a desired non-hollow shape.

In \cite{rubenstein_programmable_2014}, the K-team at Harvard University proposed the additive self-assembly algorithm. The algorithm relies on the gradient value established by a set of stationary robots to navigate moving robots. The moving robots rotate clockwise around the perimeter of the group of robots and stop inside the desired shape in layers. In a follow-up study, Gauci et al. proposed a self-disassembly method that uses light sources to guide robots residing outside the desired shape to detach from the configuration \cite{gros_programmable_2018}. In a later study \cite{wang_self-organizing_2021}, Wang and Zhang utilize multiple groups for the formation task. By executing the additive self-assembly algorithm concurrently, they effectively increased the coverage speed of the formation algorithm. One of the critical issues of the additive self-assembly algorithm, as pointed out by the technical review \cite{niazi_technical_2014}, is that the algorithm can only form simply connected shapes without holes.

In \cite{yang_distributed_2019}, Yang et al. proposed a formation algorithm based on the square lattice and motion-chains. The robots on the outermost layer separate from the shape to form motion-chains and assemble inside the desired shape. Two motion-chains (moving in the clockwise and counter-clockwise directions, respectively) meet in the middle of the desired shape. This mechanism speeds up the assembly process by parallelism. In a later study \cite{yang_distributed_2020}, Yang's team proposed an extension of the algorithm to allow the formation of hollow shapes by introducing additional mechanisms on the motion-chains. The motion-chain approach enables the manipulability of muti-robot movements at the cost of scalability, that is, when robots move as a chain, localization error propagates as they move further from the starting position. This accumulated error can lead to misalignment when the motion chain attempts to form a hole.

In \cite{derakhshandeh_universal_2016}, Derakhshandeh et al. designed a shape formation algorithm based on the triangular expansion behavior of programmable matters. The algorithm can be adapted to any complex shape composed of equilateral triangles. However, similar to Yang's approach, without global information on coordinates, the robots may not be able to form the holes correctly.

Current methods either lack the ability to form complex shapes involving holes or require global coordinates for accurate hole forming. In this article, we propose a pattern formation method, the add-subtract algorithm, to handle formation shapes with holes. In our algorithm, information on global coordinates is not required by each robot, as they could locate themselves based on the position of neighboring robots. The algorithm is inspired by the ribbon idea introduced in the additive self-assembly algorithm \cite{rubenstein_programmable_2014}. In our algorithm, we assume a grid-based workspace and choose the hexagonal lattice due to its compactness and suitability for disk-shaped robots. Several studies have explored hexagonal lattice configurations. K. Tomita et al. developed a nucleation method using hexagon-shaped robots \cite{tomita_self-assembly_1999}; M. De Rosa et al. defined hole motion in a hexagonal lattice for modular robot shape sculpting \cite{de_rosa_scalable_2006}; and M. Rubenstein et al. utilized hexagonal arrangements to organize idle robots \cite{rubenstein_programmable_2014}. Building on this foundation, we propose a systematic method to partition the lattice points inside the desired shape into layers (referred to as ribbons) and plan the movement sequences for each individual robot. During the execution phase, robots first occupy all the lattice points, including the holes, layer by layer, and then the extra robots move out of the holes in reverse order (starting from the outermost layer). Compared with other existing methods, our method bears the following distinct features:
\begin{enumerate}
	\item the planning procedure can assist users to decide whether a complex shape, in particular, those with holes, can be formed by robots with limited localization capabilities; and
	\item in the case the given shape is feasible, a concrete moving sequence will be generated, and each robot will be assigned with some simple piece of information that does not rely on the coordinates of all robots but only those of neighboring ones, thus facilitating distributed execution; and
	\item the distributed execution phase can be further enhanced with fault tolerance capabilities via fast re-planning, thanks to our effective planning procedure, making the method robust to uncertainties in implementation.
\end{enumerate}

The remainder of the paper is organized as follows. First, the problem is introduced and formulated in section \ref{sect-problemDef}. The add-subtract algorithm is introduced and analyzed in section \ref{sect-AddSubAlgorithm}. To ensure a smooth presentation flow, while we present the outline of the proof of correctness in section \ref{sect-outlinedProof}, details of all proofs are presented in the Appendices. Simulation results are given in section \ref{sect-simulation}. Conclusions are drawn in section \ref{sect-conclusion} together with a discussion of future research directions. For the pseudocode of the algorithm, please refer to appendix \ref{AppendiceA}. For the proof of the properties of the ribbon structure and the proof of correctness of the algorithm can be found in appendix \ref{sect-Ribbonization} and appendix \ref{sect-methodAndProof}.

% followed by the proof of correctness\footnote{Details of the proof can be found in appendix \ref{sect-Ribbonization} and appendix \ref{sect-methodAndProof}} in section \ref{sect-outlinedProof}. 
\section{Problem Formulation}
\label{sect-problemDef}

In this section, we give a formal definition of the problem (section \ref{probDescription}) and introduce the setup of the system (section \ref{sect-workspace}). Finally, we introduce the concept of epoch used in the algorithm (section \ref{subsect-modelingOfTime}).

\subsection{Problem Description}
\label{probDescription}
In this paper, we study the problem of filling up a given shape with a given collection of small mobile robots. The shape may contain holes, and each mobile robot has a limited localization capability, forcing them to rely on neighboring robots to precisely determine its current location. Due to this limitation, during the shape-filling process, robots must follow a certain releasing sequence to ensure that each newly released robot can rely on previously released and positioned robots to precisely locate itself, before moving into a designated position in the shape. We assume that each robot has sufficient memory to store several pieces of critical information, such as the releasing sequence, which encoded the boundaries of the shape and internal holes. In addition, we assume that each robot can only communicate with robots within a neighborhood whose maximum radius is known in advance. Next, we provide more details of the system setup.

\bigskip

\noindent Given a set of a sufficient number of functionally identical (or homogeneous) robots, the following setup is assumed:
\begin{itemize}
    \item [-] each robot has a tight circular convex hull with a physical radius $r$;
    \item [-] each robot has a communication range $l$ that allows it to communicate with at least three other robots compactly positioned on the hexagonal lattice. In this paper, we set $l > \frac{4r}{\sqrt{3}-1}$;
    \item [-] each robot can measure its distance to other robots within the communication range $l$;
    \item [-] each robot's motion is holonomic.
    % \item [-] the user-specified shape $S$ defined in Euclidean coordinate system in $\mathbb{R}^2$ which is topologically homeomorphic to an open disk in $\mathbb{R}^2$.

\end{itemize}

We consider a user-specified shape $S_0$ defined in the Euclidean coordinate system in $\mathbb{R}^2$ which is topologically homeomorphic to an open disk in $\mathbb{R}^2$. We consider a set $D_0$ of user-specified Holes defined in the Euclidean coordinate system in $\mathbb{R}^2$, each of which is topologically homeomorphic to a collection of closed disks in $\mathbb{R}^2$ and contained inside $S_0$. With a slight abuse of notation, we use $D_0$ to denote the union of all such holes. Thus, $D_0$ is a subset of $S_0$.

\begin{figure}[H]
    \centering
    \includegraphics[scale=0.4]{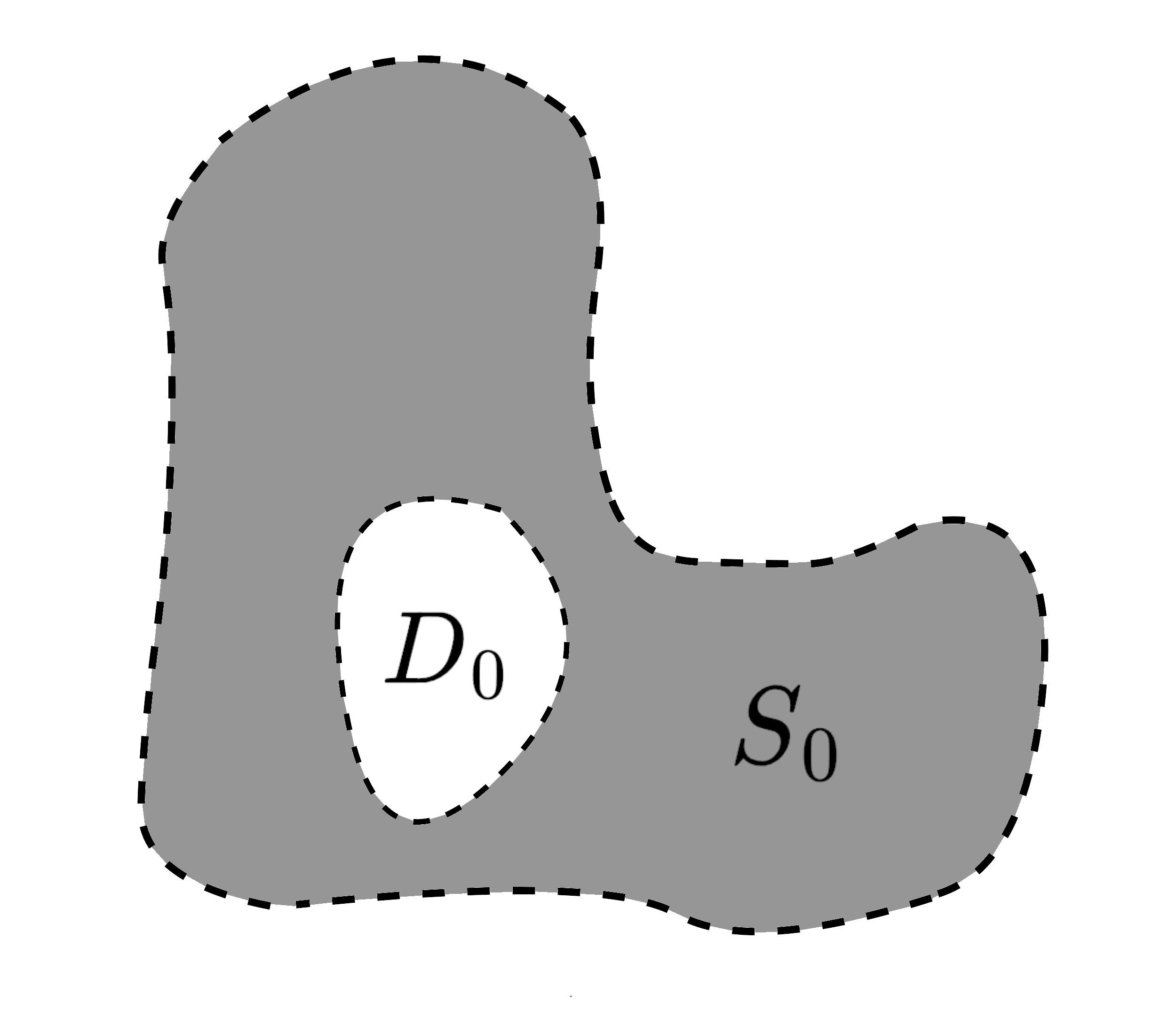}
    \includegraphics[scale=0.45]{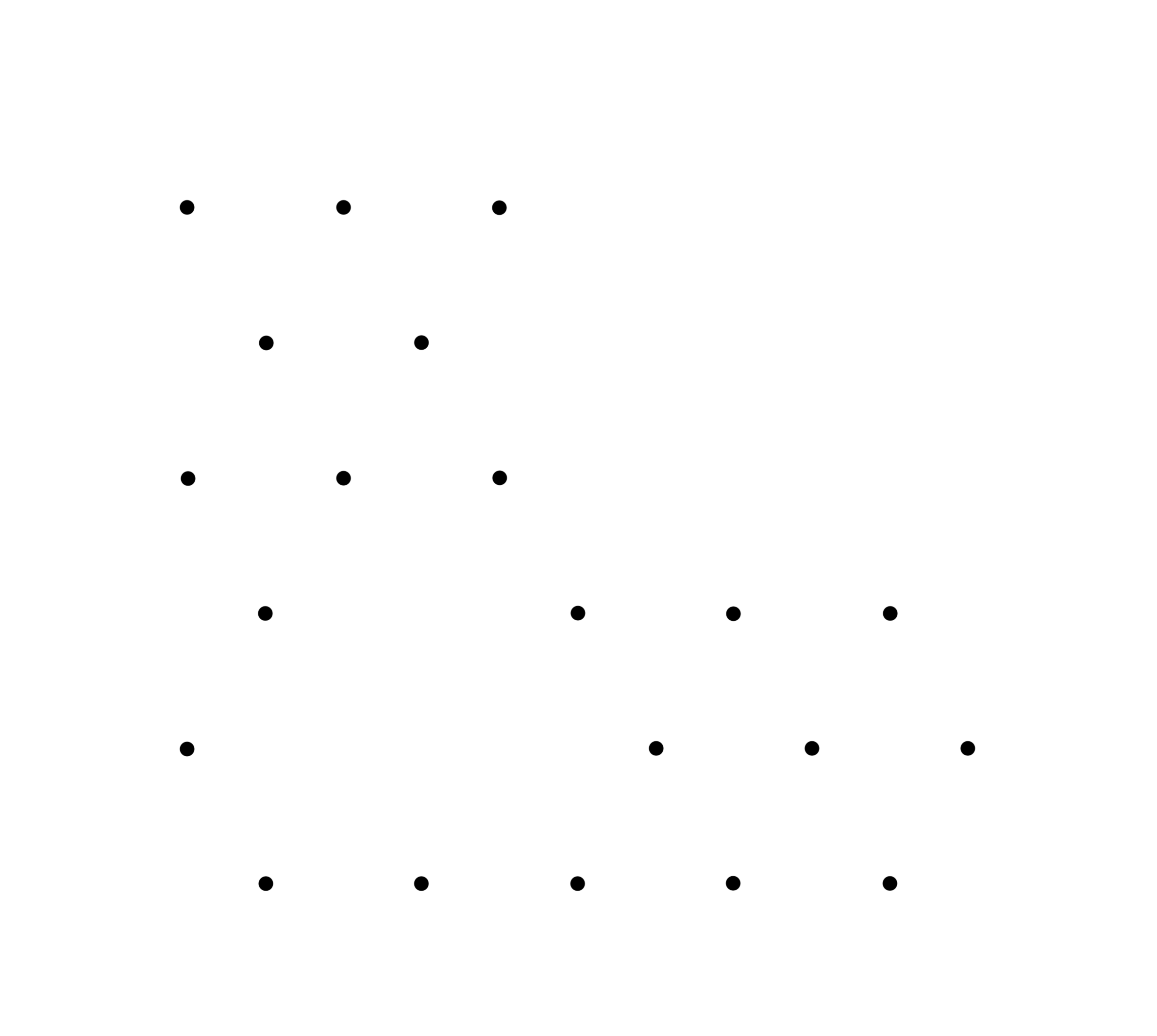}
    \caption{Left: example of a user-defined shape $S$ and the collection of holes $D$; Right: a proper lattice point representation of $S \setminus D$}
\end{figure}

We assume an underlying hexagonal lattice $L$. $L$ consists of all the points in $\mathbb{R}^2$ that are $d = \frac{2r}{\sqrt{3}-1}$ from each other\footnote{The $d$ is chosen such that the gap between each pair of robots located on hexagonal lattice point is large enough for another robot to pass through.}, and each point is surrounded by six equally spaced neighbors. After superimposing with the hexagonal lattice, we obtain a set $ S\setminus D$ as the intersection of $S_0 \setminus D_0$ and $L$. By rotating and translating $S_0 \setminus D_0$, we pick the $S\setminus D$ with the largest cardinality as the discrete representation of the shape $S_0 \setminus D_0$. The corresponding discrete representation $D := D_0 \cap L$ is the set of hexagonal lattice points in $D_0$ subjected to the same rotation and translation. Similarly, $S := S_0 \cap L$. We say a lattice point $x \in S$ is on the boundary of $S$ if $x$ is adjacent to a lattice point not in $S$. 

For a proper lattice representation $S \setminus D$, we require the following:
\begin{enumerate}
    \item all the lattice points in $S$ are connected (in the sense that there is a path between any pair of lattice points in $S$ through adjacent lattice points in $S$); and
    \item $S \setminus D$ correctly captures the holes $D$ such that every lattice point on the boundary of $S$ is not in $D$. 
    \item $S$ has a smooth boundary in the sense that
        \[ (\forall x,y \in S) (||x-y||=2d \Rightarrow ((\forall z \in L) ||x-z||=||y-z||=d \Rightarrow z \in S ))\] 
    and
        \[ (\forall x,y \in S) (||x-y||=\sqrt{3}d \Rightarrow \neg((\forall z \in L) ||x-z||=||y-z||=d \Rightarrow z \notin S ))\]
\end{enumerate}

\medskip

% \noindent \underline{\textbf{Problem Statement:}} Given a sufficient number of robots (with the properties previously defined) and a user-defined shape $S$ and a set $D$ of holes (with the attributes previously specified), design a distributed algorithm deployable on each robot that can move a collection of robots (denoted as $K$) to fill the shape $S \setminus D$ such that:
% \begin{enumerate}
%     \item for all lattice point $x$ in $S\setminus D$, there is a robot in $K$ with its center at $x$;
%     \item there is no robot in $K$ with its center not in $S$;
%     \item there is no robot in $K$ with its center in $D$;
%     % \item there is no point inside $S\setminus D$ that is more than or equal to $2r$ from all robot in $K$, i.e., there is no empty space left in $S\setminus D$ after shape filling, which is big enough to hold a robot.
% \end{enumerate}

\noindent \underline{\textbf{Problem Statement:}} Given a sufficient number of robots(with an initial set of robots whose position coordinates are known in advance) and a user-defined shape $S \setminus D$ (with the attributes previously specified), suppose the collection of robots initially occupies the set of lattice points $I$, determine a movement sequence of tuples $(a_1,b_1,1), \ldots, (a_M, b_M, M)$ where $a_k \in L$ and $b_k \in L$ denotes the starting and stopping positions of the $k$th moving robot in the tuple $(a_k,b_k,k)$, such that:
\begin{enumerate}
    % \item before the movement of any robot, the robots centered at points in $B$ can locate their position.
    \item the final collection of robots forms the desired shape in the workspace (defined by a half-plane) in the sense that $\{\{b_1,\ldots, b_M\} \circleddash \{a_1,\ldots, a_M\}\} \sqcap H  = S \setminus D$ where $\{a_1,\ldots, a_M\}$ and $ \{b_1,\ldots, b_M\}$ are multisets, and $H$ is the set of all lattice points of some half-plane.
    \item For each prefix subsequence $(a_1, b_1, 1), \ldots, (a_{k}, b_{k},k)$, $1 \leq k \leq M$, let $A_{k} = \{a_1, \ldots, a_{k-1}\}$, $B_{k} = \{b_1, \ldots, b_{k-1}\}$, and $\overline{B_{k}} \subseteq B_{k}$ be the set of stopping positions that have been reached by the corresponding moving robots. Then the starting position of the $k$th robot $a_k \in I \oplus \overline{B_{k}} \circleddash A_{k}$ and the robot can only relay on robots centered at $I \oplus \overline{B_{k}} \circleddash A_{k} \circleddash \{a_k\}$ to locate its position before stopping at $b_k$. 
    % Please note that $\overline{T_{k}} \subseteq T_{k}$ is the set of stopping lattice positions that has been reached by the corresponding moving robot.  
\end{enumerate}

\subsection{Setup of the System}
\label{sect-workspace}
    We assume the underlying hexagonal lattice of the workspace as $L$. $L$ consists of all the points in $\mathbb{R}^2$ that are $d = \frac{2r}{\sqrt{3}-1}$ from each other, and each point is surrounded by six equally spaced neighbors.
    The hexagonal lattice $L$ induces a graph $\Gamma =(L,E)$ such that for any $x,y \in L$, there is an edge $xy \in E$ iff $||x-y||=d$. Two lattice points $x,y$ are \textbf{adjacent} to each other iff $||x-y||=d$. The \textbf{neighborhood} of a lattice point $x \in L$ is the set of lattice points in $S$ that are adjacent to $x$, and each lattice point has a maximum of six neighbors in its neighborhood.

    \begin{figure}[H]
        \centering
        \includegraphics[scale=0.3]{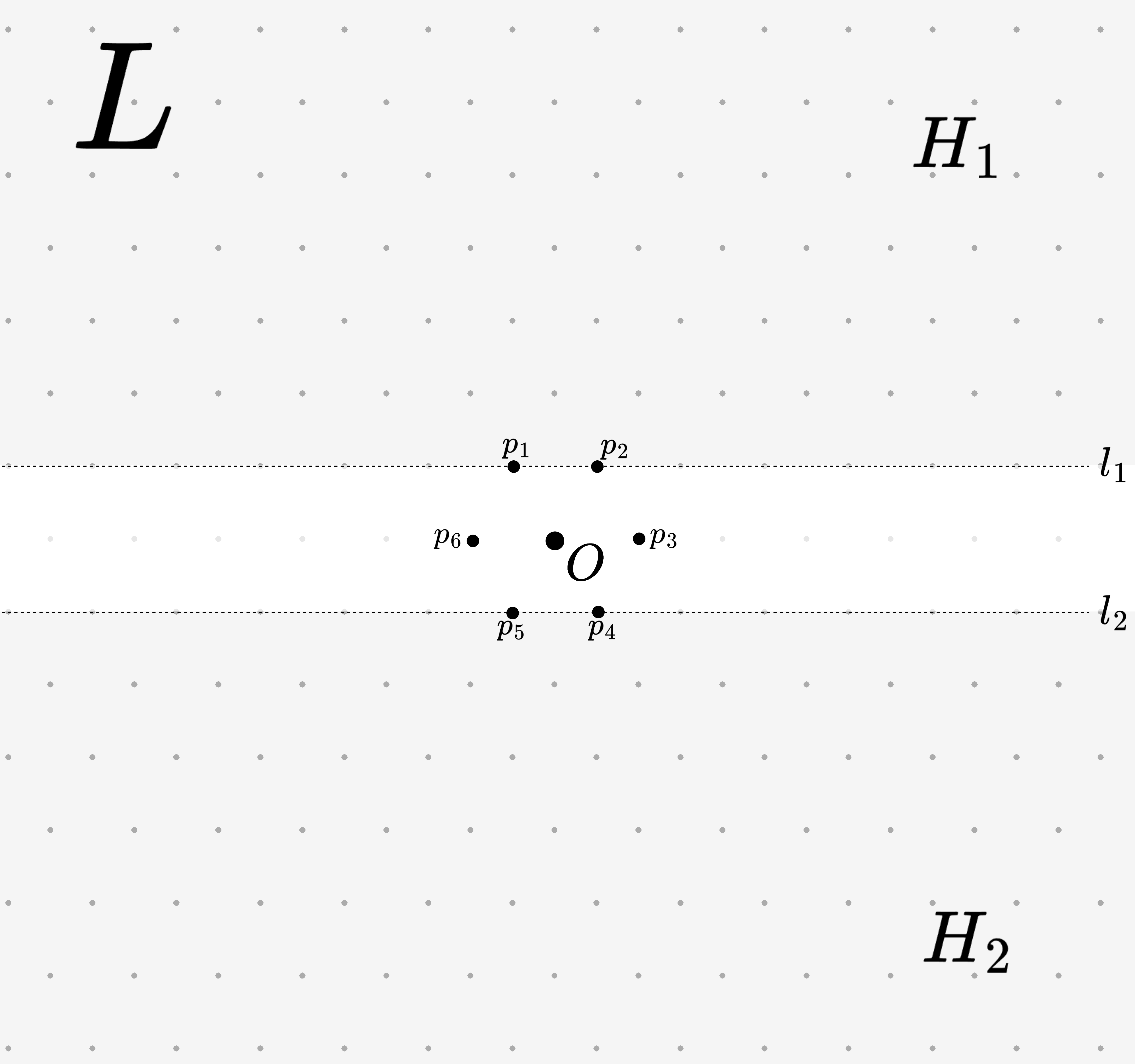}
        \caption{Illustration of the workplace with two half-plane and underlying hexagonal lattices}
    \end{figure}

    We further establish a coordinate system $\mathbb{O}$ by picking an arbitrary point $O \in L$ as the origin and denote the lattice points adjacent to the origin as $p_1, p_2,...,p_6$, labeled clockwise. We denote the line going through $p_1,p_2$ as $l_1$ and the line going through $p_4, p_5$ as $l_2$. Notice that $l_1$ divides $\mathbb{R}^2$ into two open half-planes\footnote{An open half-plane does not include the line that induces the half-plane.}. One of the open half-planes does not contain $l_2$, and we denote it as $H_1$. Similarly, the open half-plane induced by $l_2$ that does not contain $l_1$ is denoted as $H_2$. The idle robots are to be placed in $H_2$ and we call $H_2$ the \textbf{idle half-plane}. The robots will form the desired shape in $H_1$ and we call $H_1$ the \textbf{formation half-plane}. The buffering region located between $l_1$ and $l_2$ ensures no physical interaction between idle robots and the robots in $S$. 

    The finite collection of all robots is denoted by $U$. For each individual robot $u \in U$, its center at time $t$ in the coordinate system $\mathbb{O}$ is denoted as $p_t(u)$. Two robots $u,v$ are \textbf{neighboring} each other at time $t$ if their centers $p_t(u)$ and $p_t(v)$ are within distance $d$ from each other. The set of neighbors of $u$ at time $t$ is denoted as $N_t(u)$. At each sub-epoch $T = n^-$, the center of each inactive robot $u \in U$ is denoted as $p_{n^-}(u)$. The center of each robot $u$ at each sub-epoch $T = n^+$ is denoted as $p_{n^+}(u)$.
    
    A lattice point $x \in L$ is \textbf{occupied} by a robot $u$ at time $t$ if $p_t(u)=x$; otherwise, it is \textbf{unoccupied}. If a vertex of ribbon $R$ is occupied by a robot $u$, we say the robot $u$ is \textbf{located on} the ribbon $R$. Given a ribbon $R$, we say the ribbon is \textbf{filled} if all the vertices of $R$ are occupied; the ribbon is \textbf{empty} if no vertex of $R$ is occupied; a ribbon is \textbf{half-filled} if it is neither filled nor empty. We say a robot belongs to a ribbon if the robot occupies one of the vertices of the ribbon. The robot occupying the leader vertex of a ribbon is the \textbf{leader robot} of the ribbon.

\subsection{The Concept of Epoch}
\label{subsect-modelingOfTime}
    
    For each robot $u$ of the multi-robot system (MRS), we adopt the conventional epoch model and use $T$ to denote a particular epoch. During the shape-filling process, by assuming that at any time there is at most one moving robot, the time steps for the system can be partitioned into a finite number of periods (or epochs), where in each epoch there is only one robot moves, except for the initial epoch where no robot moves. We define the \textbf{epoch} $T$ as follows:
    \begin{itemize}
        \item[-] $T=0$ corresponds to the epoch up to the time instant before the first robot transiting into movement; and
        \item[-] $T=n$ corresponds to the epoch between the moment when the $n$th robot starts to move and the moment just before the $(n+1)$th robot starts to move. For now, we assume that at the end of the $n$th epoch, the $n$th robot has already moved into a position in the shape $S\setminus D$.   
    \end{itemize}
    We denote the $n$th moving robot as $u_n$, which is referred to as the \textbf{active robot} of the epoch $T=n$, and the rest of the robots as the \textbf{inactive robots} of the epoch $T=n$.

    \noindent Each epoch $T=n$ is further divided into two sub-epochs: 
    \begin{itemize}
        \item [-] $T=n^-$ denotes the sub-epoch form the moment when $u_n$ starts to move till the moment when $u_n$ stops; and
        \item [-] $T=n^+$ denotes the sub-epoch form the moment when $u_n$ stops till the moment just before $u_{n+1}$ starts to move.
    \end{itemize}
    The active robot $u_n$ is the only \textbf{moving} robot of the sub-epoch $T=n^-$, and the remaining robots are \textbf{non-moving} robots. There is no moving robot in the sub-epoch $T=n^+$, and all the robots are non-moving robots. For the case when $T=0$, there is only one sub-epoch $T=0^+$ where all the robots are in their initial configuration and no robot is moving. 

     \begin{figure}[H]
        \centering
        \includegraphics[scale=1]{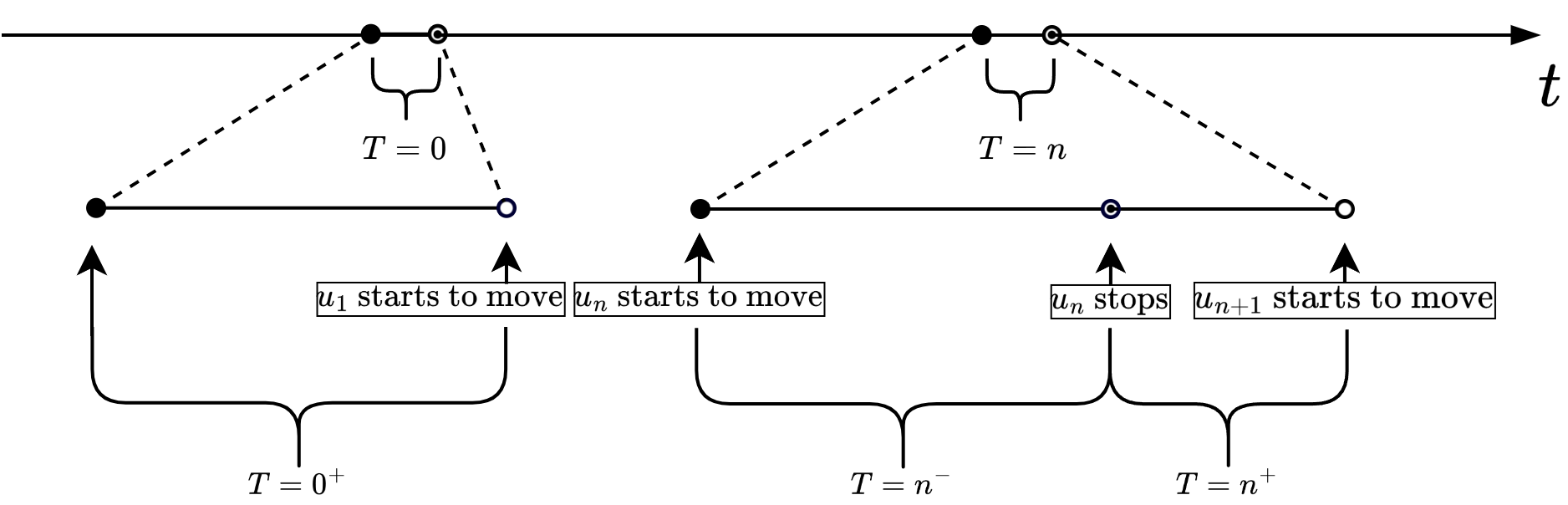}
        \caption{Illustration of epoch, sub-epoch}
    \end{figure}

\section{The Add-subtract Algorithm}
\label{sect-AddSubAlgorithm}

We are given sufficient idle robots in the idle half-plane $H_2$, and we place the idle robots in the following manner: 
\begin{enumerate}
    \item  There exists a collection of five robots $B = \{v_0, v_1, v_2, v_3, v_4\}$ as a subset of $U$, where $p_0(v_0) = O$ and $p_0(v_1) = p_1, p_0(v_2) = p_2, p_0(v_3)=p_4, p_0(v_4)=p_5$. The five robots form a collection of seed robots that remain stationary throughout the shape-filling (or assembly) process. Each seed robot knows its coordinates in $O$.
    
    \item The idle robots are placed on the hexagonal lattice $L$ within the idle half-plane $H_2$. We also require the idle robots to be placed adjacent to the seed robots layer by layer in the sense that 
        \begin{enumerate}
            \item [-] every idle ribbon is either empty or filled;
            \item [-] if an idle ribbon $IR$ is filled, then all idle ribbons closer to the root robot are filled.
        \end{enumerate} 
\end{enumerate}

The algorithm is designed to form the shape $S$ in the formation half-plane. We start by placing the $S$ in the formation half-plane near the origin.

\begin{figure}[H]
    \centering
    \includegraphics[scale=0.8]{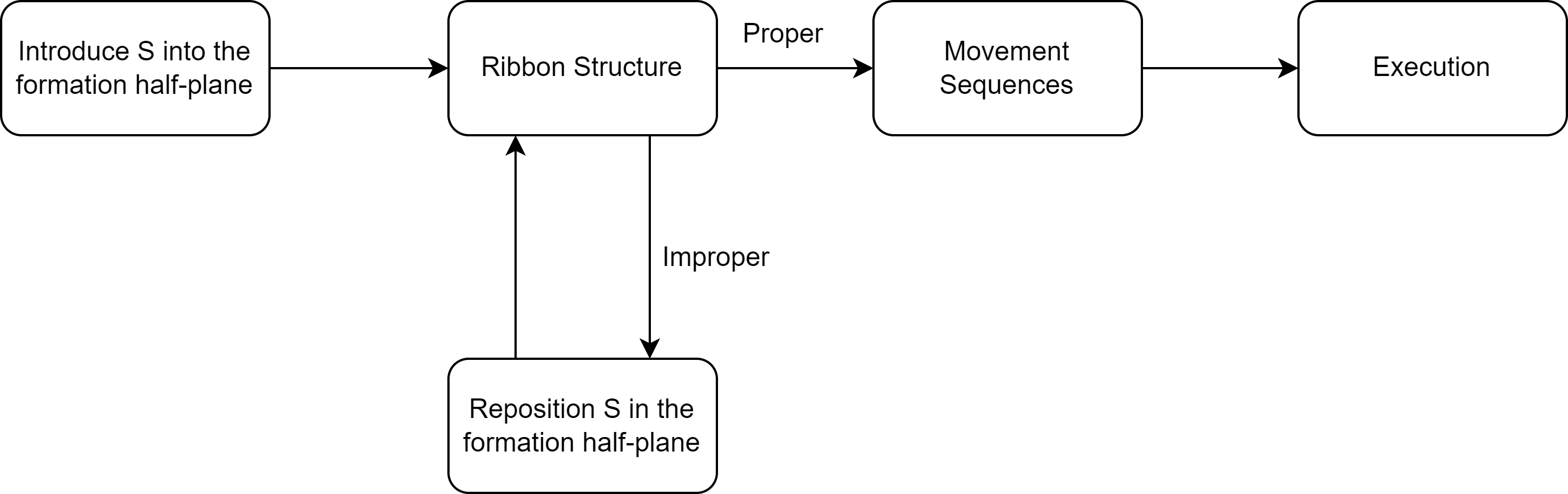}
    \label{fig:overall-flow}
    \caption{Flow diagram of the add-subtract algorithm}
\end{figure}

\subsection{Ribbon Structure}
\label{ribbon_structure}
The lattice points in $S$ can be partitioned according to their distance\footnote{Measured by the length of the shortest path from the lattice point to the origin, with the additional requirement that the vertices of the path must be in $S$.} from the origin of $\mathbb{O}$. Notice the following\footnote{Proof of the statements are given in section \ref{sect-Ribbonization}.}: 
\begin{enumerate}
    \item each partition induced a line graph (see figure \ref{fig:ribbons}), and we could impose the clockwise direction to obtain a directed line graph.
    Each directed line graph, or, equivalently, an ordered list of vertex points $(x_1, x_2,...,x_n)$ on the directed line graph is called a \textbf{ribbon}, denoted as $R_i$, where $x_1$ and $x_n$ are the start point\footnote{The vertex with out-degree $0$.} and end point\footnote{The vertex with in-degree $0$.} of the ribbon, respectively. For any two points $x_k$ and $x_j$ on the ribbon, we use $x_k \rightarrow x_j$ to denote that there exists a directed sub-line segment from $x_k$ to $x_j$ in the ribbon. The ribbon to which a lattice point $x$ belongs is denoted as $R(x)$. The number of vertices of $R_i$ in the hole(s) $D$ is denoted by $c(R_i)$. The vertex point with out-degree $0$ is called the leader vertex of the ribbon, denoted as $L(R_i)$;
    \item each vertex of a ribbon $R$ is adjacent to a vertex of another ribbon $R_p$, where $R_p$ is closer to the origin. Such a ribbon $R_p$ is unique and is referred to as the "parent ribbon" of the given ribbon $R$;
    \item by considering $(p_2,p_1)$ as a ribbon, all ribbons form a tree structure rooted at $(p_2,p_1)$. For any two ribbons $R_i$ $R_j$ of the tree, we use $R_i \twoheadrightarrow R_j$ to denote that there exists a chain of ribbons $(R_i,R_{i+1},\ldots,R_{j-1},R_j)$ where $R_i$ is a parent ribbon of $R_{i-1}$,...,$R_{j-1}$ is a parent ribbon of $R_j$.
\end{enumerate}

Then, we check whether the ribbon structure satisfies the following conditions:

\begin{enumerate}
    \item $S$ is placed adjacent to the set of seed robots but $S \cap \{p_0,p_1,p_2,p_3,p_4,p_5,p_6\} = \emptyset$;
    \item $(p_2,p_1)$ has only one child ribbon and the ribbon has a length of three;
    \item each ribbon has two vertices;
    % \item $S$ has a smooth boundary in the sense that
    %         \[ (\forall x,y \in S) (||x-y||=2d \Rightarrow ((\forall z \in L) ||x-z||=||y-z||=d \Rightarrow z \in S ))\] 
    %         and
    %         \[ (\forall x,y \in S) (||x-y||=\sqrt{3} \Rightarrow \neg((\forall z \in L) ||x-z||=||y-z||=d \Rightarrow z \notin S ))\]
\end{enumerate}

If the above conditions are satisfied, the ribbon structure is called \textbf{proper}. If the conditions are not met, we adjust the position of $S$ in $H_1$ by rotation and translation until a proper ribbon structure is obtained.

In the same way, all the lattice points inside the idle half-plane $H_2$ can be partitioned into ribbons. To distinguish from the ribbons defined in $S$, we call such ribbons defined in $H_2$ as the "idle ribbons" and denoted by "$IR_i$". The idle ribbons also form a tree structure (with a single branch) rooted at $(p_5,p_4)$. Notice that the idle ribbon structure is proper.

\begin{figure}[H]
    \centering
    \includegraphics[scale=0.35]{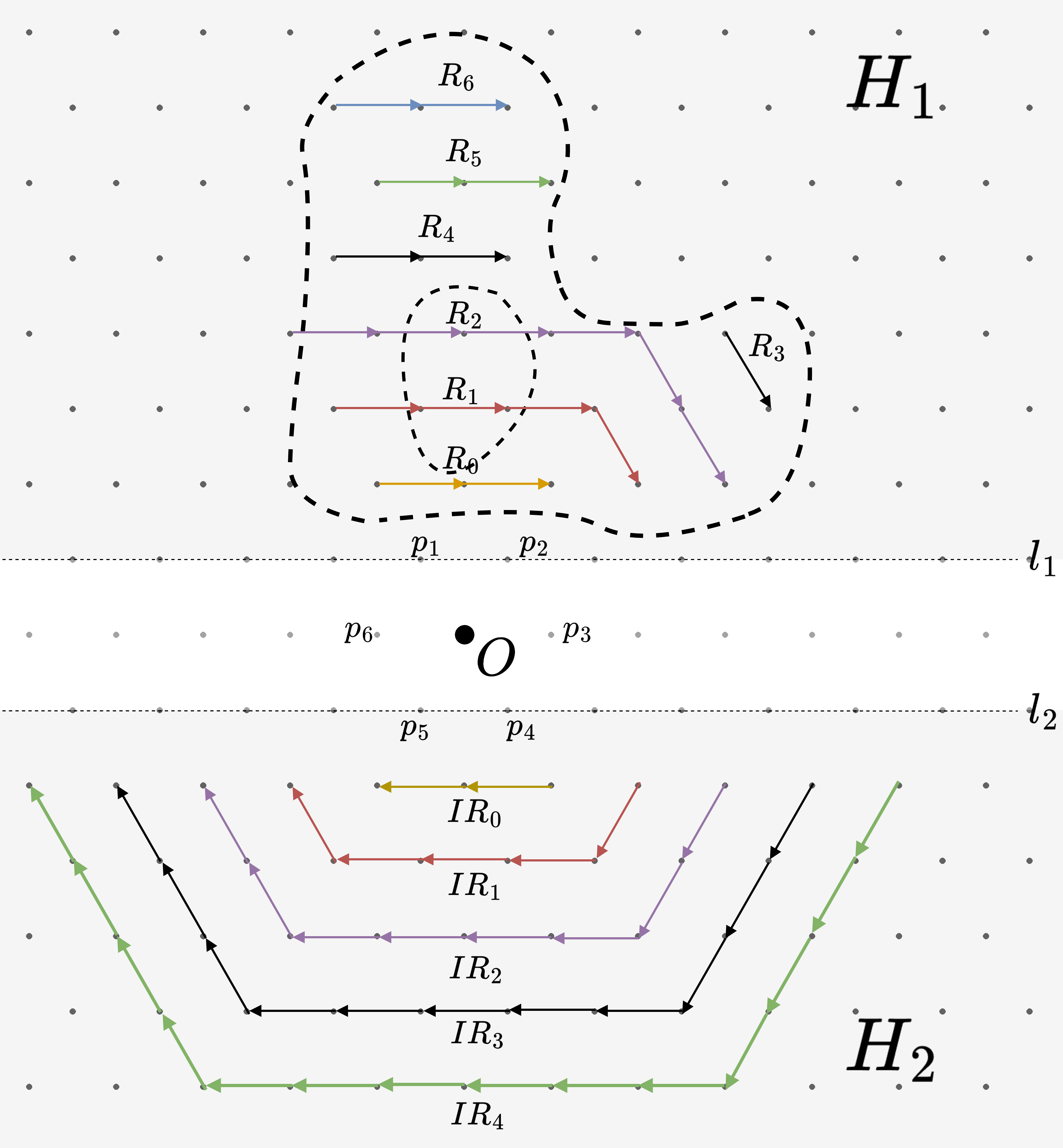}
    \caption{The ribbons defined in $S$ and the idle ribbons defined in $H_2$}
    \label{fig:ribbons}
\end{figure}

%%%%%%%%%%%%%%%%%%%%%%%%%%%%%%%%%%%%%%%%%%%%%%%
\subsection{Planning of Movement Sequences}

In view of the tree structure of directed ribbons, an order exists for all lattice points on the same ribbon and for all ribbons. By using ribbons and the tree structure of ribbons, we could define the following orders:

\begin{enumerate}
    \item The \textbf{ribbon order} is a total order of the lattice points of a ribbon $R$ such that for all different $x,y$ of ribbon $R$, $x < y$ if $y \rightarrow x$.
    \item The \textbf{tree order} is a partial order of the ribbons of $S$ such that for all different ribbons $R_i, R_j$, $R_i< R_j$ if $R_i \twoheadrightarrow R_j$. A \textbf{complete tree order} is a total order that extends from the tree order\footnote{A partial order extends to a total order by the Szpilrajn extension theorem.}. 
\end{enumerate}

\noindent Similarly, by using idle ribbons and the tree structure of idle ribbons, we could define the following orders:

\begin{enumerate}
    \item The \textbf{idle ribbon order} is a total order of the lattice points of an idle ribbon $IR$ such that for all different $x,y$ of idle ribbon $IR$, $x < y$ if $y \rightarrow x$.
    \item The \textbf{idle tree order} is a total\footnote{Notice that the tree of idle ribbons has a single branch.} order of the idle ribbons such that for all different idle ribbons $IR_i, IR_j$, $IR_i< IR_j$ if $IR_i \twoheadrightarrow IR_j$
\end{enumerate}

For each additive epoch $T$ (i.e., $1 \leq T \leq N$), we use two sequences of lattice points to indicate which robot becomes active and where it is expected to stop:

\begin{itemize}
    \item [-] the activation sequence $s[k]$: the $k$th entry denotes the lattice point at which robot $u_k$ starts to move;
    \item [-] the assembly sequence $t[k]$: the $k$th entry denotes the lattice point at which robot $u_k$ stops.
\end{itemize}

Now, we shall use the orders previously defined to formulate the sequences. The activation sequence $s$ is generated recursively such that each element $s[k]$ is the largest lattice point in $H_2$ (w.r.t the idle ribbon order) on the largest idle ribbon (w.r.t the idle tree order) that has not been included in the early subsequence $s_{[1:k-1]}$. The assembly sequence $t$ is generated recursively such that each element $t[k]$ is the smallest lattice point in $S$ (w.r.t the ribbon order) on the smallest ribbon (w.r.t the complete tree order) that has not been included in the early subsequence $t_{[1:k-1]}$. Loosely speaking, in each epoch, the idle robot at the tail of the outermost idle ribbon is activated first, and the lattice positions in $S$ are to be occupied ribbon by ribbon, from front to tail.

Similarly, for the subtractive epoch $T$ (i.e., $T \geq N+1$),  we use two sequences of lattice points to indicate which robot becomes active at each epoch and where it is expected to stop:

\begin{itemize}
    \item [-] the re-activation sequence $s_r[k]$: the $k$th entry denotes the lattice point at which robot $u_{N+k}$ starts to move;
    \item [-] the re-assembly sequence $t_r[k]$: the $k$th entry denotes the lattice point at which robot $u_{N+k}$ stops.
\end{itemize}

The re-activation sequence $s_r[k]$ for $ 1\leq k \leq N$ is obtained recursively such that each element $s_r[k]$ is the smallest lattice point in $S$ (w.r.t the ribbon order) on the largest ribbon (w.r.t the complete tree order) which has not been included in ${s_r}_{[1:k-1]}$. The re-assembly sequence $t_r[k]$ for $ 1\leq k \leq N$ is generated recursively such that each element $t_r[k]$ satisfies:
\begin{itemize}
    \item[-] if the lattice position $s_r[k]$ is among the first $c(R(s_r[k]))$\footnote{Recall that $c(R)$ is the number of vertices of $R$ in the hole(s) $D$.} lattice points of the ribbon $R(s_r[k])$, $t_r[k]$ is the smallest lattice point in $s_{[1:N]}$ (w.r.t the idle ribbon order) on the smallest idle ribbon (w.r.t the idle tree order) which has not been included in ${t_r}_{[1:k-1]}$;
    \item[-] if $s_r[k]$ is not among the first $c(R(s_r[k]))$ lattice points of the ribbon $R(s_r[k])$, $t_r[k]$ is the smallest lattice point in $S \setminus D$ (w.r.t the ribbon order) on the largest ribbon (w.r.t the complete tree order) which has not been included in ${t_r}_{[1:k-1]}$.
\end{itemize}

By following the re-activation sequence and the re-assembly sequence, extra robots of each ribbon move outside $S$ while others rearrange themselves along the ribbon to fill up the lattice points in $S \setminus D$. We refer to the two types of robots as \textbf{recycled robots} and \textbf{rearranged robots}, respectively. Loosely speaking, if the robot is a recycled robot, it moves along the ribbon until it exits the shape $S$, after which it orbits around the perimeter of the group of robots and assembles layer by layer in $H_1$. Conversely, a rearranged robot moves forward along the ribbon and stops at the next unoccupied vertex in $S \setminus D$. 

All robots obtain a copy of the activation sequence, re-activation sequence, and the corresponding assembly sequence and re-assembly sequence before the execution of the algorithm, which means a centralized information sharing mechanism will be adopted here. Such an assumption can be relaxed in real applications, which will be discussed at the end of this paper.

%%%%%%%%%%%%%%%%%%%%%%%%%%%%%%%%%%%%%%%%%%%%%%%
\subsection{Execute of the Algorithm}

The execution of this algorithm utilizes the following elementary behaviors:
\begin{enumerate}

    \item Localization: the moving robot establishes its coordinates based on the location of the non-moving robots in its communication range by triangulation. The following figure shows the localization of the green robot:
        \begin{figure}[H]
        \centering
        \includegraphics[scale=0.7]{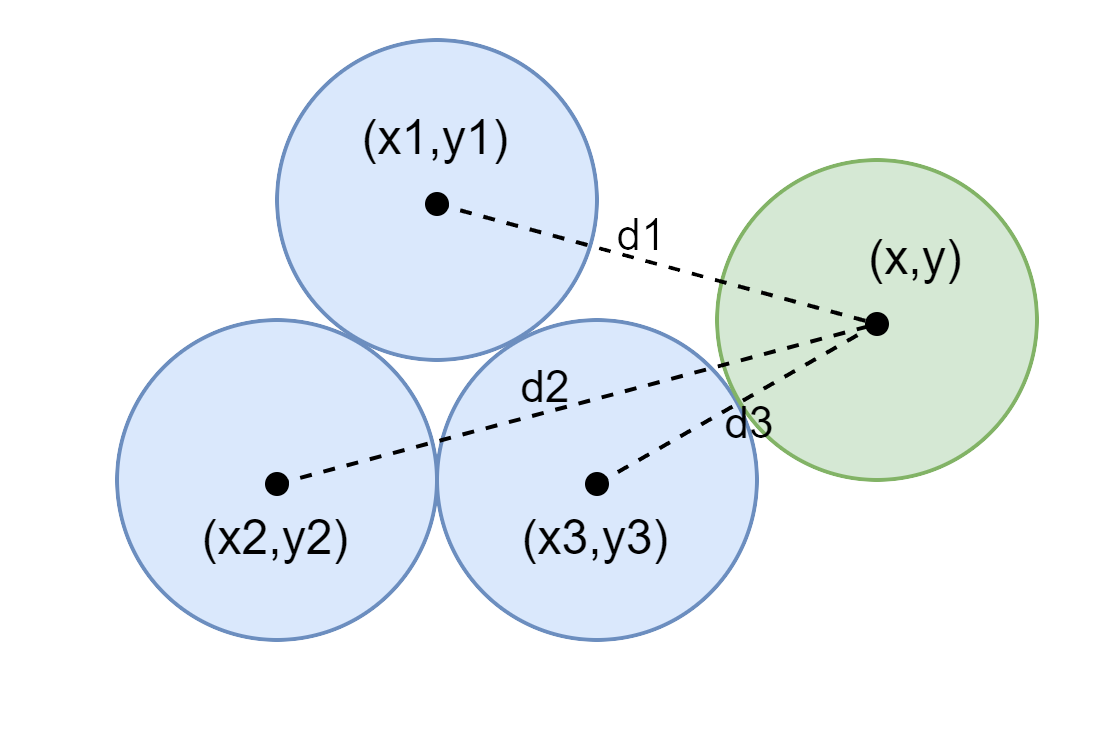}
        \label{fig:Localization}
        \end{figure}
    where
    \[
        (x,y) = argmin(\sum_{i=1}^{3} |d_i - \sqrt{(x-x_i)^2 + (y-y_i)^2}|)
    \]
    Note that the algorithm requires the coordinates of at least three non-collinear robots to deduce a unique solution.

    \item Gradient-Formation: The root robot maintains a constant gradient value of 0. All the other robots choose the smallest gradient value among their neighboring robots and set their own gradient value to be the minimum gradient value plus $1$.

    \item Edge-Following: the moving robot orbits around a collection of non-moving robots in the clockwise direction while maintaining distance $d$ to the closest non-moving robot. This behavior is achieved by turning right if the edge-following robot senses its distance to the nearest stationary robot is larger than $d$ and turning left if the distance is smaller than $d$.  The movement can be visualized as the robot travels along the perimeter of the MRS.

    \item Ribbon-Following: the moving robot orbits around a ribbon in the clockwise direction by maintaining distance $d$ to the closest robot located on the same ribbon. Please be reminded that $d = \frac{2r}{\sqrt{3}-1}$ is slightly larger than 2r. Thus, the collision can still be avoided while a robot performs the ribbon-following movement within the shape $S$.

    \item Epoch-Update: each robot assumes the maximum epoch among its neighbors.

\end{enumerate}

\medskip

Now, we explain the execution of the algorithm step-by-step. Each robot $v$ continuously broadcasts its current coordinates $p_t(v)$, its gradient value $g(v)$ and its epoch $T(v)$ to the robots in its neighborhood. Each robot $v$ will also receive the coordinates $p_t(u)$, the distance $d(u,v)$, the gradient value $g(u)$ and the local epoch $T(u)$ from all its neighboring robot $u \in N_t(v)$. Here, we provide a flow diagram for the execution of the algorithm and the corresponding pseudocodes.

\begin{figure}[H]
    \centering
    \includegraphics[scale=0.6]{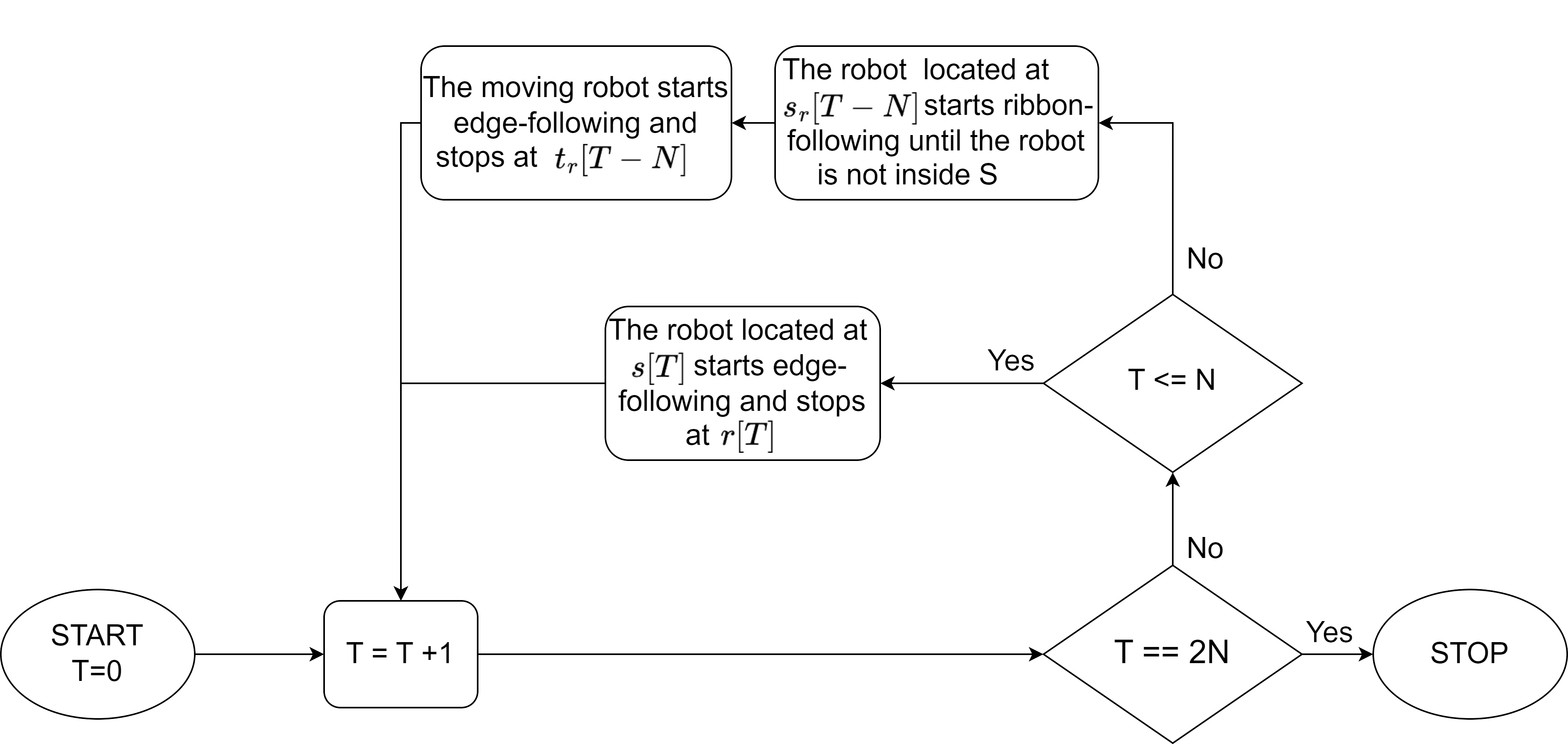}
    \caption{Flow diagram of the execution of the add-subtract algorithm}
    \label{fig:flowchart}
\end{figure}

Each robot of the MRS is given sufficient time to execute the localization algorithm and the gradient formation algorithm before the start of the iterations of the add-subtract algorithm. After each non-seed robot is stabilized, the execution of the iterations starts from epoch $T=1$. During each epoch $T$:

\begin{enumerate}
    \item if $1 \leq T \leq N$, the robot located at $s[T]$ starts edge-following and stops at $t[T]$. After each robot stops, the epoch count is increased by 1. In the additive stage, robots build up layer by layer and occupy all lattice points in $S$.
    \begin{figure}[H]
        \centering
        \includegraphics[scale=0.4]{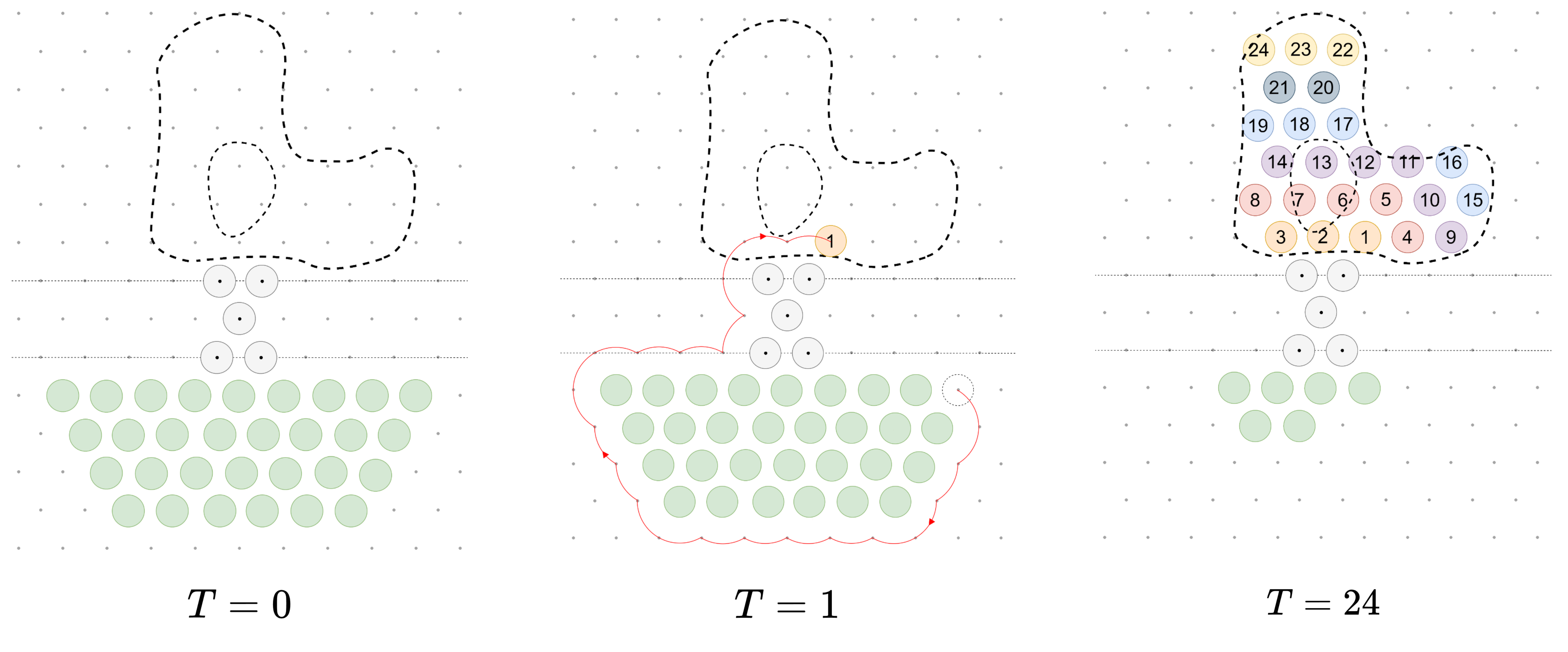}
        \caption{Illustration of the additive stage}
        \label{fig:Illustration of the additive stage}
    \end{figure}

    \item if $N+1 \leq T \leq 2N$, the robot located at $s_r[T-N]$ starts ribbon-following until it is outside the shape $S$. Then, the robot executes the edge-following algorithm until it stops at $t_r[T-N]$. After each robot stops, the epoch count is increased by 1 until $T=2N$, and the algorithm stops. In the subtractive stage, extra robots exit the hole $D$ and assemble in the idle half-plane. 
    \begin{figure}[H]
        \centering
        \includegraphics[scale=0.4]{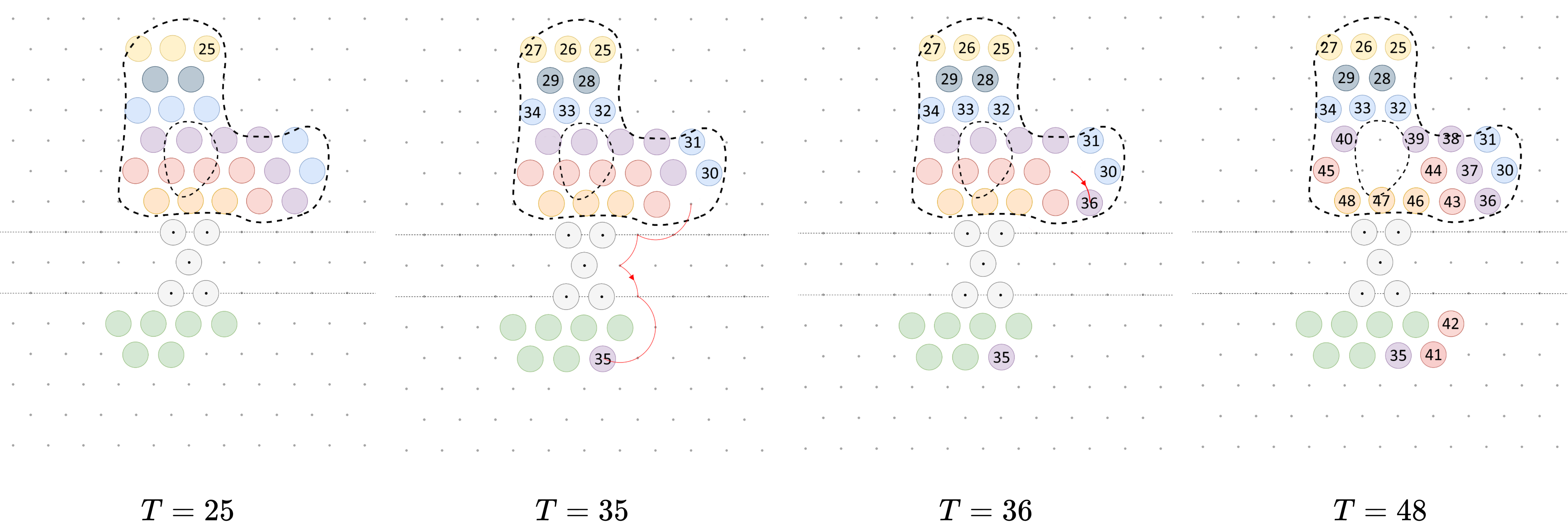}
        \captionsetup{justification=centering}
        \caption{Illustration of the subtractive stage. Notice that some robots appear stationary: \\ for instance, robot $u_{25}$, which has the same starting and stopping position since $s_r[25]=t_r[25]$}
        \label{fig:Illustration of the subtractive stage}
    \end{figure}

\end{enumerate}

\section{Proof of Correctness}
\label{sect-outlinedProof}

In this section, we introduce the important definitions and provide an outline of the proof of correctness. The proof is divided into two main parts: 
\begin{enumerate}
    \item in part I, we establish a formal way to discretize shape $S$ into layers of ribbons (definition \ref{def-directedRibbon-doc}). Then, we prove that the ribbons form a tree structure if $S$ is "properly" placed in the workspace.  The resulting set of ribbons is referred to as the \textbf{ribbonization} of $S$, and the ribbonization is \textbf{proper} if a series of requirements are satisfied (definition \ref{def-properRibbonization-doc}). Finally, the order among ribbons and the order among lattice points of the same ribbon (see definition \ref{def-ribTreeOrder}) are defined. The same procedure is applied to discretize the idle half-plane $H_2$ and respective orders are defined (see definition \ref{def-idleRibTreeOrder}). 
    \item in part II, we establish the movement sequences $s[k]$, $t[k]$, $s_r[k]$ and $t_r[k]$ according to the orders established before. After that, we show that:
    \begin{itemize}
        \item[-] for an arbitrary epoch $T$ of the additive stage, the path from $s[T]$ to $t[T]$ is unobstructed, and the moving robot can localize itself whenever necessary;
        \item[-] for an arbitrary epoch $T$ of the subtractive stage, the path from $s_r[T-N]$ to $t_r[T-N]$ is unobstructed, and the moving robot can localize itself whenever necessary.
    \end{itemize}
    Notice that the activation sequence $s$, the reactivation sequence $s_r$, the assembly sequence $t$ and the re-assembly sequence $t_r$ uniquely defined the sequence of tuples $(s_1,t_1,1), \ldots, (s_N,t_N,N), ({s_r}_1,{t_r}_1, N+1), \ldots, ({s_r}_N,{t_r}_N, 2N)$. We conclude that the sequence of tuples satisfies the problem statement. 
\end{enumerate}

Due to the page limit, only the outline for each part will be presented. For the complete proof, the reader could refer to appendix \ref{sect-Ribbonization} for the proofs of part I and appendix \ref{sect-methodAndProof} for the proofs of part II.

\subsection{Part I}

We start by placing the target shape $S$ in the workspace. We assume that the placement is \textbf{proper} such that $S$ is placed entirely in the assembly half-plane $H_1$ and adjacent to the seed robots.
\begin{definition}
    \label{def-properPlacement-doc}
    A placement of $S$ in the coordinate system $\mathbb{O}$ is \textbf{proper} if  $S \in H_1$ and points $O,p_1,p_2,p_3,p_4,p_5,p_6$ are not in $S$.
\end{definition}

\begin{definition}
    For arbitrary lattice point $x \in S$, the \textbf{hop count} for $x$ is the length of the shortest path from $x$ to the origin $O$ in $\Gamma$ that involves only lattice points in $S$ and lattice points in $\{p_1,p_2,\ldots,p_6\}$. If such a path does not exist, the hop count is defined as $-1$. The hop count of $x$ is denoted as $h(x)$. 
\end{definition}

After investigating the properties of hop count, we derive the following two patterns of hop counts:

\begin{proposition}
    Given a lattice point $x \in S$ with a positive hop count $k$, the neighborhood of $x$ can only be one of the two types:
    \begin{itemize}
        \item [-] Type $A$: $x$ is adjacent to one lattice point $y$ with hop count $k-1$. The lattice point(s) in $S$ adjacent to both $x$ and $y$ has the hop count equal to $k$. The remaining neighbors in $S$, if any, have the hop count equal to $k+1$.
        \item [-] Type $B$: $x$ is adjacent to two adjacent lattice point $y,z$ with hop count $k-1$. The lattice point(s) in $S$ that is adjacent to both $x$ and $y$ or adjacent to both $x$ and $z$ has the hop count equal to $k$. The remaining neighbors in $S$, if any, have the hop count equal to $k+1$.
    \end{itemize}
    \begin{figure}[H]
        \centering
        \includegraphics[scale=0.5]{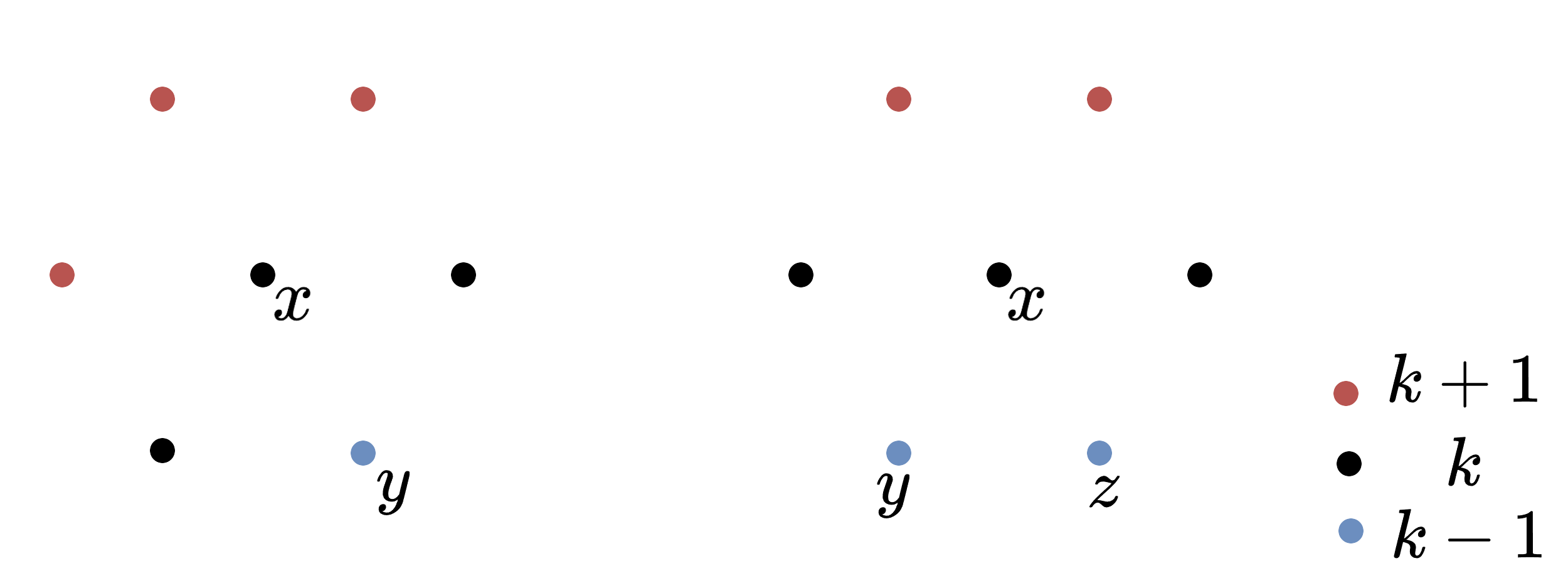}
        \caption{Left: type $A$ neighborhood of $x$; Right: type $B$ neighborhood of $x$}
    \end{figure}

\end{proposition}

Then, we define undirected ribbons as the maximum continuous sequence of adjacent lattice points of the same hop count:

\begin{definition}
    Let $G$ be the subgraph of $\Gamma$ induced by the sets of lattice points with the same hop count. An \textbf{undirected ribbon} $R^*$ is defined as a component of $G$. The \textbf{gradient} of $R^*$ is the hop count of any lattice point on the ribbon $R^*$.
\end{definition}

Then the directed ribbons, or simply ribbons, can be defined:
\begin{definition}
    \label{def-directedRibbon-doc}
    For a given undirected ribbon $R^*$, its corresponding directed ribbon $R$ is a directed graph such that for each vertex $x \in R$, the arc entering $x$ and the arc leaving $x$ is defined such that:
    \begin{figure}[H]
        \centering
        \includegraphics[scale=0.5]{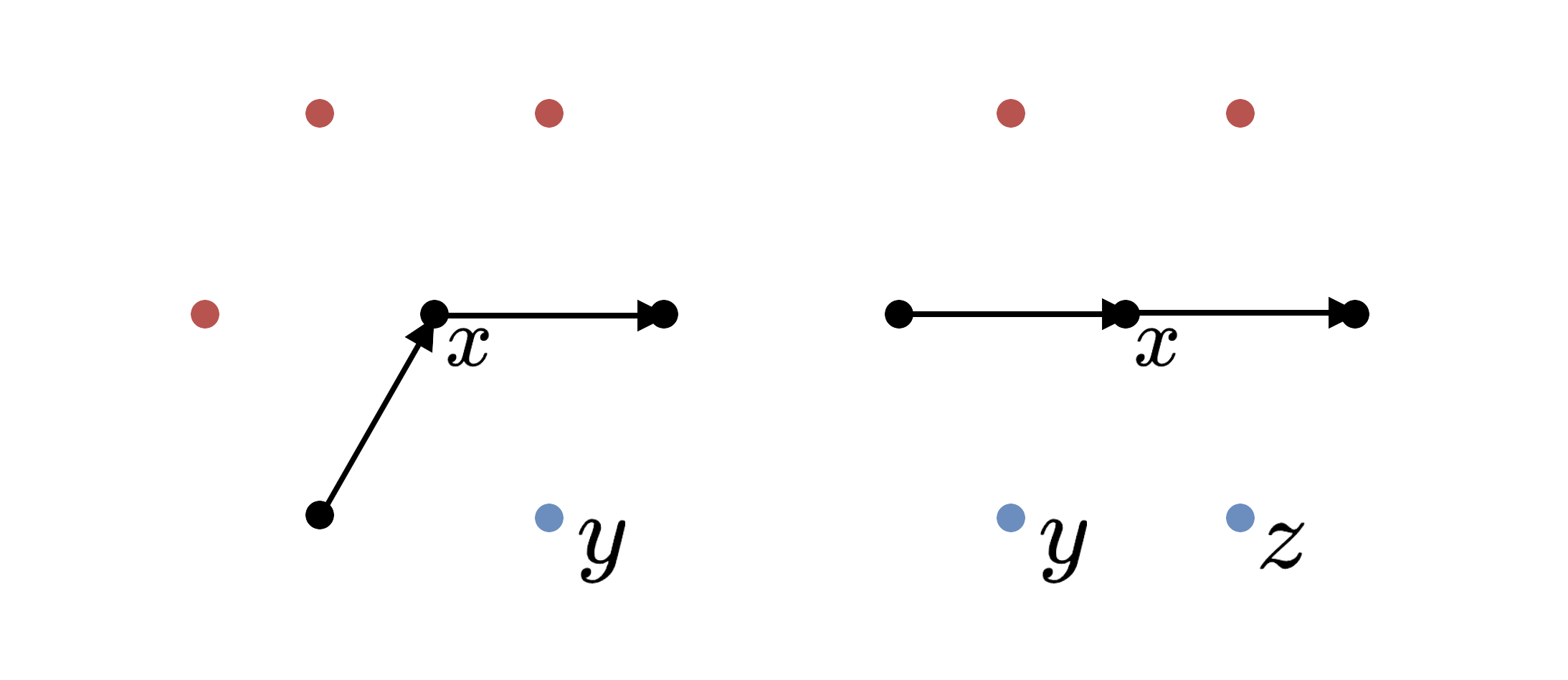}
    \end{figure}
    \begin{itemize}
        \item [-] if $x$ has the type $A$ neighborhood, the two arcs are directed clockwise around $y$;
        \item [-] if $x$ has the type $B$ neighborhood, the arc entering $x$ is directed clockwise around $y$, and the arc leaving $x$ is directed clockwise around $z$.
    \end{itemize}
    The \textbf{gradient} of ribbon $R$ is the hop count of any lattice point belonging to the ribbon $R$. Denote as $\bm{g(R)}$. The \textbf{length} of $R$ is the number of vertices in $R$. 
\end{definition} 

We define the process of discretizing $S$ into ribbons as \textbf{ribbonization}, denoted as $\hat{S}$. The requirements for a proper ribbonization are given as follows:
% \begin{definition}
%     \label{def-properRibbonization-doc}
%     A \textbf{proper ribbonization} of a shape $S$ satisfies the following properties:
%     \begin{enumerate}
%         \item the properly placed condition: $S$ is properly placed in $\mathbb{O}$;
%         \item the connected condition: every lattice point in $S$ has a positive hop count;
%         \item the minimum length condition: $\forall R \in \hat{S}$, the length of ribbon $R$ is at least two;
%         \item the single root condition: only one ribbon with the gradient value equal to two, and the ribbon has a length of three;
%         \item the well-coverage condition: for any point $x \in S \setminus D$, there is a lattice point $y$ inside $S \setminus D$ such that $||x-y|| < d$;
%         \item the hole capture condition: every boundary point\footnote{The set of boundary point is defined as all lattice points in $S$ which is adjacent to a lattice point outside $S$} of $S$ is not in $D$; and             
%         \item the boundary smoothness condition: the line segment $l$ connecting any two lattice point in $S$ must locate inside $S$ unless $||l||>2d$.
        
%         %\item  $\forall CR_i'^{(j)}$ with $j=1,...,n$, $\forall x \in V_{CR_i'^{(j)}}$, $\exists y \in V_{\bigcup_{j}{CR_{i-1}'^{(j)}}}$ s.t. $||x-y||=2r$.
%     \end{enumerate}
    
% \end{definition}
\begin{definition}
    \label{def-properRibbonization-doc}
    A \textbf{proper ribbonization} of a shape $S$ satisfies the following properties:
    \begin{enumerate}
        % \item $(\forall x\in S\setminus D) \exists y \in \{S\setminus D\} \cap L, s.t. ||x-y||<d$;
        \item $S$ is properly placed in $\mathbb{O}$;

        \item every lattice point in $S$ has a positive hop count, that is,
            \[ (\forall x \in \{S\setminus D\} \cap L) h(x)>0  \]
        \item $(\forall R \in \hat{S}$) $len(R) \geq 2$;
        \item only one ribbon with the gradient value equal to two, and the ribbon has a length of three, that is,
            \[ (\exists! R \in \hat{S}) g(R)=2 \land len(R)=3\]

        \item every lattice point on the boundary of $S$ is not in $D$, that is
            \[ (\forall x \in S) (\exists y \in S^c \cap L \text{ and } ||x-y||=d) \Rightarrow x \notin D \]    
        \item $S$ has a smooth boundary in the sense that
            \[ (\forall x,y \in S) (||x-y||=2d \Rightarrow ((\forall z \in L) ||x-z||=||y-z||=d \Rightarrow z \in S ))\] 
            and
            \[ (\forall x,y \in S) (||x-y||=\sqrt{3} \Rightarrow \neg((\forall z \in L) ||x-z||=||y-z||=d \Rightarrow z \notin S ))\]        
        % \item the line segment $l$ connecting any two lattice point in $S$ must locate inside $S$ unless $||l||>2d$, that is
        %     \[ (\forall x,y \in S \cap L) ||x-y|| \leq 2d \Rightarrow (line(x,y) \subset S \setminus D)\]

    \end{enumerate}
    
\end{definition}

Notice that, by satisfying the proper lattice representation requirements in section \ref{probDescription} and the proper ribbon structure requirements in section \ref{ribbon_structure}, the ribbonization of $S$ is proper. We denote the ribbon with a gradient of $2$ as $R_0$. This ribbon serves as the root of the tree of ribbons.

\begin{definition}
    Given a proper ribbonization $\hat{S}$, its corresponding \textbf{ribbonization tree $Tr(\hat{S})$} is defined as a directed graph $(V(Tr(\hat{S})), A(Tr(\hat{S})))$ such that 
    \begin{enumerate}
        \item $V(Tr(\hat{S})) = \hat{S}$ and 
        \item $(\forall R_i, R_j \in V(Tr(\hat{S})))$ $(R_i,R_j) \in A(Tr(\hat{S}))$ if $R_i$ is the parent of $R_j$. 
    \end{enumerate}
\end{definition}

\begin{theorem}
    The ribbonization tree is a directed tree rooted at $R_0$.
\end{theorem}

In the similar way, the half-plane $H_2$ can be ribbonized into a tree-structure of idle ribbons $\{IR_0, IR_1, IR_2,... \}$. 

\begin{figure}[H]
    \centering
    \includegraphics[scale=0.13]{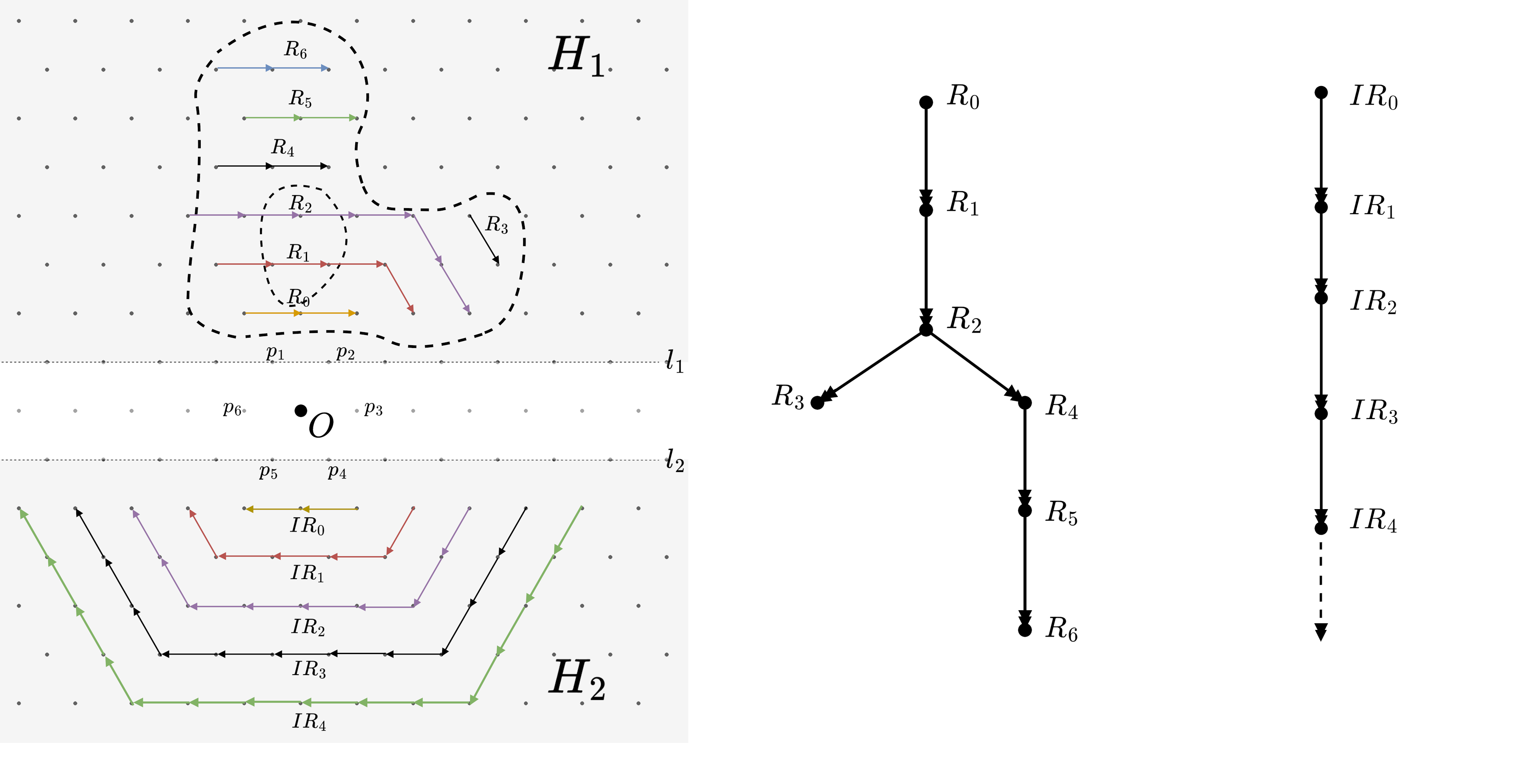}
    \captionsetup{justification=centering}
    \caption{Left figure: Ribbonization of $S$ and ribbonization of $H_2$;\\ middle figure: tree structure of ribbons; right figure: tree structure of idle ribbons}
    \label{fig:treeStructure}
\end{figure}

In view of the tree structure among ribbons, we could define the \textbf{ribbon order} and the \textbf{tree order} (definition \ref{def-ribTreeOrder}). Similarly, for the tree structure among idle ribbons, we could define the \textbf{idle ribbon order} and the \textbf{idle tree order} (definition \ref{def-idleRibTreeOrder}):

\begin{definition}
    \label{def-ribTreeOrder}
    \leavevmode
    \makeatletter
    \@nobreaktrue
    \makeatother
    \begin{enumerate}
        \item The \textbf{ribbon order}, denoted as $<$, is a total order of the lattice points of a ribbon $R$ such that
            \[(\forall x,y \in V(R)) x < y \iff y \rightarrow x \]
        \item The \textbf{tree order}, denoted as $<$, is a partial order of the ribbons in $S$ such that
            \[(\forall R_i,R_j \in \hat{S}) R_i < R_j \iff R_i \twoheadrightarrow R_j \]
        The \textbf{complete tree order} is a strict total order that extends the tree order\footnote{A partial order extends to a total order by the Szpilrajn extension theorem.}. 
        % $\forall R_i, R_j$, $R_i< R_j$ iff the path from the root $R_0$ to $R_j$ passes through $R_i$. A \textbf{complete tree order} is a strict total order that extends from the tree order.
    \end{enumerate}
\end{definition}

\begin{definition}
    \label{def-idleRibTreeOrder}
    \leavevmode
    \makeatletter
    \@nobreaktrue
    \makeatother
    \begin{enumerate}
        \item The \textbf{idle ribbon order}, denoted as $<$, is a total order of the lattice points of an idle ribbon $IR$ such that
            \[(\forall x,y \in V(IR)) x < y \iff y \rightarrow x \]
        \item The \textbf{idle tree order}, denoted as $<$, is a partial order of the idle ribbons in $S$ such that
            \[(\forall IR_i,IR_j \in \hat{H_2}) IR_i < IR_j \iff IR_i \twoheadrightarrow IR_j \]
        % $\forall R_i, R_j$, $R_i< R_j$ iff the path from the root $R_0$ to $R_j$ passes through $R_i$. A \textbf{complete tree order} is a strict total order that extends from the tree order.
    \end{enumerate}
\end{definition}

% \begin{definition}
%     \label{def-idleRibTreeOrder}
%     \leavevmode
%     \makeatletter
%     \@nobreaktrue
%     \makeatother
%     \begin{enumerate}
%         \item The \textbf{idle ribbon order} is a strict total order of the lattice points of an idle ribbon $IR$ such that for all $x,y$ of idle ribbon $IR$, $x < y$ iff there is a directed path from $y$ to $x$.
%         \item The \textbf{idle tree order} is a strict total order of the idle ribbons in $S$ such that $\forall IR_i, IR_j$, $IR_i< IR_j$ iff the path from the root $IR_0$ to $IR_j$ passes through $IR_i$.
%     \end{enumerate}
% \end{definition}    

\subsection{Part II}
We prove the correctness of the algorithm in two steps: for the additive stage, we show that the path between $s[T]$ and $t[T]$ is unobstructed, and the active robot could localize itself using non-active robots at $s[T]$ and $t[T]$; for the subtractive stage, we show that the path between $s_r[T-N]$ and $t_r[T-N]$ is unobstructed, and the active robot could localize itself using non-active robots at $s_r[T-N]$, $t_r[T-N]$ and the merging point (defined shortly).

\subsubsection{Additive Stage}

By induction on the perimeter of the MRS in each additive epoch, together with the boundary smoothness condition (required by the proper ribbonization), we show that there is a unique and unobstructed path for robot $u_{T}$ when it executes the edge-following movement and $t[T]$ is on this unique path.

\begin{proposition}
    \label{thm-unobstructedPath-outline}
    In each epoch $T \leq N$,  for the robot $u_T$ that is executing the edge-following algorithm, its path from $s[T]$ to $t[T]$ is unobstructed.
\end{proposition}

In terms of localization, we first notice that all idle robots are able to localize themselves. Since every ribbon has a minimum length of two (as required by the minimum length condition of proper ribbonization), we deduce that the robot $u_T$ is localizable at $t[T]$.

\begin{figure}[H]
    \centering
    \includegraphics[scale=0.6]{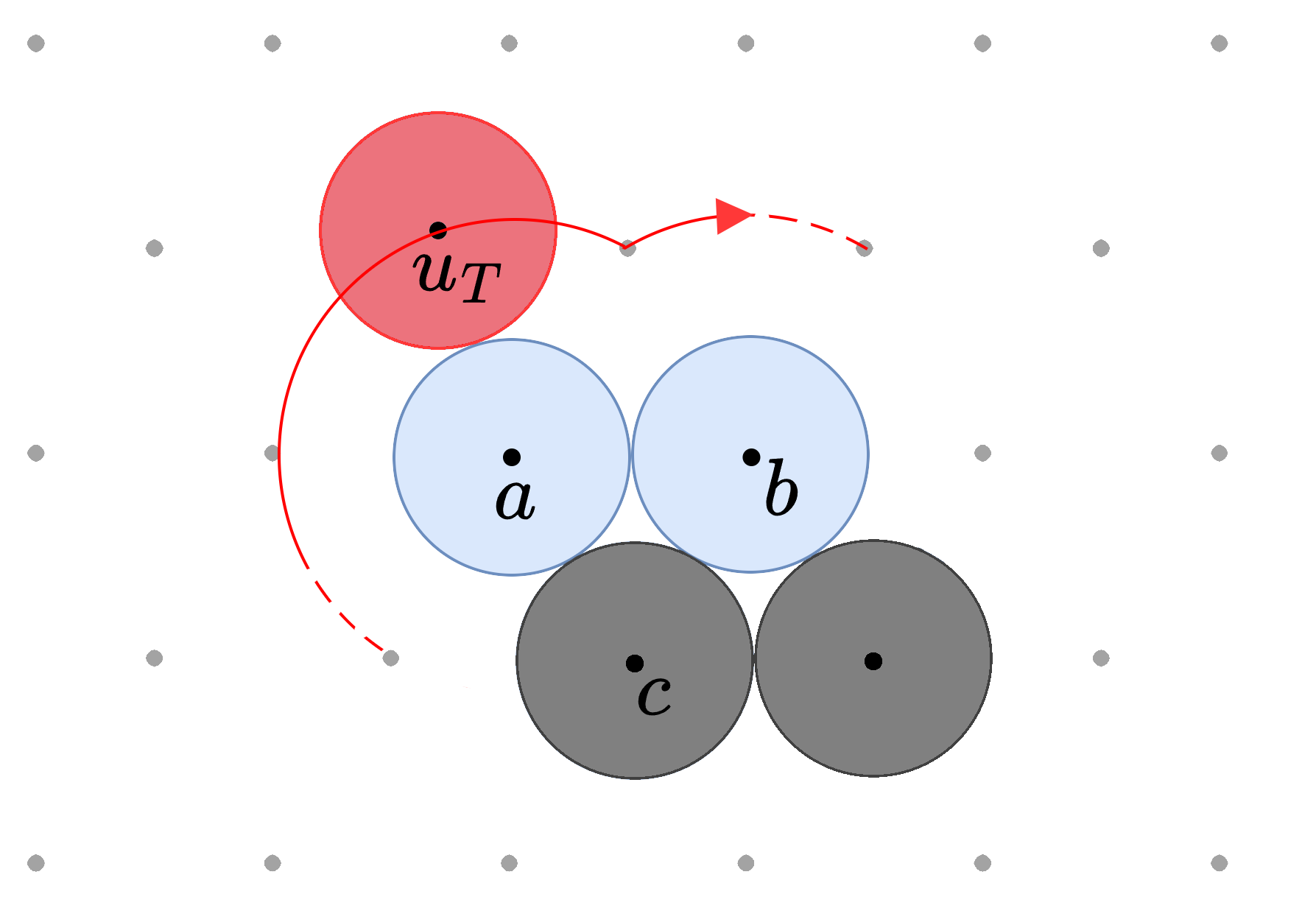}
    \captionsetup{justification=centering}
    \caption{Illustration of the localization: $u_T$ is localizable at the current position\\ since stopped robots $a,b,c$ are within $2d$ distance from $u_T$}
    
\end{figure}

\begin{proposition}
    \label{thm-localizable-outline1}
    In each epoch $1 \leq T \leq N$, robot $s_T$ is localizable at lattice point $s[T]$ and lattice point $t[T]$.
\end{proposition}

% Since each active robot $u_T$ executing the edge-following algorithm has an unobstructed path from $s[T]$ and $t[T]$, and $u_T$ is localizable at $s[T]$ and $t[T]$, we conclude that:
% \begin{theorem}
%     \label{thm-additive-stage-outline}
%     In each epoch $T$ with $1 \leq T \leq N$, the idle robot located at $s[T]$ starts to move and stops at $t[T]$.
% \end{theorem}  

\subsubsection{Subtractive Stage}
For the subtractive stage, in each epoch $T$ with $N+1\leq T \leq 2N$, the active robot $u_T$ has three cases: 
\begin{itemize}
    \item [-]case 1: if $s_r[T-N] = t_r[T-N]$, the robot $u_{T}$ shall stay at its current position;
    \item [-]case 2: if $s_r[T-N] \neq t_r[T-N]$ and $t_r[T-N] \in H_2$, then the robot $u_{T}$ shall perform the ribbon-following movement until it exits the shape. After that, the robot performs the edge-following movement until it stops at $t[T]$;
    \item [-]case 3: if $s_r[T-N] \neq t_r[T-N]$ and $t_r[T-N] \in S \setminus D$, then the robot $u_{T}$ shall only perform the ribbon-following movement and stops inside $S \setminus D$.
\end{itemize}
The correctness of the first case is trivial, and we shall show the correctness of the second and the third case\footnote{The third case can be viewed as a special case of the second, where only ribbon-following is executed}. 

To facilitate the proof, we define the \textbf{merging point} $\mathbf{m}$ for each ribbon as the first lattice point a ribbon-following robot encounters after exiting the shape $S$. Next, we investigate the perimeter of the swarm after $u_N$ stops, and show that the merging point is in front of all the other points on the perimeter contributed by the same ribbon. This property ensures that the robot exiting $S$ follows the perimeter of the swarm by executing the edge-following algorithm. The path from $s_r[T-N]$ to $t_r[T-N]$ is divided by the merging point into the \textbf{interior path} and the \textbf{exterior path}:

\begin{figure}[H]
    \centering
    \includegraphics[scale=0.4]{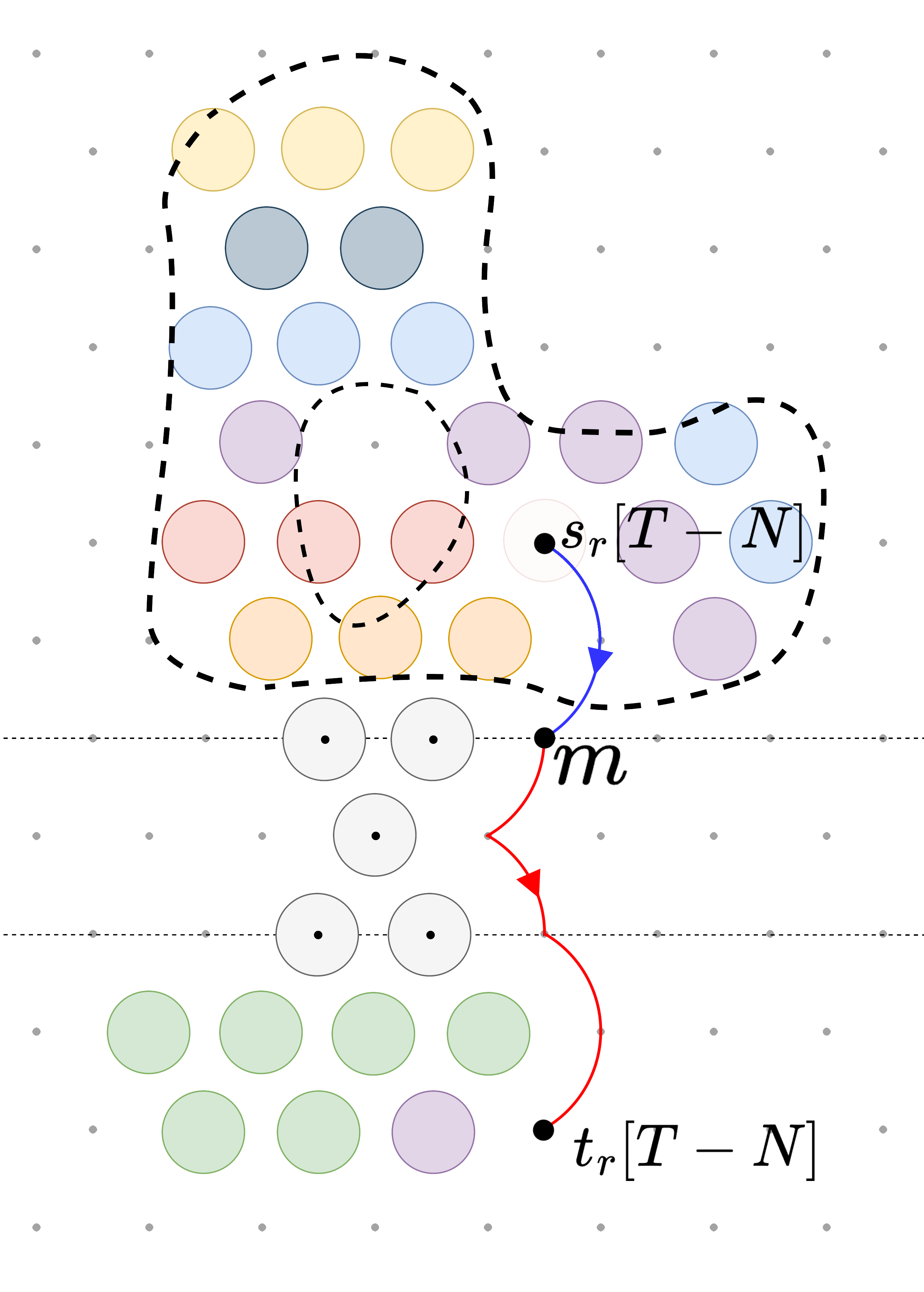}
    \caption{Illustration of the interior path (blue), the merging point $m$, and the exterior path (red)}
\end{figure}

Note that the interior path corresponds to the path executed with the ribbon-following algorithm, and the exterior path corresponds to the path executed with the edge-following algorithm. By investigating the geometry of the lattice, we show that the interior path from $s_r[T-N]$ to $m$ is always unobstructed. By similar arguments in the additive stage, we deduce that the exterior path from $m$ to $t_r[T-N]$ is unobstructed. As a result,

\begin{proposition}
    \label{thm-localizable-outline3}
    For any active robot in the subtractive stage, its path from $s_r[T-N]$ to $t_r[T-N]$ is unobstructed.
\end{proposition}

In terms of localization, we apply the similar arguments as in the additive stage and conclude that
\begin{proposition}
    \label{thm-localizable-outline2}
    In each epoch $T$ with $N+1 \leq T \leq 2N$, robot $s_T$ is localizable at lattice point $s_r[T-N]$, $t_r[T-N]$ and the merging point $m$.
\end{proposition}

% Since each active robot $u_T$ executing the ribbon-following and the edge-following algorithm has an unobstructed path from $s_r[T-N]$ to $t_r[T-N]$, and $u_T$ is localizable at $s_r[T-N]$, $t_r[T-N]$ and the merging point, we conclude that:

% \begin{proposition}
%     \label{thm-subtractive-stage-outline}
%     In each epoch $T$ with $N+ 1 \leq T \leq 2N$, the robot $s_T$ locating at $s_r[T-N]$ stops at $t_r[T-N]$.
% \end{proposition}

Finally, by proposition \ref{thm-unobstructedPath-outline}, proposition \ref{thm-localizable-outline1}, proposition \ref{thm-localizable-outline3} and  proposition \ref{thm-localizable-outline2}, we conclude that:
\begin{theorem}
    Given sufficient robots (with an initial set of robots $B$ whose position coordinates are known in advance) and the user-defined shape $S \setminus D$. The add-subtract algorithm with the movement sequence of tuples $(s_1,t_1,1), \ldots, (s_N,t_N,N), ({s_r}_1,{t_r}_1, N+1), \ldots, ({s_r}_N,{t_r}_N, 2N)$ satisfies:
    \begin{enumerate}
        % \item before the movement of any robot, the robots centered at points in $B$ can locate their position.
        \item the final collection of robots forms the desired shape in the workspace $H_1$ in the sense that $\{\{b_1,\ldots, b_{2N}\} \circleddash \{a_1,\ldots, a_{2N}\}\} \sqcap H_1  = S \setminus D$.
        \item For each prefix subsequence $(a_1, b_1, 1), \ldots, (a_{k}, b_{k},k)$, $1 \leq k \leq 2N$, let $A_{k} = \{a_1, \ldots, a_{k-1}\}$, $B_{k} = \{b_1, \ldots, b_{k-1}\}$, and $\overline{B_{k}} = B_{k}$. Then the starting position of the $k$th robot $a_k \in I \oplus \overline{B_{k}} \circleddash A_{k}$ and the robot can only relay on robots centered at $I \oplus \overline{B_{k}} \circleddash A_{k} \circleddash \{a_k\}$ to locate its position before stopping at $b_k$. 
    \end{enumerate}
\end{theorem}

\section{Simulation Case Studies}
\label{sect-simulation}

We use the Kilombo simulator \cite{jansson_kilombo_2016} with robot size $r=17 mm$\footnote{The same size as a kilobot.}, and lattice size $d = 1.1 * \frac{2r}{\sqrt{s}-1}\approx 51.08 mm$\footnote{Here we multiply $1.1$ to avoid any possible collisions that caused by truncation errors}. 

To demonstrate the scalability of the add-subtract algorithm, the example shape is scaled to different sizes. The simulation result with scaling parameter $0.5,1.0,1.5,2.0,3.0$ is shown as follows:
\begin{figure}[H]
    \centering
    \includegraphics[scale=0.4]{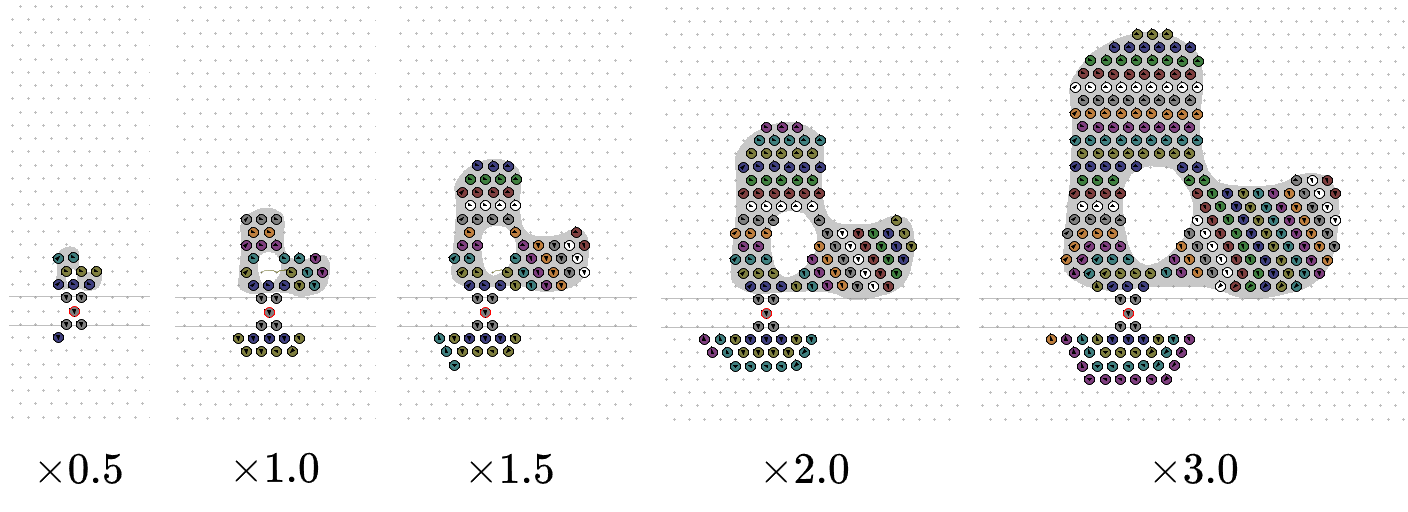}
    \captionsetup{justification=centering}
    \caption{Simulation result on the assembly of different-sized shapes}
    \label{fig:Simulation result on the assembly of different scaled shape}
    
\end{figure}

Nevertheless, more complicated shapes have been tested with success:
\begin{figure}[H]
    \centering
    \includegraphics[scale=0.5]{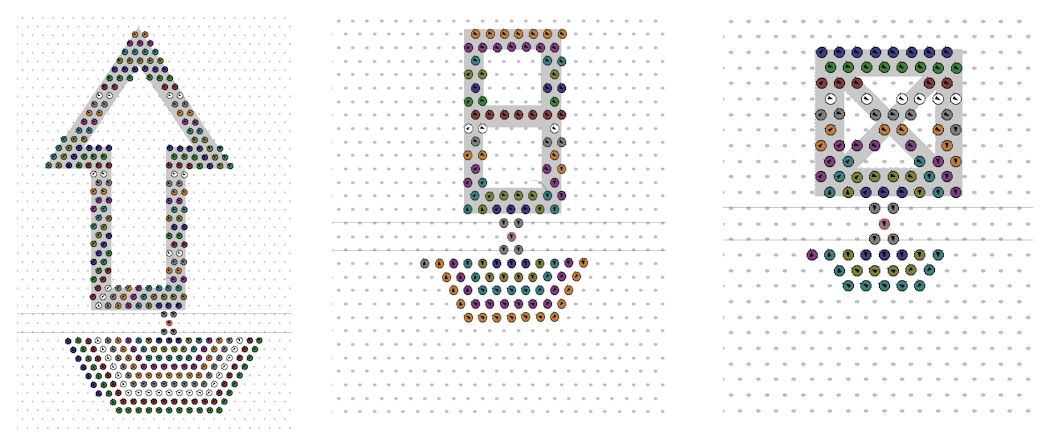}
    \captionsetup{justification=centering}
    \caption{Simulation result on the assembly of complicated shapes}
    \label{fig:Simulation on the assembly of complicated shapese}
    
\end{figure}

% However, for certain shapes, the algorithm fails:

\subsection{Limitation of the Algorithm}

In fact, the add-subtract algorithm works for all connected shapes $S$ which can be discretized into a proper ribbon structure. Roughly speaking, a "proper" ribbon structure of $S$ must ensure that each ribbon has a minimum length of two, and the lattice points on the boundary of $S$ must exhibit certain smoothness. Fig. \ref{fig:Simulation on the assembly of improper shapese} shows the simulation of the failure of forming a hollow $C$-shape, which lacks a smooth boundary. Notice that the robot $u_{142}$ (indicated by the red arrow) gets stuck at the inner boundary.
%  satisfy the following requirements:
% \begin{enumerate}
%     \item the ribbon forms a tree structure in $H_1$ and each ribbon has a minimum length of two; and
%     \item the set of lattice in $S$ is connected and has a smooth boundary: we require that for arbitrary line segment $l$ connecting two lattice points in $S$,  $l$ is located inside $S$ if $||l|| \leq 2d$.
% \end{enumerate}

\begin{figure}[H]
    \centering
    \includegraphics[scale=0.38]{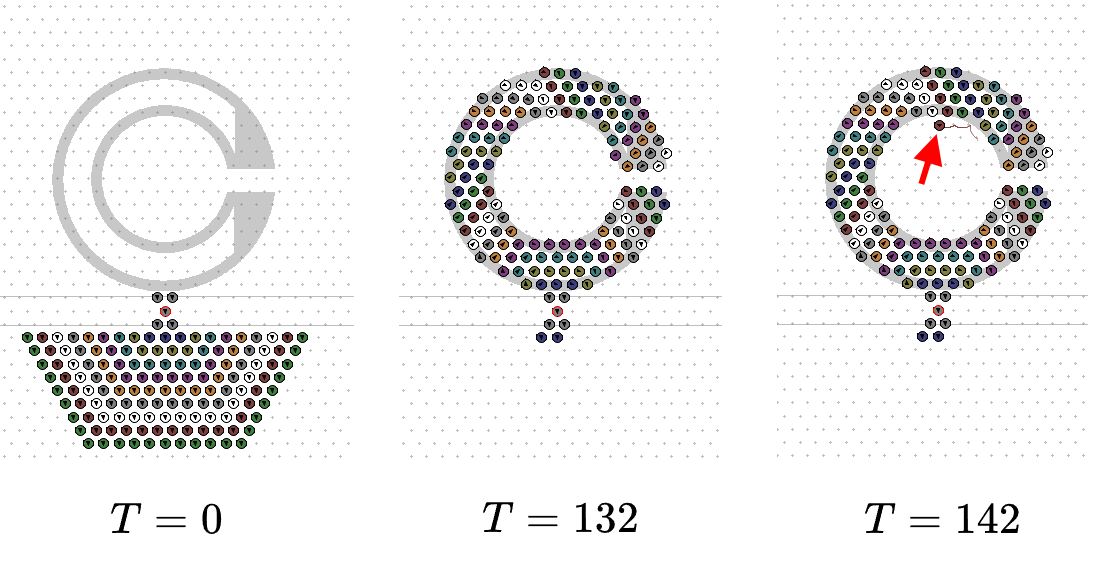}
    \captionsetup{justification=centering}
    \caption{Simulation result on the assembly of a shape without a smooth boundary}
    \label{fig:Simulation on the assembly of improper shapese}
    
\end{figure}

% The outline of the proof is given in the next section.

\subsection{Real World Application}
Suppose we aim to use UAVs to monitor the flow of people in Punggol Park, a community park located in the Hougang area of northeastern Singapore. Noticing that the park encompasses a lake, and no UAVs are required to be deployed over the lake. In this scenario, the add-subtract algorithm provides an efficient way to deploy the UAVs:

\begin{figure}[H]
    \centering
    \includegraphics[scale=0.5]{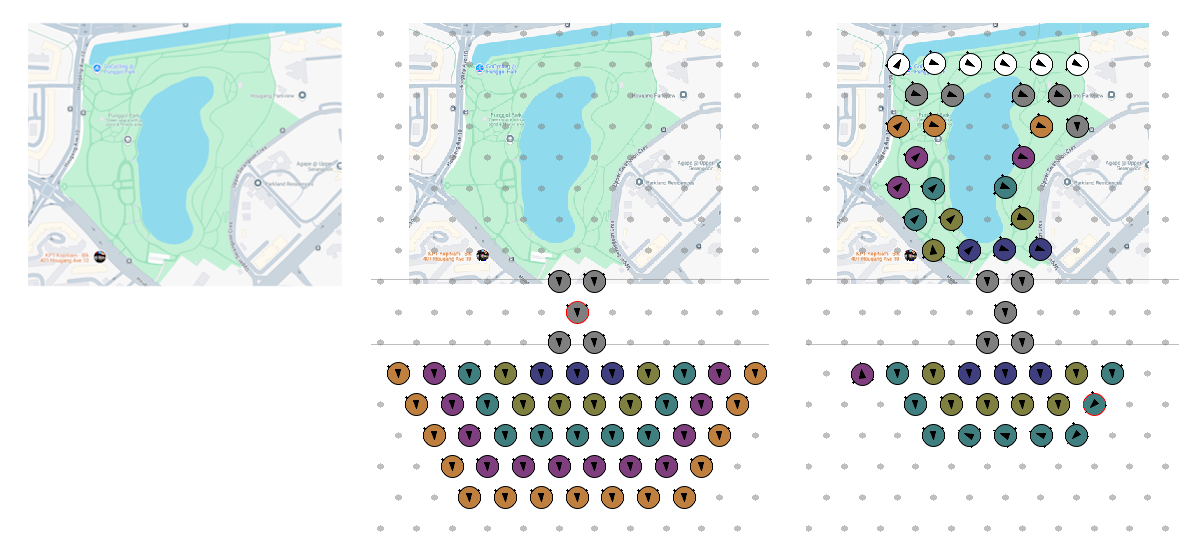}
    \captionsetup{justification=centering}
    \caption{Formation of drones over the Punggol park (highlighted in green \cite{googlemap})}
    \label{fig:punggol}
\end{figure}

\section{Conclusion and Future Works}
\label{sect-conclusion}
In this paper, we have formalized the idea of "ribbon" in the hexagonal lattice setting and investigated the structure induced by ribbons. We have shown that there exists one algorithm, i.e., the add-subtract algorithm, that can effectively solve the pattern formation problem in a more general setting, where a given shape may contain holes. The proposed algorithm works under the condition that there exists a centralized information sharing mechanism and a centralized epoch management mechanism to ensure that (1) each robot is fully aware of the shape and holes' information; (2) each robot knows the movement sequences; (3) there is only one robot moving in each epoch. Obviously, these assumptions are too strong to be scalable in real applications. Therefore, our immediate future work is to relax these assumptions and achieve the add-subtract algorithm in a fully distributed manner. There are still a few open questions worth studying: first, one could investigate ways to reduce the total agents required; second,  one could extend the algorithm to the formation of 3D shapes; and third, one could further analyze the ribbon structure to identify the optimal ordering that leads to a faster assembly.

\newpage
\begin{appendices}
% Reset counters for appendix
\setcounter{theorem}{0}
\setcounter{definition}{0}
\section{Algorithms}
\label{AppendiceA}

\begin{algorithm}[H]
    \caption{Localiztion Algorithm for robot v}\label{alg:Localiztion}
    \begin{algorithmic}
        \If{$v$ is a seed robot}
            
            \State $p_t(v) \gets$ predefined seed robot coordinates
        \ElsIf{$v$ is not moving}
            \State $p_t(v) \gets p_{t-1}(v)$
        \Else
            \State $U \gets \{u \in N_t(v)| p_t(u) \neq N.A.\}$
            \If{$\exists u_1, u_2, u_3 \in U$, such that $p_t(u_1), p_t(u_2), p_t(u_3)$ are noncollinear}
                \State $p_t(v) \gets argmin_{p_t(v)} (\sum_{u_i \in U}{|d(u_i,v){\renewcommand{\thefootnote}{\fnsymbol{footnote}}\footnotemark[1]} - || p_t(v) - p_t(u_i) |||})$
            \Else
                \State $p_t(v) \gets N.A.$
            \EndIf
        \EndIf
    \end{algorithmic}
\end{algorithm}
{\renewcommand{\thefootnote}{\fnsymbol{footnote}}\footnotetext[1]{Here $d(u_i,v)$ denotes the distance from $v$ to $u_i$ obtained from the sensor of robot $v$.}}

\begin{algorithm}[H]
    \caption{Gradient Formation Algorithm for robot v}\label{alg:Gradient}
    \begin{algorithmic}
        \If{$v$ is the root robot}
            \State $g(v) \gets 0$
        \Else
            \State $g(v) \gets min \{g(w) |w \in N_t(v)\} + 1$
        \EndIf
    \end{algorithmic}
\end{algorithm}

\begin{algorithm}[H]
    \caption{Edge-Following algorithm for robot v}\label{alg:edge-following}
    \begin{algorithmic}
        \If{$c(v) == N.A.$}
            \State $c(v) \gets w$, where $w$ is a random robot in $N_t(v)$
        \Else
            \If{further clockwise rotation around $c(v)$ will result in $||p_t(v) - p_t(w)|| < d$ for some $w \in N_t(v)$ }
                \State $c(v) \gets w$
            \Else
                \State rotate clockwise around $c(v)$ with radius $d$ 
            \EndIf
        \EndIf
    \end{algorithmic}
\end{algorithm}

\begin{algorithm}[H]
    \caption{Ribbon-Following algorithm for robot v}\label{alg:ribbon-following}
    \begin{algorithmic}
        \If{$c(v) == N.A.$}
            \State $c(v) \gets w$, where $w$ is a random robot in $ \{u \in N_t(v)|g(u) = g(v) - 1\}$
        \Else
            \If{further clockwise rotation around $c(v)$ will result in $||p_t(v) - p_t(w)|| < d$ for some $w \in \{u \in N_t(v)|g(u) = g(v) - 1\}$ }
                \State $c(v) \gets w$
            \Else
                \State rotate clockwise around $c(v)$ with radius $d$ 
            \EndIf
        \EndIf
    \end{algorithmic}
\end{algorithm}

\begin{algorithm}[H]
    \caption{Epoch-Update algorithm for robot v}\label{alg:epoch}
    \begin{algorithmic}
        \State $ T(v) = max \{T(u) | u \in N_t(v)\}$
        % \Comment{$T(u)$ refers to the epoch at which robot $u$ is in}
    \end{algorithmic}
\end{algorithm}

\newpage
\renewcommand{\algorithmicrequire}{\textbf{Input:}}
\renewcommand{\algorithmicensure}{\textbf{Output:}}
\algnewcommand{\algorithmicbreak}{\textbf{(Continued on next page...)}}
\newcommand{\algcontinue}{\textbf{(Continued from previous page)}}
\newcommand{\Break}{\State \textbf{break} }
\begin{algorithm}[H]
    \small
    \caption{Movement Sequence Generation}\label{alg:movement sequence 1}
    \begin{algorithmic}
        \Require the lattice points in $S$, the lattice points in $H_2$ which are occupied by idle robots
        \Ensure the movement sequences $s[k]$, $t[k]$, $s_r[k]$, $t_r[k]$
        \State $N \gets$ number of lattice points in $S$

        \State 
        \State $ribbonCount \gets 0$ \Comment{Generate the ribbon structure}
        \State $Tr \gets [[p_2, p_1]]$ 
        \State $A \gets$ the set of lattice points in $S$
        % \State sequence of ribbons $Tr$, with $Tr[0] \gets [p_2, p_1]$ 
        % \State $g(R[0]) \gets 1$
        % \State $h(p_2) \gets 1$
        % \State $h(p_1) \gets 1$
        \While{$A \neq \emptyset$}
            % \State $ribbonCount \gets ribbonCount + 1$
            % \State $R[ribbonCount] \gets (\text{   })$

            \For{all $parentRib \in Tr$}
                \For{all $x \in A$} \Comment{Identify the head of a ribbon}
                    \If{$x$ is adjacent to some lattice point $p$ in $parrentRib$, and
                    \\ \hspace*{5em}$x$ is adjacent to a lattice point $y \notin S$, and
                    \\ \hspace*{5em}$y$ is on the clockwise side of $x$ with respected to $p$ } 
                        \State $ribbonCount \gets ribbonCount + 1$
                        \State $Tr[ribbonCount] \gets [x]$
                        \State $A \gets A \setminus \{x\}$     
                        % \State $h(x) \gets h(p) + 1$
                        % \State $R[ribbonCount] \gets append(R[ribbonCount], x)$
                    \EndIf
                \EndFor

                \State  
                \While{$Tr[ribbonCount][last]$ is adjacent to \Comment{Identify subsequent points}
                \\ \hspace*{3.5em} some lattice point $x$ in $A$} 
                    \State $Tr[ribbonCount] \gets Tr[ribbonCount] + [x]$
                    \State $A \gets A \setminus \{x\}$                    
                \EndWhile                
            \EndFor
        \EndWhile

        \State 
        \For{$i = 1: ribbonCount$} \Comment{Count the number of lattice points in $D$}
            \State $C(Tr[i]) \gets 0$
            \For{$j = 1:len(Tr[ribbonCount])$}
                \If{$Tr[i][j] \in D$}
                    \State $C(Tr[i]) = C(Tr[i]) + 1$
                \EndIf
            \EndFor
            
        \EndFor

        \State
        \State $idleRibbonCount \gets 0$ \Comment{Generate the idle ribbon structure}
        \State $Ti \gets [[p_5, p_4]]$ 
        \State $B \gets$ the lattice points in $H_2$ which are occupied by idle robots
        % \State $g(R[0]) \gets 1$
        % \State $h(p_2) \gets 1$
        % \State $h(p_1) \gets 1$
        \While{$B \neq \emptyset$}
            % \State $ribbonCount \gets ribbonCount + 1$
            % \State $R[ribbonCount] \gets (\text{   })$

            \For{all $parentRib \in Ti$}
                \For{all $x \in B$} \Comment{Identify the head of an idle ribbon}
                    \If{$x$ is adjacent to some lattice point $p$ in $parrentRib$, and
                    \\ \hspace*{5em}$x$ is adjacent to a lattice point $y \notin H_2$, and
                    \\ \hspace*{5em}$y$ is at the clockwise side of $x$ with respected to $p$ } 
                        \State $idleRibbonCount \gets idleRibbonCount + 1$
                        \State $Ti[idleRibbonCount] \gets [x]$
                        \State $B \gets B \setminus \{x\}$     
                        % \State $h(x) \gets h(p) + 1$
                        % \State $R[ribbonCount] \gets append(R[ribbonCount], x)$
                    \EndIf
                \EndFor

                \State  
                \While{$Ti[idleRibbonCount][end]$ is adjacent to \Comment{Identify subsequent points}
                \\ \hspace*{3.5em} some lattice point $x$ in $B$} 
                    \State $Ti[idleRibbonCount] \gets Ti[idleRibbonCount] + [x]$
                    \State $B \gets B \setminus \{x\}$                    
                \EndWhile                
            \EndFor
        \EndWhile

        \State \algorithmicbreak
    \end{algorithmic}
\end{algorithm}

\addtocounter{algorithm}{-1}
\begin{algorithm}[H]
    \caption{Movement Sequence Generation}\label{alg:movement sequence 2}
    \begin{algorithmic}
        \State \algcontinue
        \State 
        \State sequence $s \gets [\text{ }]$ \Comment{Generate movement sequences $s$ and $r$}
        \State sequence $t \gets [\text{ }]$
        \For{$index \gets 1$ to $N$}
            % \State $n \gets index$
            % \For{$i \gets idleRibbonCount$ to $1$}
            %     \If{$n \leq size(Ti[i])$}
            %         \State $s[index] \gets Ti[i][len(Ti[i]) - n]$
            %         \Break
            %     \EndIf
            %     \State $n \gets n - size(Ti[i])$
            % \EndFor

            % \State $n \gets index$
            % \For{$i \gets 1$ to $size(Tr)$}
            %     \If{$n \leq size(Tr[i])$}
            %         \State $t[index] \gets Tr[i][n]$
            %         \Break
            %     \EndIf
            %     \State $n \gets n - size(Tr[i])$
            % \EndFor
            \State 
            \State $Found \gets false$ \Comment{Obtain starting position}
            \For{$i \gets idleRibbonCount$ to $1$}
                \If{$Found$}
                    \Break
                \EndIf
                \For{$j \gets size(Ti[i])$ to $1$}
                    \If{$Ti[i][j] \notin s$}
                        \State $s[index] \gets Ti[i][j]$
                        \State $Found \gets true$
                        \Break
                    \EndIf
                \EndFor
            \EndFor

            \State
            \State $Found \gets false$ \Comment{Obtain stopping position}
            \For{$i \gets 1$ to $ribbonCount$}
                \If{$Found$}
                    \Break
                \EndIf
                \For{$j \gets 1$ to $size(Tr[i])$}
                    \If{$Tr[i][j] \notin t$}
                        \State $t[index] \gets Tr[i][j]$
                        \State $Found \gets true$
                        \Break
                    \EndIf
                \EndFor
            \EndFor

        \EndFor
        
        \State
        \State \algorithmicbreak
    \end{algorithmic}
\end{algorithm}

\addtocounter{algorithm}{-1}
\begin{algorithm}[H]
    \caption{Movement Sequence Generation}\label{alg:movement sequence 3}
    \begin{algorithmic}
        \State \algcontinue

        % \State 
        % \State sequence $s \gets [\text{ }]$ \Comment{Generate movement sequences}
        % \State sequence $t \gets [\text{ }]$
        % \For{$index \gets 1$ to $N$}
        %     % \State $n \gets index$
        %     % \For{$i \gets idleRibbonCount$ to $1$}
        %     %     \If{$n \leq size(Ti[i])$}
        %     %         \State $s[index] \gets Ti[i][len(Ti[i]) - n]$
        %     %         \Break
        %     %     \EndIf
        %     %     \State $n \gets n - size(Ti[i])$
        %     % \EndFor

        %     % \State $n \gets index$
        %     % \For{$i \gets 1$ to $size(Tr)$}
        %     %     \If{$n \leq size(Tr[i])$}
        %     %         \State $t[index] \gets Tr[i][n]$
        %     %         \Break
        %     %     \EndIf
        %     %     \State $n \gets n - size(Tr[i])$
        %     % \EndFor
        %     \State $Found \gets false$
        %     \For{$i \gets idleRibbonCount$ to $1$}
        %         \If{$Found$}
        %             \Break
        %         \EndIf
        %         \For{$j \gets size(Ti[i])$ to $1$}
        %             \If{$Ti[i][j] \notin s$}
        %                 \State $s[index] \gets Ti[i][j]$
        %                 \State $Found \gets true$
        %                 \Break
        %             \EndIf
        %         \EndFor
        %     \EndFor

        %     \State $Found \gets false$
        %     \For{$i \gets 1$ to $ribbonCount$}
        %         \If{$Found$}
        %             \Break
        %         \EndIf
        %         \For{$j \gets 1$ to $size(Tr[i])$}
        %             \If{$Tr[i][j] \notin t$}
        %                 \State $t[index] \gets Tr[i][j]$
        %                 \State $Found \gets true$
        %                 \Break
        %             \EndIf
        %         \EndFor
        %     \EndFor

        % \EndFor
        
        \State
        \State sequence $s_r \gets [\text{ }]$ \Comment{Generate movement sequences $s_r$ and $t_r$}
        \State sequence $t_r \gets [\text{ }]$
        \For{$index \gets 1$ to $N$}
            % \State $n \gets index$
            % \For{$i \gets ribbonCount$ to $1$}
            %     \If{$n \leq size(Tr[i])$}
            %         \State $s_r[index] \gets Tr[i][n]$
            %         \Break
            %     \EndIf
            %     \State $n \gets n - size(Tr[i])$
            % \EndFor
            \State
            \State $Found \gets false$ \Comment{Obtain starting position}
            \For{$i \gets ribbonCount$ to $1$}
                \If{$Found$}
                    \Break
                \EndIf
                \For{$j \gets 1$ to $size(Tr[i])$}
                    \If{$Ti[i][j] \notin s_r$}
                        \State $s_r[index] \gets Tr[i][j]$
                        \State $Found \gets true$
                        \Break
                    \EndIf
                \EndFor
            \EndFor
                
            \State
            \If{$s_r[index] \in Tr[i][1:C]$}   \Comment{Obtain stopping position for recycled robots}
                
                \State $Found \gets false$
                \For{$i \gets 1$ to $idleRibbonCount$}
                    \If{$Found$}
                        \Break
                    \EndIf
                    \For{$j \gets 1$ to $size(Ti[i])$}
                        \If{$Ti[i][j] \in s$ and $Ti[i][j] \notin tr$}
                            \State $t_r[index] \gets Ti[i][j]$
                            \State $Found \gets true$
                            \Break
                        \EndIf
                    \EndFor
                \EndFor
                
            \Else   \Comment{Obtain stopping position for rearranged robots}

                \State $Found \gets false$
                \For{$i \gets ribbonCount$ to $1$}
                    \If{$Found$}
                        \Break
                    \EndIf
                    \For{$j \gets 1$ to $size(Tr[i])$}
                        \If{$Ti[i][j] \notin D$ and $Ti[i][j] \notin tr$}
                            \State $t_r[index] \gets Tr[i][j]$
                            \State $Found \gets true$
                            \Break
                        \EndIf
                    \EndFor
                \EndFor
            \EndIf

        \EndFor

    \end{algorithmic}
\end{algorithm}

\begin{algorithm}[H]
    \caption{Add-subtract algorithm for robot $v$}\label{alg: add-sub}
    \begin{algorithmic}
        \Require $s[k],t[k],s_r[k],t_r[k]$, rotation radius $d$, synchronized time $t${\renewcommand{\thefootnote}{\fnsymbol{footnote}}\footnotemark[1]} and warm-up time $t_{warmup}$
        \Ensure The robot stops at the desired position, either inside $S\setminus D$ or $H_2$

        \State

        \State $t_{start} \gets t$ \Comment{Start of Initialization}
        \State $state(v) \gets STOP$ 
        \While{$t < t_{start}+t_{warmup}$}
            \State Call \textsc{Localization}()
            \State Call \textsc{Gradient formation}()  
        \EndWhile
    
        \State
        \State $T(v) \gets 1$ \Comment{Start of Iteration}
        
        \While{$T(v) \leq 2N$ } 
            \If{$T(v) < N$}
                \If{$p_t(v) = s[T]$ and $state(v) = STOP$}
                    \State $state(v) \gets EDGE$   
                \EndIf
                
                \If{$p_t(v) = t[T]$}
                    \State $state(v) \gets STOP$
                    \State $T(v) \gets T(v) + 1$
                \EndIf
    
            \Else
                
                \If{$p_t(v) = s_r[T]$ and $state(v) = STOP$}
                
                    \State $state(v) \gets RIBBON$   
                    
                \EndIf

                \If{$p_t(v) \in L$ and $ p_t(v) \notin S$}
                    \State $state(v) \gets EDGE$
                \EndIf
                
                \If{$p_t(v) = t[T]$}

                    \State $state(v) \gets STOP$
                    \State $T(v) \gets T(v) + 1$
                
                \EndIf
            \EndIf

            \State
            \State Call \textsc{Epoch update}() 
            \State Call \textsc{Localization}()
            \State Call \textsc{Gradient formation}() 
            \State            
            \If{$state(v) = EDGE$}
                \State Call \textsc{Edge-Following}()
            \EndIf
            \State
            \If{$state(v) = RIBBON$}
                \State Call \textsc{Ribbon-Following}()
            \EndIf

        \EndWhile

    \end{algorithmic}
\end{algorithm}
{\renewcommand{\thefootnote}{\fnsymbol{footnote}}\footnotetext[1]{Time $t$ is synchronized between all robots of the swarm at the start of the algorithm and updated locally while executing the algorithm. Thus, for simplicity, we assume that all the robots of the swarms share a synchronized clock.}}

% % Change numbering for appendix
% \renewcommand{\thetheorem}{A\arabic{theorem}} % Prefix theorem numbering with "A"
% \renewcommand{\thedefinition}{A\arabic{definition}} % Prefix definition numbering with "A"

\newpage
\section{Ribbon and its Properties}
\label{sect-Ribbonization}

In appendix \ref{sect-Ribbonization}, we discuss the ribbonization procedure: ribbonization essentially partitions the shape $S$ (represented by a set of discrete lattice points) as layers stacking on one another. Each layer is captured by a set of lattice points in $S$, and each set of points forms a line graph if we impose an edge between adjacent lattice points. We call the path graphs induced by such sets (or equivalently, the order list of vertex points of the path graphs) "ribbons". There also exists a tree structure between ribbons, and this property is essential when we define the movement sequence in section \ref{Section-MovingSequence} and section \ref{sect-extendedSequence}.

For generalization purposes, we start with an arbitrary shape $P_0$ that is homeomorphic to an open disk. Suppose its corresponding discrete representation is $P$. To partition the lattice points in $P$, we first introduce the shape $P$ into the workplace with the underlying hexagonal lattice, and certain requirements are imposed to ensure a "proper" placement. Then, we introduce the hop count defined on every lattice point in $P$ that denotes the distance from the lattice point to the origin. This hop count value depends on the shape of $P$ in the sense that it represents the length of the shortest path within $P$. The properties of hop count are discussed in section \ref{sect-hopCount}. We introduce ribbons as adjacent lattice points in $P$ with the same hop count value. Ribbons are defined in section \ref{sect-defRibbon} and the ribbonization of shape $P$ is investigated in section \ref{sect-ribbonizationOfP}. Lastly, we consider the two special cases for $P$: in section \ref{sect-ribbonizationOfS}, we let $P$ be the user-defined shape $S$ and we obtain ribbons of $S$; in section \ref{sect-ribbonizationOfH2}, we let $P$ to be the set of lattice points in $H_2$ and we obtain the idle ribbons.

Please note that the ribbon idea is adapted from \cite{rubenstein_programmable_2014}, and proposition \ref{prop-10+1} to \ref{prop-clique} follows the proof given in the supplementary materials of \cite{rubenstein_programmable_2014}.
    
\subsection{Hop count}
\label{sect-hopCount}

    When we introduce the shape $P$ into the workspace, certain conditions need to be satisfied:
    \begin{definition}
        \label{def-properPlacement}
        The placement of $P$ in the coordinate system $\mathbb{O}$ is \textbf{proper} if  $P$ is placed next to the origin but points $O,p_1,p_2,p_3,p_4,p_5,p_6$ are not in $P$.
    \end{definition}

    In section \ref{sect-hopCount} and section \ref{sect-defRibbon}, we shall assume the placement of $P$ is proper without separate mention.

    \begin{definition}
        For arbitrary lattice point $x \in P$, the \textbf{hop count} for $x$ is the length of the shortest path from $x$ to the origin $O$ in $\Gamma$ that involves only lattice points in $P$ and lattice points in $\{p_1,p_2,\ldots,p_6\}$. If such a path does not exist, the hop count is defined as $-1$. The hop count of $x$ is denoted as $h(x)$. 
    \end{definition}
    \begin{figure}[H]
        \centering
        \includegraphics[scale=0.3]{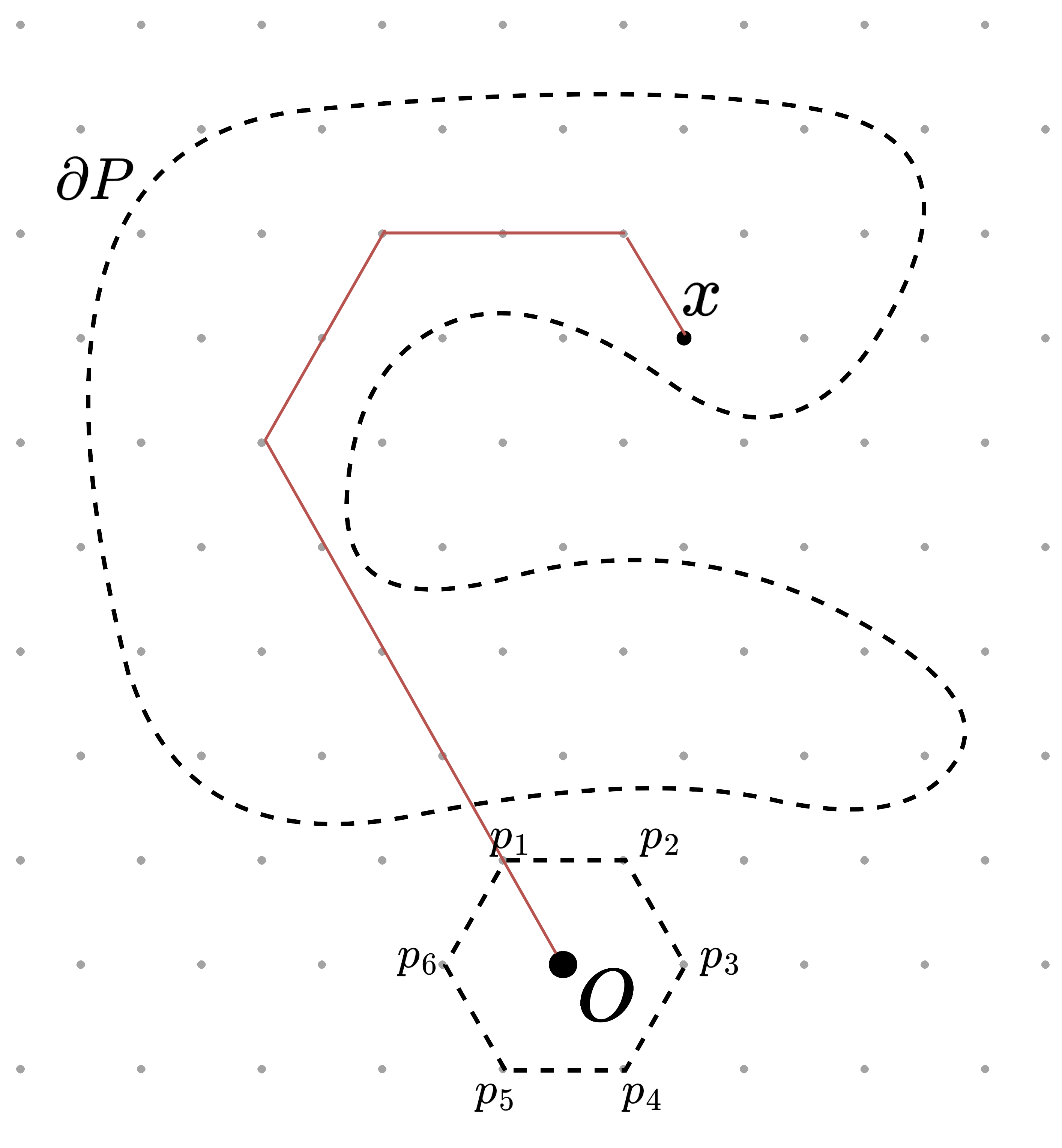}
        \caption{Hop count for $x$ is 10}
    \end{figure}

    \begin{definition}
        %The \textbf{neighbors} of a lattice point $x$ is the lattice point(s) of $L$ that is/are at a distance $d$ to $x$; the \textbf{neighborhood} of $x$ refers to the collection of its neighbors.
        The \textbf{neighbors} of a lattice point $x$ is the six lattice points that are at a distance $d$ to $x$; the \textbf{neighborhood} of $x$ refers to the collection of its neighbors.
    \end{definition}

    Note that a lattice point has a maximum of six neighbors.

    \begin{definition}
        A \textbf{straight line path} is a path in $\Gamma$ that consists of lattice points in a straight line.
    \end{definition}    

    \begin{proposition}
        \label{prop-10+1}
        For any lattice point $x \in L$ with hop count $k$, its neighbors can only have hop count $k-1, k, k+1$.
    \end{proposition}
    
    \begin{proposition}
        \label{prop1}
        Any lattice point $x$ with a positive hop count $k$ has a neighbor with hop count $k-1$.
    \end{proposition}
    \begin{proof}
        Since $x$ has a positive hop count, there exists a shortest path from $v$ to the origin $O$. This path must go through one of its neighbors with hop count $k-1$.
    \end{proof}
    
    \begin{proposition}
        \label{prop2}
        For any lattice point $x$ with a positive hop count $k$, $v$ cannot have two non-adjacent neighbors with hop count $k-1$.
    \end{proposition}
    \begin{proof}
        Suppose $x$ has two neighbors $y,z$ with hop count $k-1$, and $y,z$ are not neighbors. Then there exist two distinct shortest paths from $x$ to $O$ such that $path_y$ goes through $x,y,O$ and $path_z$ goes through $x,z,O$. Since $path_y$ and $path_z$ surround some lattice points in $P$, there must exist a straight line path going through $v$ that intersects $path_y$ or $path_z$. This contradicts $path_y$ and $path_z$ both being the shortest path.

        % \begin{figure}[H]
        %     \centering
        %     \includegraphics[scale=0.5]{figures/proof-canonical/proofCanonical-prop-k-1.png}
        %     \caption{Illustration for proposition \ref{prop2}}
        % \end{figure}
    \end{proof}
    
    \begin{proposition}
        \label{prop3}
        For any lattice point $x$ with a positive hop count $k$, $x$ either has one neighbor with hop count $k-1$ or has two adjacent neighbor robots with hop count $k-1$.
    \end{proposition}
    
    \begin{proof}
        Suppose not, then either $x$ has more than two neighbors with hop count $k-1$, or $v$ has two non-adjacent neighbors with hop count $k-1$. In either case, there exist non-adjacent neighbors with hop count $k-1$, which contradicts proposition \ref{prop2}.
    \end{proof}
    
    \begin{proposition}
        \label{prop-clique}
        A clique of three lattice points can't have the same hop count if their hop counts are positive.
    \end{proposition}
    
    \begin{proof}
        Suppose not, then there is a clique of three lattice points $u,v,w$ that have the same hop count $k$. We denote the shortest path from $u$ to $O$ as $path_u$, the shortest path from $w$ to $O$ as $path_w$. Since $path_u$ and $path_w$ cannot go through $v$, $v$ is either surrounded by $path_u$ and $path_w$ or outside the region enclosed by $path_u$ and $path_w$.
    
        \begin{figure}[H]
            \centering
            \includegraphics[scale=0.5]{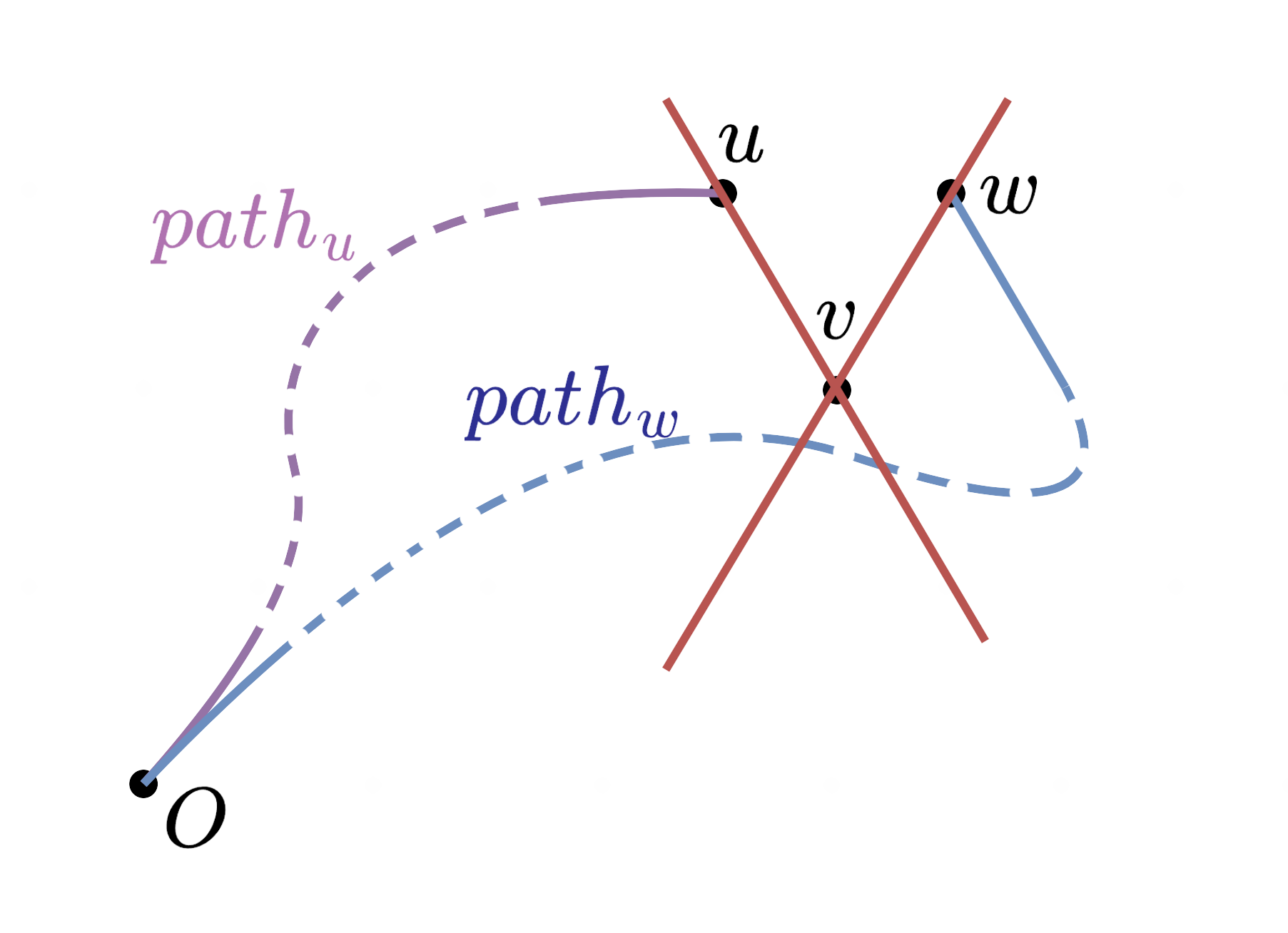}
            \caption{Illustration for proposition \ref{prop-clique}}
        \end{figure}

        If $v$ is surrounded by $path_u$ and $path_w$, the straight line path starts from $u$ in the direction of $v$, and the straight line path starts from $w$ in the direction of $v$ must intersect $path_u$ or $path_w$. It contradicts $path_u$ being the shortest path or $path_w$ being the shortest path, respectively.
    
        If $v$ is not surrounded by $path_u$ and $path_w$, we denote the shortest path from $v$ to $O$ as $path_v$. Then, either $path_v$ and $path_u$ surrounds $w$, or $path_v$ and $path_w$ surrounds $u$. By similar arguments, we reach a contradiction. 
    \end{proof}

    \begin{proposition}
        \label{prop-typeOfNeighborhood}

        Given a lattice point $x \in L$ with a positive hop count $k$, the neighborhood of $x$ can only be in one of the two types:
        \begin{itemize}
            \item [-] Type $A$: $x$ is adjacent to one lattice point $y$ with hop count $k-1$. The lattice point(s) in $P$ that are adjacent to both $x$ and $y$ has the hop count equal to $k$. The remaining neighbors in $P$, if any, have the hop count equal to $k+1$.
            \item [-] Type $B$: $x$ is adjacent to two adjacent lattice point $y,z$ with hop count $k-1$. The lattice point(s) in $P$ that is adjacent to both $x$ and $y$ or adjacent to both $x$ and $z$ has the hop count equal to $k$. The remaining neighbors in $P$, if any, have the hop count equal to $k+1$.
        \end{itemize}
    \end{proposition}
    \begin{figure}[H]
        \centering
        \includegraphics[scale=0.5]{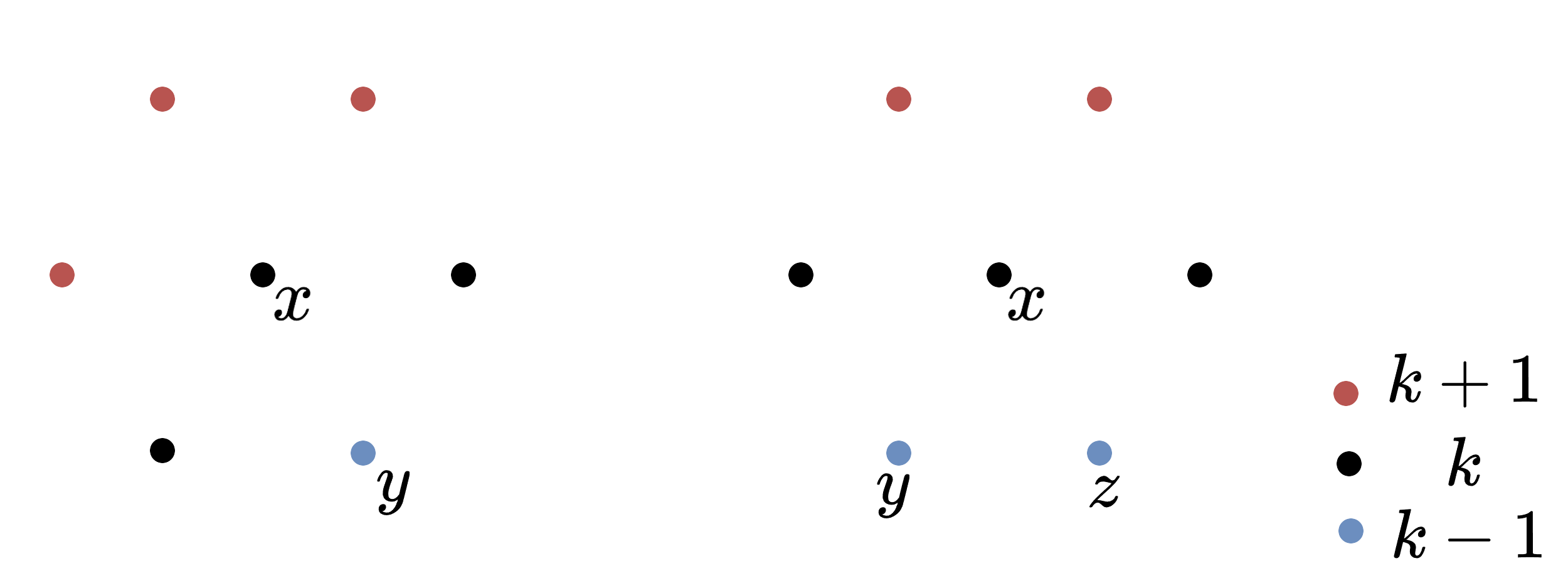}
        \caption{Left: type $A$ neighborhood of $x$; Right: type $B$ neighborhood of $x$}
    \end{figure}
    \begin{proof}
        By proposition \ref{prop3}, $x$ has one neighbor with hop count $k-1$ or has two adjacent neighbor robots with hop count $k-1$:
        \begin{itemize}
            \item case 1: $x$ has only one neighbor $y$ with hop count $k-1$. We shall inspect the six neighbors of $x$.
            % \begin{figure}[H]
            %     \centering
            %     \includegraphics[scale=0.5]{figures/proof-canonical/proofCanonical-type of neighborhood proof.png}
            % \end{figure}            
            Denote the two lattice points adjacent to both $x$ and $y$ as $a$ and $b$. If $a,b$ are inside $P$, then they could have the hop count equal to $k$ or $k+1$ by proposition \ref{prop-10+1}. Since they are adjacent to $y$ with $h(y)=k-1$, $a,b$ must have hop count $k$. Denote the lattice point adjacent to $a,x$ as $c$ and the lattice point adjacent to $b,x$ as $d$. Suppose $c,d$ locate in $P$, then by proposition \ref{prop-clique}, $h(c) \neq k$ and $h(d) \neq k$ so $h(c) = h(d) = k+1$. Denote the last lattice point as $e$ and assume $e \in P$. By applying proposition \ref{prop2} on lattice point $c$, we deduce that $h(e) \neq k$, so $h(e) = k+1$. We conclude that the neighborhood of $x$ in case 1 is precisely the type A neighborhood of $x$.

            \item case 2: $x$ has only two adjacent neighbors $y,z$ with hop count $k-1$. By similar arguments, we deduce the neighborhood of $x$ in case 2 is precisely the type B neighborhood of $x$.
        \end{itemize}
    \end{proof}

\subsection{Ribbon}
\label{sect-defRibbon}

    \begin{definition}
        Let $G$ be the subgraph of $\Gamma$ induced by the sets of lattice points with the same hop count. An \textbf{undirected ribbon} $R^*$ is defined as a component of $G$. The \textbf{gradient} of $R^*$ is the hop count of any lattice point belonging to the ribbon $R^*$.
        % Let $K$ denote the set consisting of all the lattice points in $P$ with the same hop count, and let $G$ be the subgraph of $\Gamma$ induced by $K$. An \textbf{undirected ribbon} $R^*$ is defined as a component of $G$. The \textbf{gradient} of $R^*$ is the hop count of any lattice point belonging to the ribbon $R^*$.
    \end{definition}

    % \begin{figure}[H]
    %     \centering
    %     \includegraphics[scale=0.3]{figures/proof-canonical/proofCanonical-defOfRibbon.png}
    %     \caption{$R_i$ and $R_j$ are the components of the subgraph induced by lattice points with hop count $4$.}
    % \end{figure}

    \begin{proposition}
        All undirected ribbons are path graphs.
    \end{proposition}
    \begin{proof}
        Let $R^*$ be an undirected ribbon with gradient $g(R^*)=k$, then there exists a graph $G$ induced by all the lattice points with hop count $k$, and $R^*$ is a component of $G$. By proposition \ref{prop-typeOfNeighborhood}, every vertex of $R^*$ cannot be adjacent to more than two vertices of $R^*$, so $R^*$ is either a path graph or a cycle graph. Suppose $R^*$ is a cycle graph, then 
        \begin{itemize}
            \item [-] case 1: the lattice points surrounded by $R^*$ has hop count less than $k$, then the shortest path from a point $x \in V(R^*)$ to the origin ends in the interior of $R^*$. This contradicts the origin $O$ being a point outside $P$.
            \item [-] case 2: the lattice points surround by $R^*$ has hop count larger than $k$. Notice that the edges of $R^*$ form a polygon, and we denote the polygon as $Y$. $Y$ has $|V(R^*)|$ vertices so the sum of interior angles of $Y$ is $(|V(R^*)| * \pi - 2*\pi)$ (by the interior angle formula). However, the interior angle at each vertex is either $\frac{4\pi}{3}$ (for type $A$ neighborhood) or $\pi$ (for type $B$ neighborhood). By summing up all the interior angles at each vertex of $R^*$, we obtain a sum that is larger than $|V(R^*)| * \pi$. This contradicts the result derived from the interior angle formula.
        \end{itemize}
        Since $R^*$ is not a cycle graph, $R^*$ must be a path graph.
    \end{proof}

    For each undirected ribbon $R^*$, we impose direction on its edges to get ribbon $R$:
    \begin{definition}
        For a given undirected ribbon $R^*$, its corresponding directed ribbon $R$ is a directed graph such that for each vertex $x \in R$, the arc entering $x$ and the arc leaving $x$ is defined such that:
        \begin{figure}[H]
            \centering
            \includegraphics[scale=0.5]{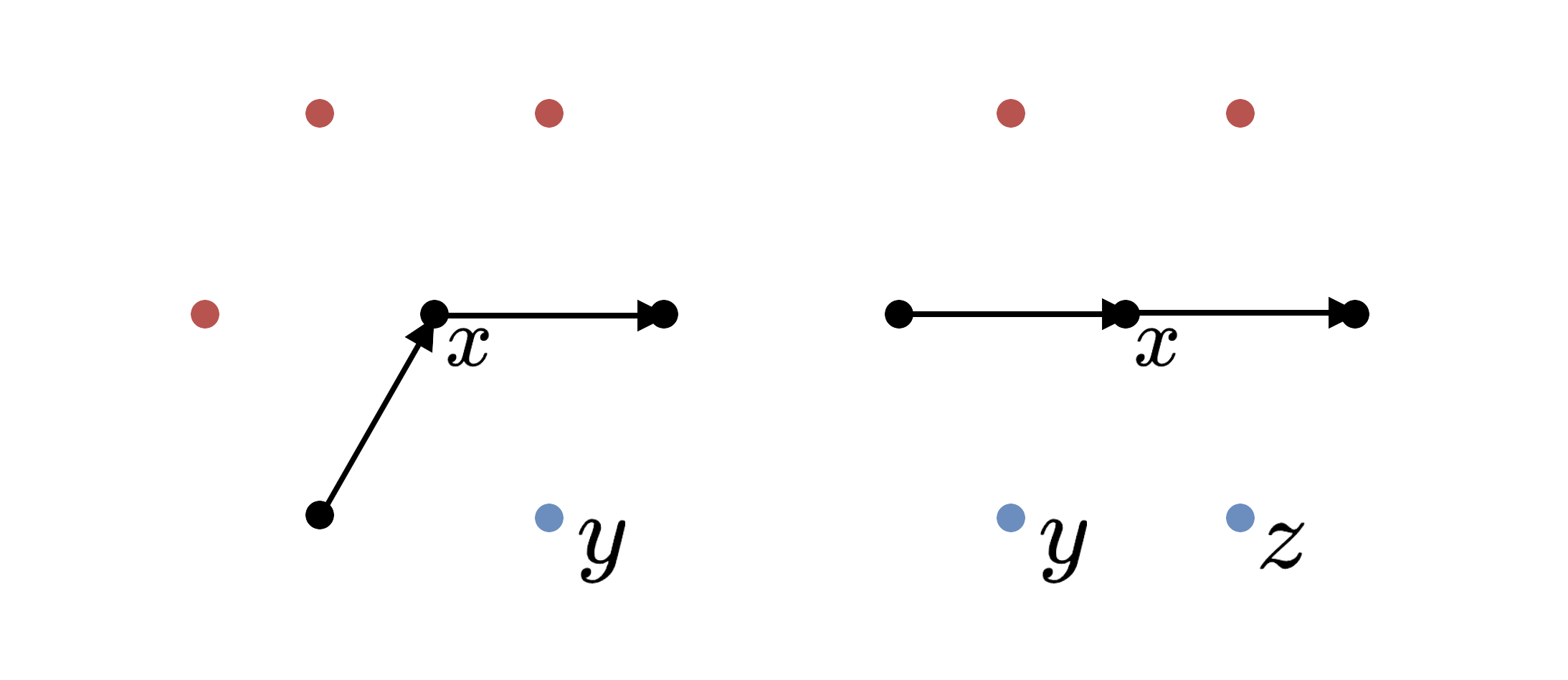}
        \end{figure}
        \begin{itemize}
            \item [-] if $x$ has the type $A$ neighborhood, the two arcs are directed clockwise around $y$;
            \item [-] if $x$ has the type $B$ neighborhood, the arc entering $x$ is directed clockwise around $y$, and the arc leaving $x$ is directed clockwise around $z$.
        \end{itemize}
        The \textbf{gradient} of $R$ is the hop count of any lattice point belonging to the ribbon $R$. Denote as $\bm{g(R)}$. The \textbf{length} of $R$ is the number of vertices in $R$, i.e., $len(R) := |V(R)|$. 
    \end{definition} 

    \begin{proposition}
        A ribbon is a directed path graph.
    \end{proposition}
    \begin{proof}
        By the definition of directed ribbon, $R$ has an underlying undirected path graph $R^*$. We observe that any end vertex of $R$ has a total degree of one, and any non-end vertex in $R$ has an in-degree of one and an out-degree of one.
    \end{proof}

    \begin{definition}
        \label{def-ribbonBehind}
        For a ribbon $R$ and two vertices $x,y \in V(R)$,
        % \begin{figure}[H]
        %     \centering
        %     \includegraphics[scale=0.5]{figures/proof-canonical/proofCanonical-def-GeneralRibbon-inFrontOf.png}
        %     \caption{Illustration of definition \ref{def-ribbonBehind}}
        % \end{figure}
        \begin{itemize}
            \item [-] $x$ is \textbf{in front of} $y$ iff there is a directed path from $y$ to $x$ in $R$, denoted as $x \rightarrow y$;
            \item [-] $x$ is \textbf{behind} $y$ iff there is a directed path from $x$ to $y$ in $R$, denoted as $y \rightarrow x$;
            \item [-] $x$ is \textbf{the first vertex} of the ribbon $R$ if no vertex is in front of $x$, and the first vertex of ribbon $R$ is denoted as $L(R)$;
            \item [-] $x$ is \textbf{the last vertex} of the ribbon if no vertex is behind it.
        \end{itemize}
    \end{definition}

    Notice that for a pair of vertices $x,y$ of $R$, $x$ is either in front of $y$ or behind $y$; and $x$ could not be in front of $y$ and behind $y$ at the same time (otherwise $x,y$ involves in a loop which contradicts $R$ being a directed path graph).

\subsection{Ribbons Induced by $P$}
\label{sect-ribbonizationOfP}

    \begin{definition}
        The \textbf{ribbonization} of $P$ is the collection of all ribbons with vertices in $P$, denoted as $\hat{P} = \{R_0, R_1, \ldots\}$. 
    \end{definition}

    \begin{definition}
        \label{def-properRibbonization}
        A \textbf{proper ribbonization} $\hat{P}$ satisfies the following properties:
        \begin{enumerate}
            \item $P$ is properly placed in $\mathbb{O}$;
            \item every lattice point in $P$ has a positive hop count, that is,
                \[ (\forall x \in P) h(x)>0  \]
            \item $(\forall R \in \hat{P}$) $len(R) \geq 2$;
            \item only one ribbon with the gradient value equal to two, and the ribbon has a length of three, that is,
                \[ (\exists! R \in \hat{P}) g(R)=2 \land len(R)=3\]
            \item $P$ has a smooth boundary in the sense that
                \[ (\forall x,y \in P) (||x-y||=2d \Rightarrow ((\forall z \in L) ||x-z||=||y-z||=d \Rightarrow z \in P ))\] 
                and
                \[ (\forall x,y \in P) (||x-y||=\sqrt{3} \Rightarrow \neg((\forall z \in L) ||x-z||=||y-z||=d \Rightarrow z \notin P ))\]

            % \item the line segment $l$ connecting any two lattice point in $S$ must locate inside $S$ unless $||l||>2d$, that is
            % \[ (\forall x,y \in S \cap L) ||x-y|| \leq 2d \Rightarrow (line(x,y) \subset S \setminus D)\]

        \end{enumerate}

        We denote the ribbon with a gradient of $2$ as $R_0$. This ribbon serves as the root of the tree of ribbons (which we shall see very soon).
    \end{definition}

    \begin{proposition}
        \label{prop-pointBelongsToUniqueRibbon}
        Every lattice point in $P$ belongs to exactly one ribbon in $\hat{P}$.
    \end{proposition}
    \begin{proof}
        Since every lattice point $x \in P$ has a finite hop count, $x$ belongs to a ribbon. Now, we show that $x$ cannot belong to more than one ribbon.

        Suppose lattice point $x \in P$ belongs to more than one ribbon, and we denote two of the ribbons as $R$ and $Q$. Then $g(R) = h(x) = g(Q)$. The lattice points in $Q$ and $R$ have the same hop count, and $Q$ and $R$ share the same lattice point x. As a result, $Q$ and $R$ are the same components of the subgraph induced by lattice points with hop count $h(x)$. This contradicts $Q$ and $R$ are different ribbons.
    \end{proof}

    \begin{definition}
        For a lattice point $x \in P$, define $R(x)$ as the ribbon such that $x \in V(R(x))$.
    \end{definition}

    We define the relationships between two ribbons as follows:
    \begin{definition}
        For two ribbons $R_i$ and $R_j$ in $\hat{P}$,
        \begin{itemize}
            %\item[-] A \textbf{sub-ribbon} of $R_i$ is a subgraph of $R_i$ which is also a path graph;
            % \item[-] $R_j$ is a \textbf{proper sub-ribbon} of $R_i$ iff $R_j$ is a proper subgraph of $R_i$, denoted as $R_i \subset R_j$;
            \item[-] $R_i$ is a \textbf{parent ribbon} of $R_j$ if 
                \begin{enumerate}
                    \item $g(R_i) = g(R_j) - 1$
                    \item $\forall x \in V(R_i)$ $\exists y \in V(R_j)$ $||x - y|| =d$.
                \end{enumerate}
                Denoted as $R_i = R_j^{[p]}$.
            %\item[-] $R_i$ is a \textbf{ancestor ribbon} of $R_j$ iff there is a chain $(R_i, R_{i+1}, \ldots, R_{j-1}, R_{j})$ such that $R_i$ is a parent ribbon of $R_{i-1}$ ... $R_{j-1}$ is a parent of $R_j$
            \item[-] $R_i$ is a \textbf{child ribbon} of $R_j$ if $R_j$ is a parent ribbon of $R_i$;
            \item[-] $R_i$ is a \textbf{descendant ribbon} of $R_j$ if there is a chain $(R_i, R_{i+1}, \ldots, R_{j-1}, R_{j})$ such that $R_i$ is a child ribbon of $R_{i-1}$, ..., $R_{j-1}$ is a child ribbon of $R_j$. We also denote $R_j$ as an \textbf{ancestor ribbon} of $R_i$, denoted as $R_j \twoheadrightarrow R_i$.
            % \item[-] $R_i$ is \textbf{in front of} $R_j$ iff $\exists R$ where $R_i \subseteq R$ and $R_j \subseteq R$, all vertices of $R_i$ is in front of $R_j$.
        \end{itemize}
    \end{definition}

    \begin{proposition}
        \label{prop-uniqueParent}
        Except for $R_0$, every ribbon has a unique parent ribbon.
    \end{proposition}
    \begin{proof}
        Given a ribbon $R = (v_0,v_1,...,v_n)$\footnote{We use the sequence of vertices to represent a directed ribbon. $v_0$ is the first vertex of the ribbon. We also assume $R$ has a finite length because there are finite lattice points with the same hop count.}, and suppose $g(R) = k$ with $k \geq 1$. We shall inspect the neighborhood of each vertex of $R$. Starting from $v_0$:

        \begin{figure}[H]
            \centering
            \includegraphics[scale=0.5]{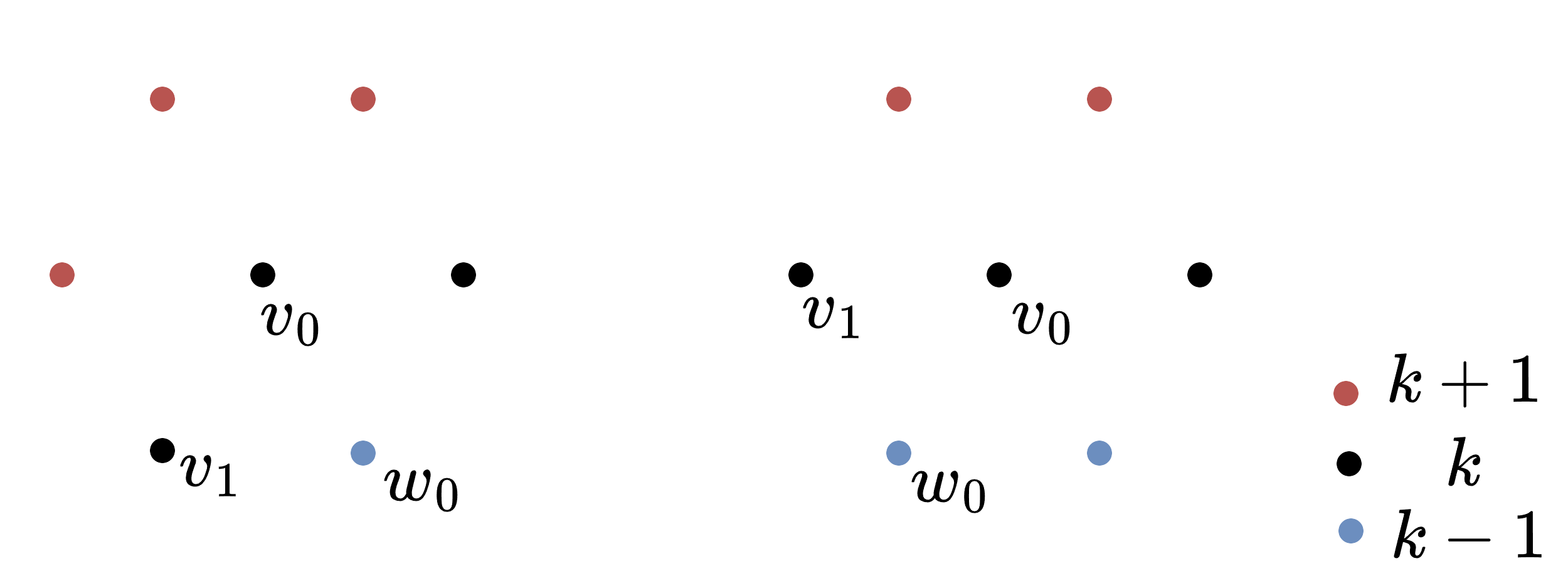}
        \end{figure}

        For both type $A$ and type $B$ neighborhood of $v_0$, we denote the lattice point adjacent to both $v_0, v_1$ as $w_0$. Then we inspect the neighborhood of $v_1$:

        \begin{figure}[H]
            \centering
            \includegraphics[scale=0.5]{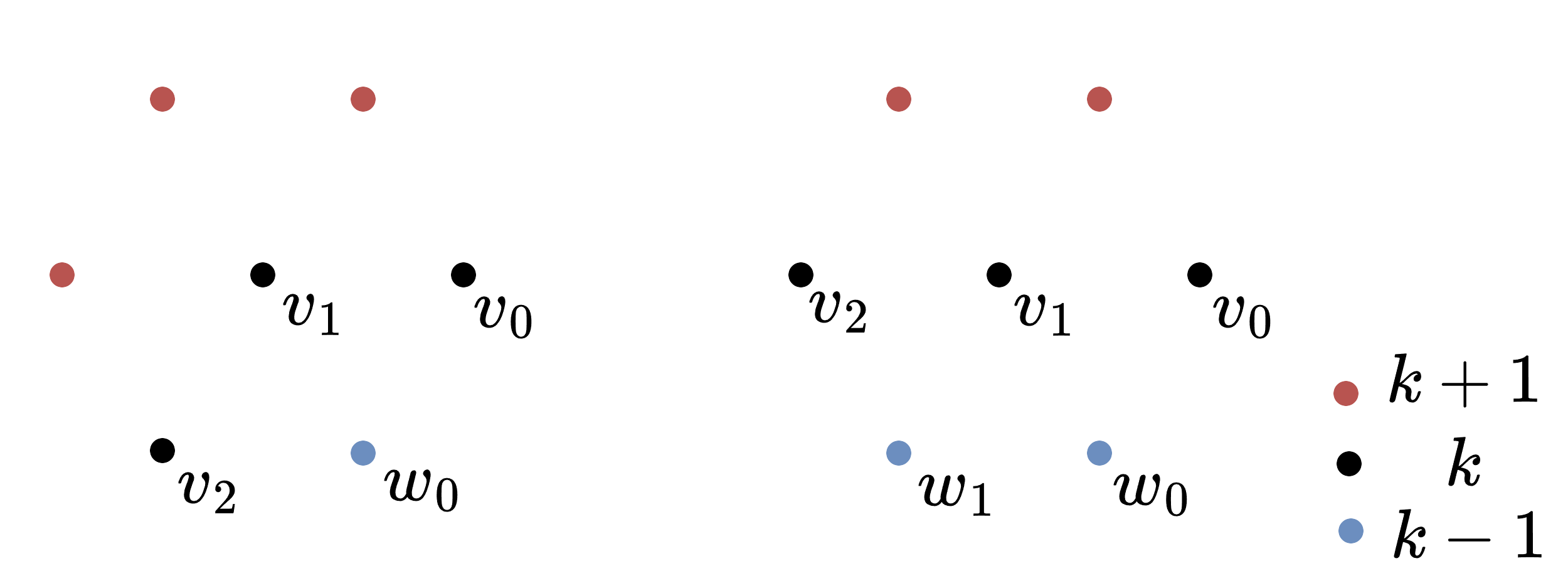}
        \end{figure}

        If the neighborhood of $v_1$ is of type $A$, we do nothing. If the neighborhood of $v_1$ is of type $B$, we denote the lattice point adjacent to both $v_1,v_2$ as $w_1$.

        We repeat this process until we are left with $v_n$ and obtain a sequence of lattice points $W = (w_0, w_1,...,w_m)$. Notice that all the lattice points in $W$ have hop count $k-1$ and every lattice point in $R$ is adjacent to a lattice point in $W$. As a result, there exists some ribbon $Q$ with subgraph $W$, and $Q$ is a parent ribbon of $R$. 
        %Since every pair of adjacent points in $W$ are adjacent in $\Gamma$, $W$ is a sub-ribbon of some ribbon $Q$ and $Q$ is a parent of $R$. 

        Now we shall show that the parent ribbon is unique. Suppose not, then there exists a ribbon $R$ with two parent ribbons $R_a$ and $R_b$. By the definition of parent ribbon, $\forall x \in V(R), \exists y \in V(R_a)$ and $z \in V(R_b)$ such that $x$ is adjacent to both $y$ and $z$. Consider the neighborhood of $x$,
        \begin{itemize}
            \item [-] if the neighborhood is of type A, then $y=z$, contradicts to proposition \ref{prop-pointBelongsToUniqueRibbon};
            \item [-] if the neighborhood is of type B, then $y$ is adjacent to $z$. Since $h(y) = h(z) = h(x)-1$, $y,z$ belongs to the same ribbon. We reach a contradiction.
        \end{itemize} 
        We conclude with the existence and uniqueness of the parent ribbon.
    \end{proof}

    We notice that the ribbons form a directed tree with root $R_0$.
    \begin{definition}
        Given a proper ribbonization $\hat{P}$, define its corresponding \textbf{ribbonization tree $Tr(\hat{P})$} as a directed graph $(V(Tr(\hat{P})), A(Tr(\hat{P})))$ such that 
        \begin{enumerate}
            \item $V(Tr(\hat{P})) = \hat{P}$ and 
            \item $(\forall R_i, R_j \in V(Tr(\hat{P})))$, $R_iR_j \in A(Tr(\hat{P})) \iff R_i \text{ is the parent of } R_j$. 
        \end{enumerate}
    \end{definition}

    % \begin{figure}[H]
    %     \centering
    %     \includegraphics[scale=0.35]{figures/proof-canonical/proofCanonical-ribbonizationOfP.png}
    %     \hspace{2cm}
    %     \includegraphics[scale=0.3]{figures/proof-canonical/proofCanonical-treeOfP.png}
    %     \caption{Left: a proper ribbonization $\hat{P}$; Right: the rooted tree induced by $\hat{P}$}
    % \end{figure}
    
    \begin{proposition}
        A ribbonization tree is a directed tree.
    \end{proposition}
    \begin{proof}
        By proposition \ref{prop-uniqueParent}, every ribbon has a unique parent ribbon except for $R_0$. As a result, by assigning $R_0$ as the root, we obtain a directed tree.
    \end{proof}

\subsection{Ribbonization of $S$}
\label{sect-ribbonizationOfS}
    By letting $P$ be the user-defined shape $S$, and restraining $S$ to be in $H_1$, we could obtain the ribbonization of $S$.

    \begin{definition}
        A lattice point $x$ is a \textbf{boundary point} of $S$ if:
        \begin{enumerate}
            \item $x$ is a lattice point in $S$; and
            \item $x$ is adjacent to a lattice point that is not in $S$.
        \end{enumerate}
    \end{definition}

    \noindent Recall that every lattice point on the boundary of $S$ is not in $D$, that is
        \[ (\forall x \in S) (\exists y \in S^c \cap L \text{ and } ||x-y||=d) \Rightarrow x \notin D \]

    % We define the ribbonization of $S$ with holes $D$ as a special case of the ribbonization of $S$:

    % \begin{definition}
    %     \label{def-properRibbonizationOfS}
    %     A \textbf{proper ribbonization} of $S$ satisfies the following properties:
    %     \begin{enumerate}
    %         \item all conditions in the proper ribbonization of $S$ as in definition \ref{def-properRibbonization}; and
        
    %         \item every lattice point on the boundary of $S$ is not in $D$, that is
    %         \[ (\forall x \in S) (\exists y \in S^c \cap L \text{ and } ||x-y||=d) \Rightarrow x \notin D \]          

    %         % \item $(\forall x\in S\setminus D) \exists y \in \{S\setminus D\} \cap L, s.t. ||x-y||<d$.
    %     \end{enumerate}
    % \end{definition}

    % Note that:
    % \begin{enumerate}
    %     \item definition \ref{def-properRibbonizationOfS}.2 ensures the hole $D$ could be correctly captured. In other words, no lattice point in the hole $D$ is on the boundary of $S$.
    %     % \item Definition \ref{def-properRibbonizationOfS}.3 ensures the lattice points in $S\setminus D$ can provide a good coverage of the shape $S \setminus D$.
    % \end{enumerate}

    \begin{figure}[H]
        \centering
        \includegraphics[scale=0.335]{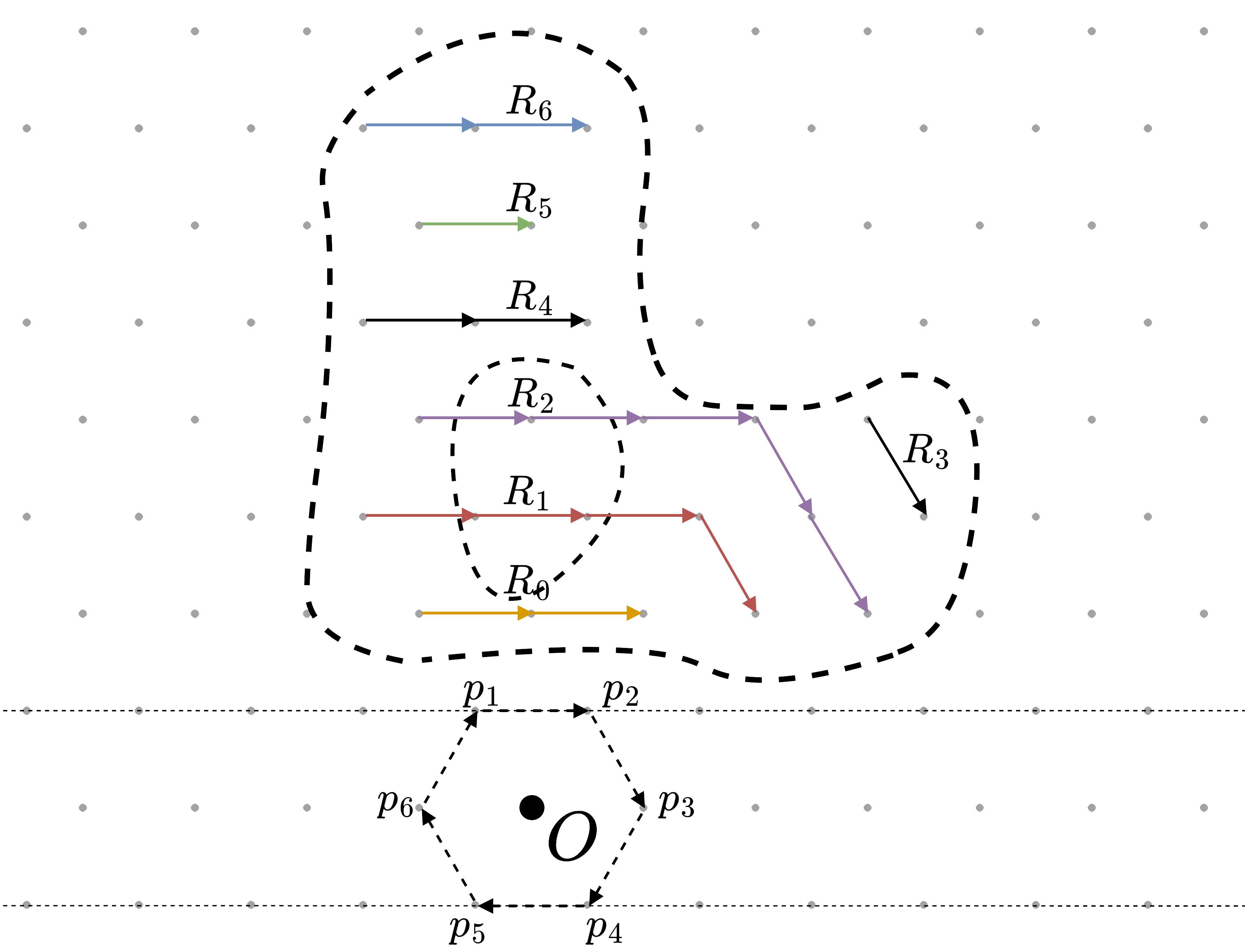}
        \hspace{2cm}
        \includegraphics[scale=0.33]{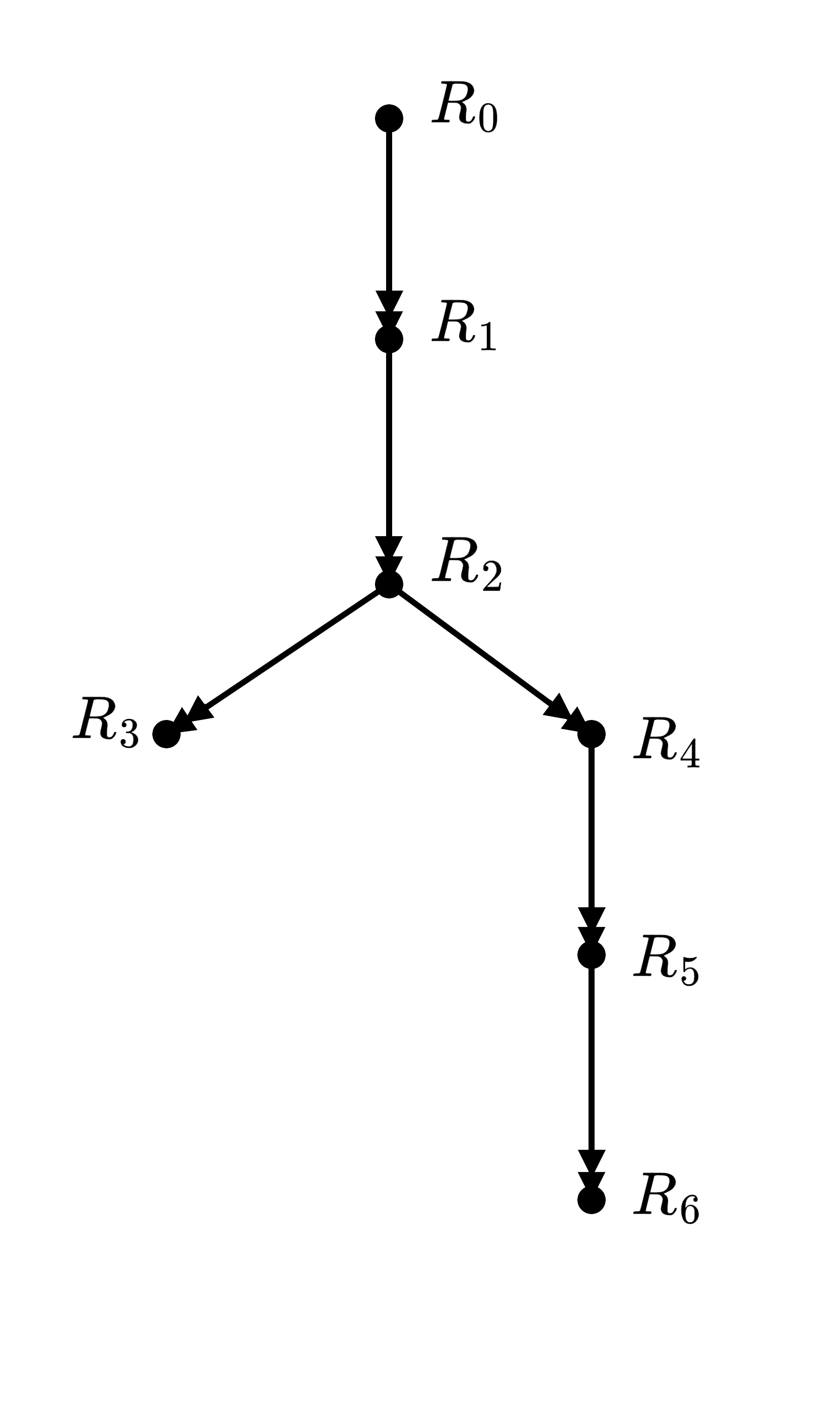}
        \caption{Left: a proper ribbonization $\hat{S_D}$; Right: the rooted tree induced by $\hat{S}$}
    \end{figure}

    \begin{definition}
        A \textbf{cavity point} is a lattice point in $D$. We use $\mathbf{c(R)}$ to denote the number of cavity points of a ribbon $R$.
    \end{definition}

    % \begin{figure}[H]
    %     \centering
    %     \includegraphics[scale=0.5]{figures/proofHoles/proofNonCanonicalWithHoles-CavityOfRibbon.png}
    %     \caption{Illustration of Ribbon $R_5$ and on-ribbon ID}
    % \end{figure}

\subsection{Ribbonization of Half-plane $H_2$}
\label{sect-ribbonizationOfH2}
    To distinguish the ribbons induced by the set of lattice points in $H_2$ from those induced by $S$, we use the notation $IR_i$ to denote the ribbon with index $i$ and call them idle ribbons:

    \begin{figure}[H]
        \centering
        \includegraphics[scale=0.3]{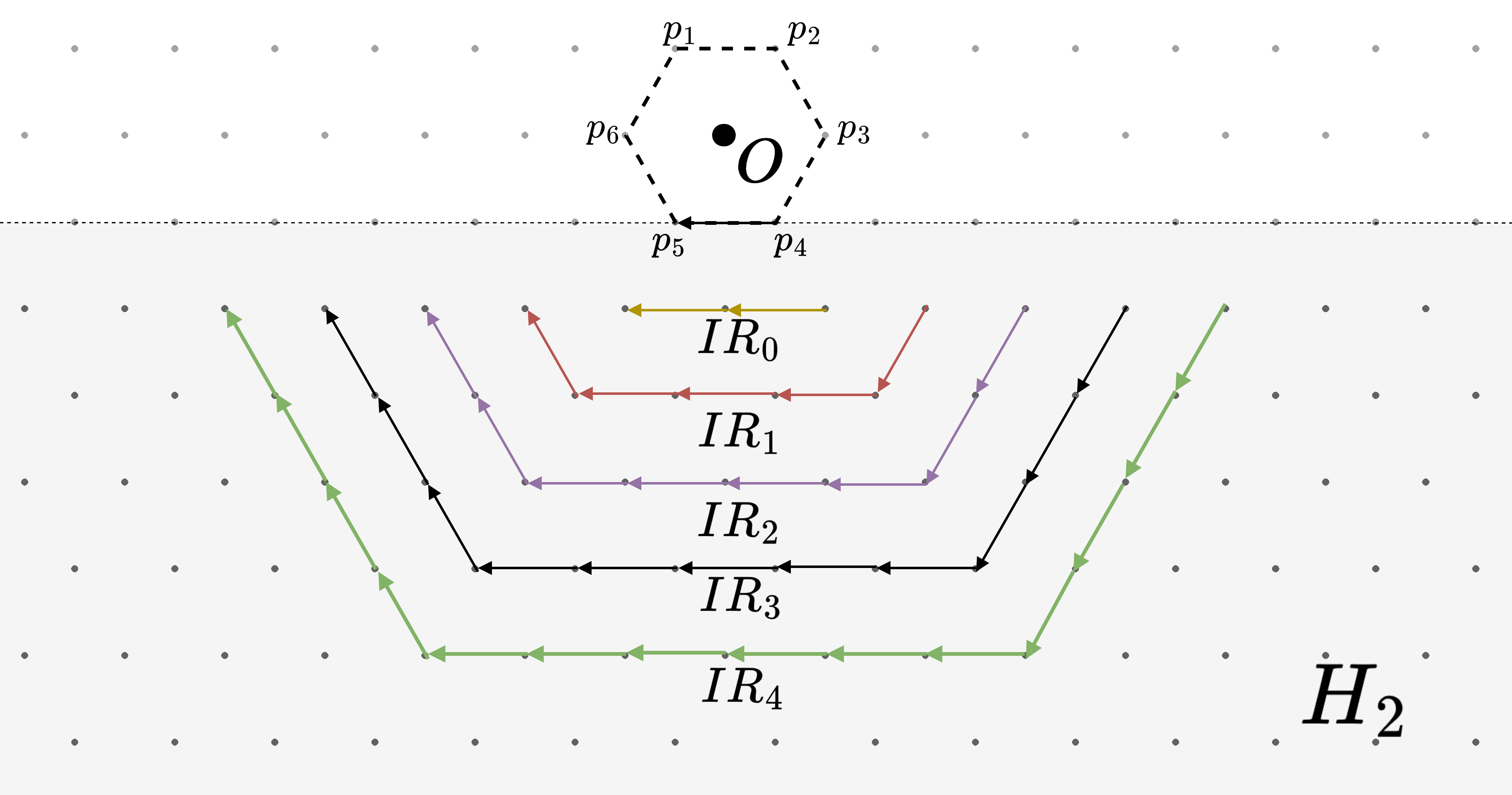}
        \hspace{2cm}
        \includegraphics[scale=0.3]{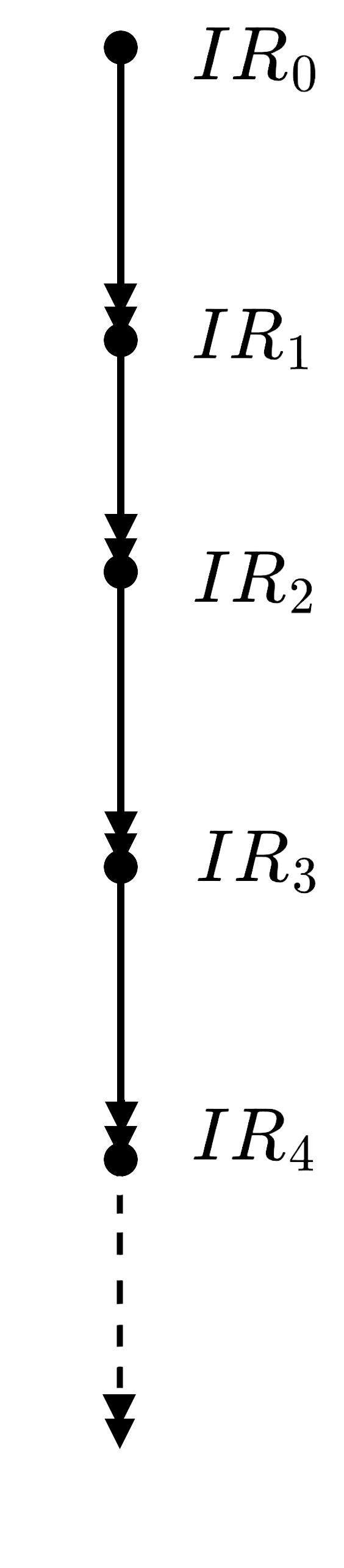}
        \caption{Left: idle ribbons induced by $H_2$; Right: the corresponding rooted tree of idle ribbons}
    \end{figure}

    % By scaling the left, right and bottom virtual boundaries of the trapezoid, we get a set of idle ribbons $\{IR_0, IR_1, IR_2,..., IR_n\}$ and $n$ can assume any positive integer number. 
    Observe that any lattice point in $H_2$ belongs to some idle ribbon, and the idle ribbons follow the properties discussed previously.

\subsection{Ordering Induced by the Ribbon Structure}
\label{sect-order}

    First, we impose an ordering on the lattice points in the idle half-plane:

    \begin{definition}
        \label{def-idleRibTreeOrder}
        \leavevmode
        \makeatletter
        \@nobreaktrue
        \makeatother
        \begin{enumerate}
            \item The \textbf{idle ribbon order}, denoted as $<$, is a total order of the lattice points of an idle ribbon $IR$ such that
                \[(\forall x,y \in V(IR)) x < y \iff y \rightarrow x \]
            \item The \textbf{idle tree order}, denoted as $<$, is a partial order of the idle ribbons in $S$ such that
                \[(\forall IR_i,IR_j \in \hat{H_2}) IR_i < IR_j \iff IR_i \twoheadrightarrow IR_j \]
            % $\forall R_i, R_j$, $R_i< R_j$ iff the path from the root $R_0$ to $R_j$ passes through $R_i$. A \textbf{complete tree order} is a strict total order that extends from the tree order.
        \end{enumerate}
    \end{definition}

    % \begin{definition}
    %     The \textbf{idle order} is a strict total order of the lattice points in $H_2$ such that for all $x,y \in H_2$, $x < y$ iff
    %     \begin{enumerate}
    %         \item $h(x) < h(y)$; or
    %         \item $h(x) = h(y)$ and $x$ is behind $y$ with respect to the idle ribbon.
    %     \end{enumerate}
    % \end{definition}

    Then, we impose orderings on the lattice points in $S$:
    % \begin{definition}
    %     \leavevmode
    %     \makeatletter
    %     \@nobreaktrue
    %     \makeatother
    %     \begin{enumerate}
    %         \item Given a ribbon $R$, the \textbf{ribbon order} is a strict total order of the lattice points of $R$ such that for all $x,y \in V(R)$, $x < y$ if $y \rightarrow x$.
    %         \item Given a ribbonization tree $T(\hat{S})$, the \textbf{tree order} of a ribbonization tree $T$ is the strict partial order on the set of ribbons $\hat{S}$ such that $\forall R_i, R_j \in \hat{S}$, $R_i< R_j$ if $R_i \twoheadrightarrow R_j$. The \textbf{complete tree order} is a strict total order that extends the tree order\footnote{A partial order extends to a total order by the Szpilrajn extension theorem.}. 
    %     \end{enumerate}

    % \end{definition}
    \begin{definition}
        \label{def-ribTreeOrder}
        \leavevmode
        \makeatletter
        \@nobreaktrue
        \makeatother
        \begin{enumerate}
            \item The \textbf{ribbon order}, denoted as $<$, is a total order of the lattice points of a ribbon $R$ such that
                \[(\forall x,y \in V(R)) x < y \iff y \rightarrow x \]
            \item The \textbf{tree order}, denoted as $<$, is a partial order of the ribbons in $S$ such that
                \[(\forall R_i,R_j \in \hat{S}) R_i < R_j \iff R_i \twoheadrightarrow R_j \]
            The \textbf{complete tree order} is a strict total order that extends the tree order\footnote{A partial order extends to a total order by the Szpilrajn extension theorem.}. 
            % $\forall R_i, R_j$, $R_i< R_j$ iff the path from the root $R_0$ to $R_j$ passes through $R_i$. A \textbf{complete tree order} is a strict total order that extends from the tree order.
        \end{enumerate}
    \end{definition}

\subsection{Conventions and Notations}
    \label{sect-Ribbon&Robots}
        In this section, we establish some conventions and technical terms that link the MRS to the ribbon structure:
    
        \begin{definition}
            At sub-epoch $T=n^-$, a lattice point $x$ is \textbf{occupied} if there is a non-moving robot $u \in U$ such that $p_{n^-}(u) = x$. The same applies to $T = n^+$.
            A lattice point is \textbf{unoccupied} if it is not occupied.
        \end{definition}
        
        \begin{definition}
            The \textbf{neighborhood} of a non-moving robot is the six lattice points at a distance $d$ to the center of the robot and the corresponding robots that occupy the lattice points (if any). 
        \end{definition}
    
        \begin{definition}
            Given two non-moving robots $u,v \in U$ at sub-epoch $T=n^-$:
            \leavevmode
            \makeatletter
            \@nobreaktrue
            \makeatother
            \begin{itemize}
                \item robot $u$ \textbf{belongs to} ribbon $R$ if its center $p_{n^-}(u)$ is a vertex of $R$, denoted as $u \in R$. The \textbf{gradient value} of robot $u$ is the gradient value of the ribbon $R$;  
                \item robot $u$ is \textbf{in front of} robot $v$ if $u,v$ belongs to the same ribbon and lattice point $p_{n^-}(u) \leftarrow p_{n^-}(v)$.
            \end{itemize}
            The same definition applies to $T=n^+$.
        \end{definition}
    
        \begin{definition}
            At sub-epoch $T=n^-$:
            \begin{itemize}
                \item a ribbon $R$ is \textbf{empty} if 
                \[\forall \text{ non-moving robot } u \in U, u \notin R\]
                \item a ribbon $R$ is \textbf{filled} if
                \[\forall x \in V(R), \exists \text{ non-moving robot } u \in U \text{ such that } x = p_{n^-}(u)\]
                \item a ribbon is \textbf{half-filled} if it is neither empty nor filled.
            \end{itemize}
            The same definition applies to $T=n^+$.
        \end{definition}

\newpage

\section{Correctness of the Add-Subtract Algorithm}
\label{sect-methodAndProof}

In appendix \ref{sect-methodAndProof}, we formalize the add-subtract algorithm and provide a rigorous proof of correctness. The correctness of the problem statement includes two parts: first, we need to show that, for each moving robot, the path between the starting position and stopping position is unobstructed and the moving robot is localizable when needed; and second, we need to show the final coverage of the MRS satisfied the requirement. The second part is can be easily verified by looking at the definition of the assembly sequence and the reassembly sequence. As a result, we shall focus on proving the first part. Suppose we have obtained a proper ribbonization of $S$ and suppose there are $N$ lattice points in $S$. The proof is carried out in three steps. 

For the first step, we analyze the additive stage. Firstly, we define the activation sequence $s[k]$ and the assembly sequence $t[k]$ in section \ref{Section-MovingSequence}. After that, we show the mobility of each moving robot $u_n$ by investigating the boundary of the MRS in section \ref{sect-motionPlanning}. In particular, we prove that the boundary of the MRS maintains a certain smoothness that allows an unobstructed path along the perimeter of the MRS from start position $s[k]$ to target position $t[k]$ for every active robot $u_k$ with $k \leq N$. Then, we show the localizability of each moving robot $u_n$ in section \ref{sect-localization}. In particular, we show that the active robot is able to localize itself at the starting and stopping positions. 

% Since the path is unobstructed and the robot is localizable, we conclude that robot $u_k$ stops at $t[k]$ (theorem \ref{thm-additiveStageCorrect}).

For the second step, we investigate the properties of the MRS after the additive stage in section \ref{sect-propertiesOfEOA}. In particular, we define the "merging point" for every ribbon to denote the lattice point on the perimeter at which a robot moves out of $S$ first visits. 

For the third step, we investigate the subtractive stage. We start by formalizing the subtractive stage of the algorithm using the reactivation sequence $s_r[k]$ and reassembly sequence $t_r[k]$ (section \ref{sect-extendedSequence}). After that, we investigate the path between lattice point $s_r[k]$ and lattice point $t_r[k]$, and show that the path for robot $u_{N+k}$ is unobstructed by utilizing the results from the additive stage (section \ref{sect-pathPlanningForSubtractive}). The localizability of each moving robot is proved in section \ref{sect-localization2}. 

% Since the path is unobstructed and the robot is localizable, we conclude that robot $u_{N+k}$ moves from $s_r[k]$ and stops at $t_r[k]$.

Finally, in section \ref{sect-finalShape}, we conclude that the sequence of tuples defined by the movement sequences satisfies the problem statement.

\subsection{Activation Sequence and Assembly Sequence}
\label{Section-MovingSequence}

During the additive stage, in each epoch, one of the idle robots becomes active and moves along the perimeter of the collection of the robots. The moving robot enters the user-defined shape and stops at one of the lattice points inside the shape. In such a way, the MRS shrinks in the idle half-plane $H_2$ and expends inside $S$ in the formation half-plane $H_1$. We use two sequences of lattice points to represent which idle robot is activated and where it shall stop:

\begin{itemize}
    \item [-] the activation sequence $s[k]$: the $n$th entry denotes the lattice point at which the robot $u_n$ starts to move;
    \item [-] the assembly sequence $t[k]$: the $n$th entry denotes the lattice point at which the robot $u_n$ stops.
\end{itemize}

We use the orderings induced by ribbons to formulate the activation sequence and the assembly sequence:
    
\begin{definition}
    The \textbf{activation sequence} $s$ is defined recursively such that each element $s[k]$ is the largest occupied lattice point in $H_2$ (w.r.t the idle ribbon order) on the largest idle ribbon (w.r.t the idle tree order) that has not been included in the early subsequence $s_{[1:k-1]}$.
    % \begin{enumerate}
    %     \item $s[\text{ }]$ contains all lattice points occupied by idle robots at $T=0$;
    %     \item $\forall s[i],s[j] \in s[\text{ }]$, $0 < i < j$ iff $s[i] < s[j]$ in the idle order; 
    % \end{enumerate} 
\end{definition}
Notice that every idle robot in the initial configuration occupies a lattice point, which appears exactly once in the activation sequence. At epoch $T=n \leq N$, the idle robot at $s[n]$ becomes active and starts to move. 

\begin{definition}
    The \textbf{assembly sequence} $t$ is defined recursively such that each element $t[k]$ is the smallest lattice point (w.r.t the ribbon order) on the smallest ribbon (w.r.t the complete tree order) that has not been included in the early subsequence $t_{[1:k-1]}$.
\end{definition}

% \begin{definition}
%     The \textbf{assembly sequence} $t[\text{ }]$ is a sequence contains all lattice points in $S$ such that $\forall t[i], t[j] \in t[\text{ }]$, $i < j$ iff
%     \begin{enumerate}
%         \item $R(t[i]) = R(t[j])$ and $t[i] < t[j]$ in the ribbon order; or
%         \item $R(t[i]) < R(t[j])$ in the tree order.
%     \end{enumerate} 
% \end{definition}
Notice that every lattice point inside $S$ appears exactly once in the assembly sequence. At epoch $T = n \leq N$, the active robot $u_n$ stops at lattice point $t[n]$. Loosely speaking, the idle robot at the tail of the outermost idle ribbon is activated first, and the lattice positions in $S$ are to be occupied ribbon by ribbon, from front to tail.

\subsection{Path for the Additive Stage}
\label{sect-motionPlanning}
    Recall that the active robot $u_n$ for the additive stage moves clockwise along the perimeter of the MRS while keeping a distance $d$ to the nearest non-moving robot. To prove the correctness of the path, we go through the following steps:
    \begin{enumerate}
        \item  firstly, we show that the starting position $s[n]$ and the goal position $t[n]$ is on the perimeter of the MRS (proposition \ref{prop-startEndBothInEFS}); and
        \item secondly, we show that the boundary maintains a certain smoothness (proposition \ref{prop-properBoundary}). As a result, the perimeter of the MRS (definition \ref{def-EFC}) is always one connected piece. In particular, we show that the perimeter of the MRS is homeomorphic to a circle (proposition \ref{prop-EFPisCircle}); and
        \item  lastly, we conclude that the perimeter of the MRS induces an unobstructed path from $s[n]$ to $t[n]$ (theorem \ref{thm-unobstructedPath}). 
    \end{enumerate}

    We start by defining the boundary of the MRS:

    \begin{definition}
        \label{def-boundary}
         A \textbf{boundary robot} is a non-moving robot occupying a vertex $x \in L$ such that $x$ is adjacent to at least one unoccupied lattice point. The \textbf{set of boundary robots}, denoted as $\bm{B_{n^-}}$ and $\bm{B_{n^+}}$, is the collection of all boundary robots at the sub-epoch $T=n^-$ and $T=n^+$, respectively.
    \end{definition}

    \begin{definition}
        \label{def-EFC}
        At sub-epoch $T=n^-$ with $n \leq N$, a lattice point $x \in L$ is an \textbf{edge-following point} iff:
        \begin{enumerate}
            \item $(\nexists u \in U)  ||p_{n^-}(u) - x|| = d$
            \item $(\exists v \in U) p_{n^-}(v) \in L \land ||x- p_{n^-}|| = d$
        \end{enumerate} 
        % \begin{enumerate}
        %     \item $x$ is not occupied by any robot, and 
        %     \item $\exists \text{ lattice point } y \in L$ such that $y$ is occupied and $x$ is adjacent to $y$.
        % \end{enumerate}
        The \textbf{edge-following set}, denoted as $\bm{EFS_{n^-}}$, is the set containing all the edge-following points at sub-epoch $T=n^-$.

        In a similar way, we define the edge-following points and the corresponding edge-following set $\bm{EFS_{n^+}}$ for sub-epoch $T=n^+$.
    \end{definition}

    \begin{figure}[H]
        \centering
        \includegraphics[scale=0.3]{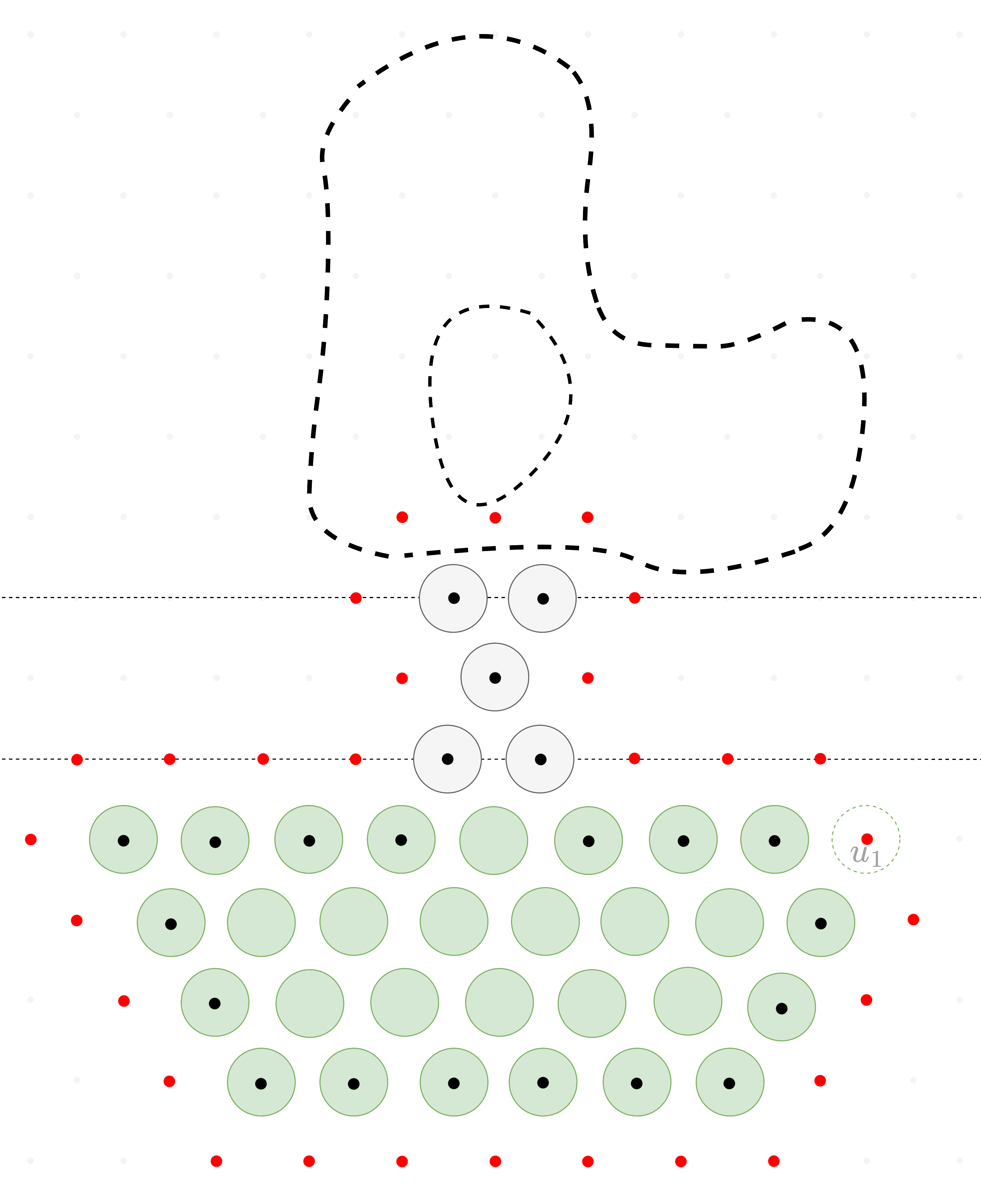}
        \includegraphics[scale=0.3]{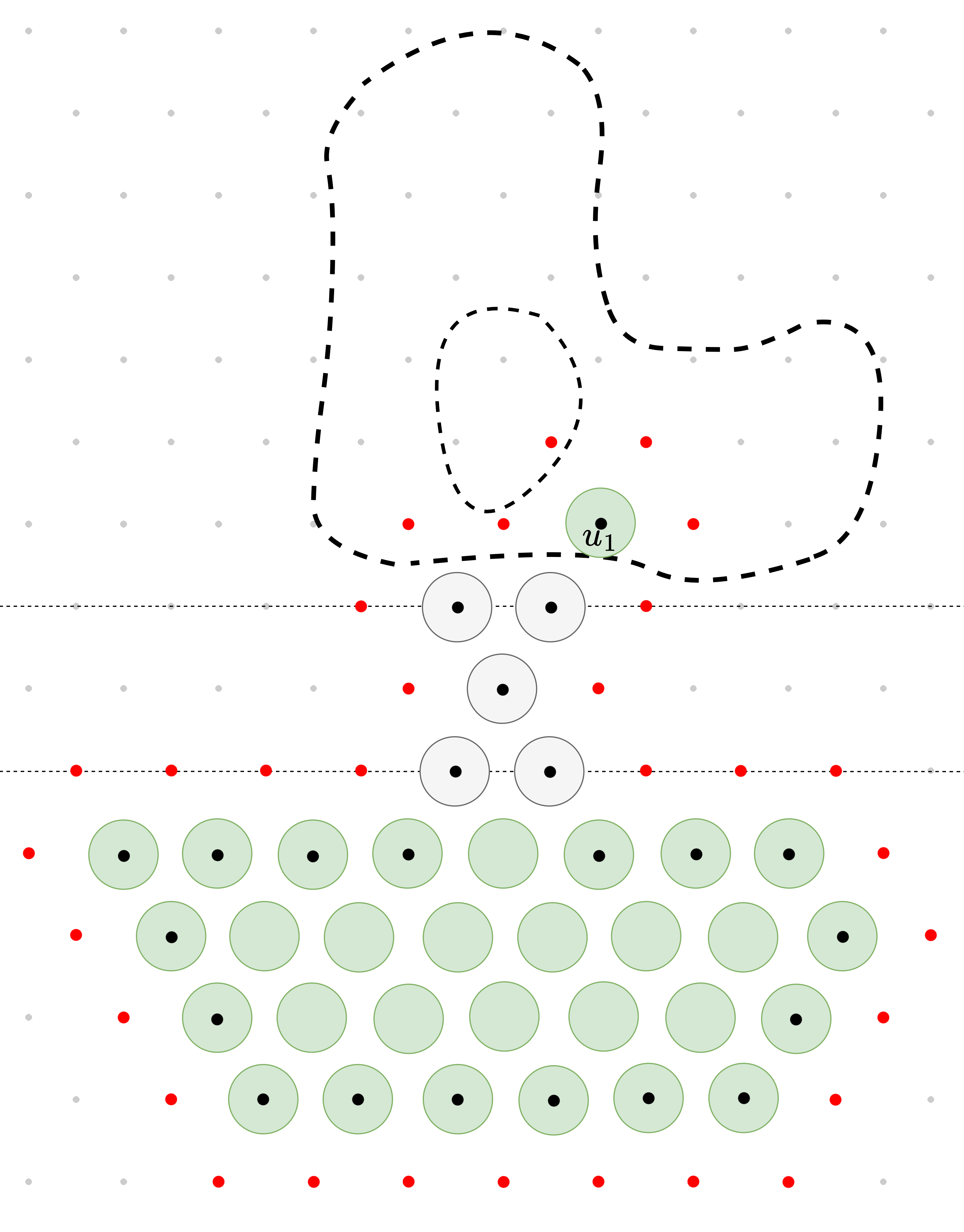}
        \caption{Left: the black dots indicate the centers of all boundary robots $B_{1^-}$, and the red dots indicate the edge-following set $EFS_{1^-}$. Notice that the lattice position previously occupied by robot $u_1$ is considered as part of the edge-following set. Right: the black dots indicate the centers of all boundary robots at $T=1^+$, and the red dots indicate the edge-following set $EFS_{1^+}$.}
    \end{figure}

    \begin{proposition}
        \label{prop-startEndBothInEFS}
        At sub-epoch $T=n^-$ with $n \leq N$, $s[n] \in EFS_{n^-}$ and $t[n] \in EFS_{n^-}$.
    \end{proposition}
    \begin{proof}
        By the activation sequence, $s[n]$ is adjacent to a robot belonging to the parent ribbon of $R(s[n])$. As a result, $s[n] \in EFS_n$ after $u_n$ transits into movement.
        Notice that by the assembly sequence, $t[n]$ is adjacent to a lattice point occupied by a stopped robot belonging to the parent ribbon of $R(t[n])$ (or a seed robot if $t[n] \in R_0$). Since $t[n]$ is unoccupied, $t[n] \in EFS_{n^-}$.
    \end{proof}

    \begin{definition}
        \label{def-properBoundaryRobots}
        At sub-epoch $T=n^-$ with $n \leq N$, the set of boundary robots $B_{n^-}$ is \textbf{proper} if
        \begin{enumerate}
            \item for any pair of robots in the set of boundary robots, if the distance between two robots is $\sqrt{3}d$, there is at most one unoccupied lattice point adjacent to the centers of both robots; and
            \item for any pair of robots in the set of boundary robots, if the distance between two robots is $2d$, then the lattice point between them is occupied.
            % \item any graph cycle formed by boundary robots surrounds no empty lattice position; and
            % \item the set of boundary robots is connected.
        \end{enumerate}
        The same definition applies to $B_{n^+}$ at $T=n^+$.
    \end{definition}

    \begin{figure}[H]
        \centering
        \captionsetup{justification=centering}
        \includegraphics[scale=0.35]{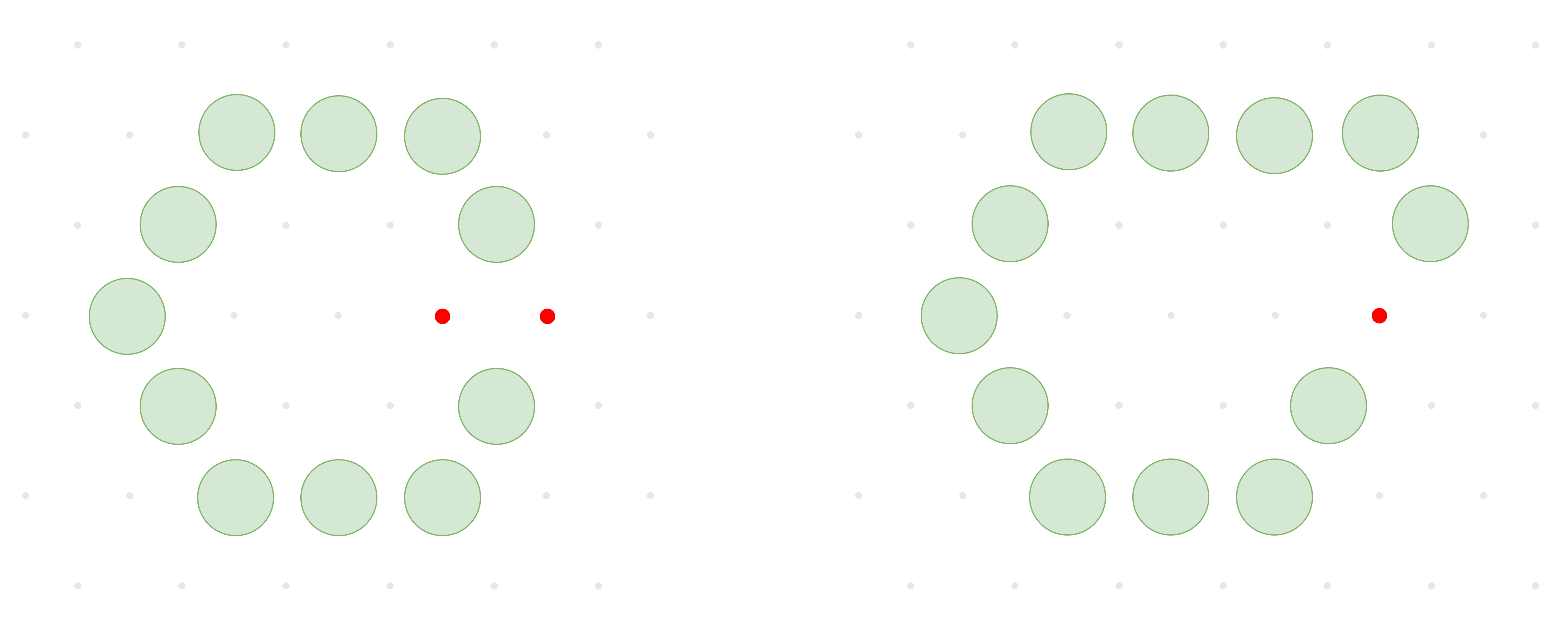}
        \caption{Left: the set of boundary robots is improper since it violates requirement 1; \\ Right: the set of boundary robots is improper since it violates requirement 2}
    \end{figure}

    % \begin{definition}
    %     The edge-following set $EFS$ is \textbf{proper} iff the set of boundary robots $B$ is proper. 
    % \end{definition}

    \begin{proposition}
        \label{prop-properBoundary}
        With a proper ribbonization $\hat{S}$, by following the activation sequence $s$ and assembly sequence $r$, the set of boundary robots $B_{n^-}$ and $B_{n^+}$ is proper for all $n \leq N$.
    \end{proposition}
    \begin{proof}
        We shall prove this by induction. Let $P(n)$ be "at sub-epoch $T=n^-$, the set of boundary robots $B_{n^-}$ is proper". Let $Q(n)$ be "at sub-epoch $T=n^+$, the set of boundary robots $B_{n^+}$ is proper". 
    
        \begin{itemize}
            \item [-] Base case:  Observe that the initial configuration satisfies the two conditions in definition \ref{def-properBoundaryRobots}. $Q(0)$ is true.
            
                % \begin{figure}[H]
                %     \centering
                %     \includegraphics[scale=0.35]{figures/proofHoles/proofNonCanonicalWithHoles-Start Configuration.png}
                %     \caption{The robot swarm at system step $T = 0$}
                % \end{figure}

            \item [-] Inductive step: suppose $Q(n-1)$ is true, we want to show that $P(n)$ is true; and suppose $P(n)$ is true, we want to show that $Q(n)$ is true. In other words, we want to show that 
                \begin{itemize}
                    \item [(1)] $Q(n-1) \rightarrow P(n)$: after $u_{n}$ starts to move and leaves an unoccupied point $x$, the set of boundary robots $B_{(n)^-}$ is proper; and 
                    \item [(2)] $P(n) \rightarrow Q(n)$: after $u_{n}$ stops, the set of boundary robot $B_{(n)^+}$ is proper.

                \end{itemize}
            
                % \begin{figure}[H]
                %     \centering
                %     \includegraphics[scale=0.35]{figures/proofHoles/proofNonCanonicalWithHoles- biglemma.png}
                %     \caption{The robot swarm at system step $T = n+1$, and $u_{n+1}$ is the moving robot}
                % \end{figure}

                %Notice that any lattice point in $H_1$ is at least $4\sqrt{3}d$ away from any lattice point in $H_2$, so any violation of the improperness in $B_{n+1}$ could not involve both $u_{n}$ and $x$. As a result, we could prove the claim by two parts: first show that the boundary robots are proper after $u_n$ stops and join $B_n$, second show that the boundary robots are proper after $u_{n+1}$ is removed from $B_n$.

            To prove (1), we notice that after robot $u_{n}$ starts to move, an additional unoccupied lattice point $x$ is created. Suppose $x$ belongs to $IR_i$ with $i \geq 1$. 
                
                \begin{figure}[H]
                    \centering
                    \includegraphics[scale=0.6]{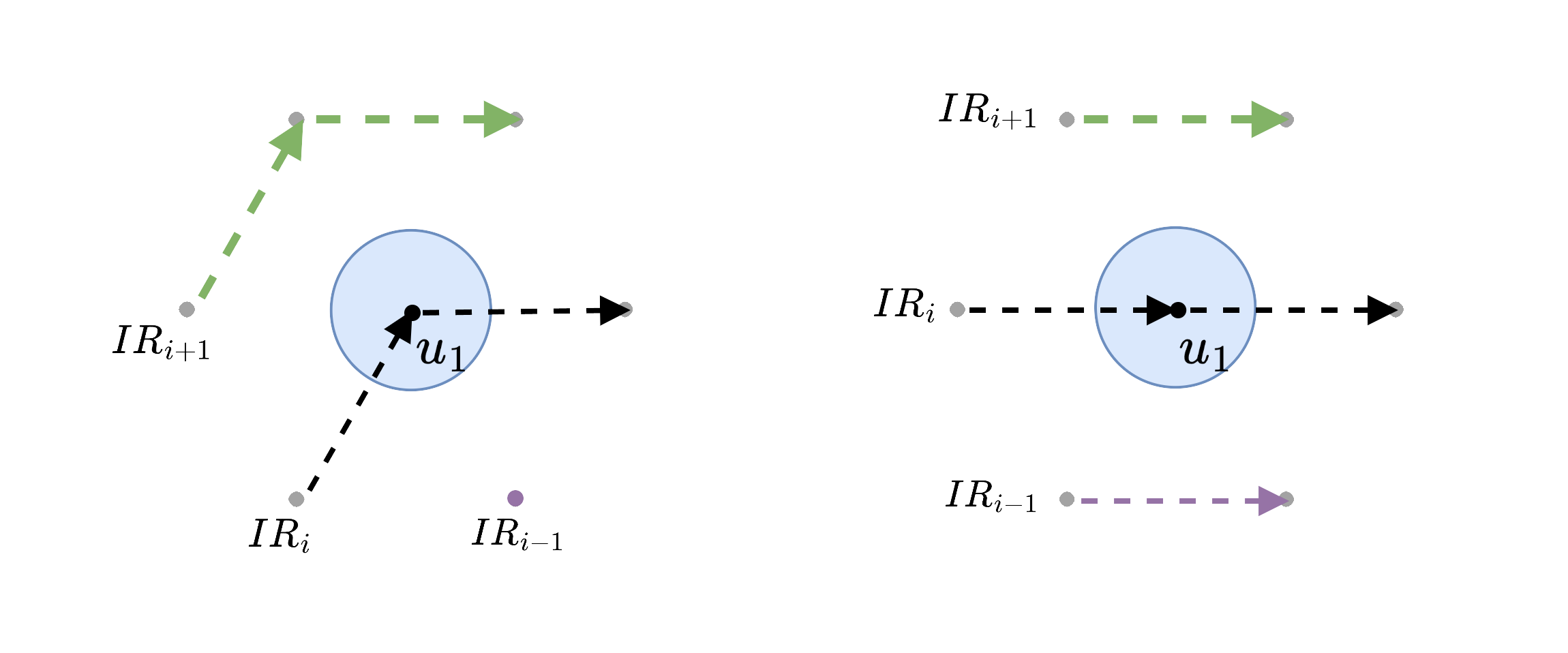}
                    \caption{Left: Type A neighborhood; Right: Type B neighborhood}
                \end{figure}

                \begin{enumerate}
                    \item We shall show that the definition \ref{def-properBoundaryRobots}.1 is satisfied by contradiction. Suppose the definition \ref{def-properBoundaryRobots}.1 is violated, then the neighborhood of $x$ contains two occupied lattice points $y,z$ with distance $\sqrt{3}d$ from each other. Notice that by the idle order, $x$ must be adjacent to an occupied lattice point $y$ from the idle ribbon with gradient value $g(IR_i)-1$. If the neighborhood of $x$ is of type A, then $z$ must belong to $IR_{i+1}$, which contradicts the idle order. If the neighborhood of $x$ is of type B and $z$ belongs to $IR_{i}$, there is an unoccupied lattice point on $IR_{i-1}$, which means an idle robot with a lower gradient is activated first. If the neighborhood of $x$ is of type B and $z$ belongs to $IR_{i+1}$, $u_1$ is activated when there is an idle robot with a higher gradient value. Both cases contradict the activation sequence.
                    \item Now, we shall show that the definition \ref{def-properBoundaryRobots}.2 is satisfied. Suppose not, then there are two occupied lattice points on the opposite side of point $x$. If the neighborhood of $x$ is of type A, there is an occupied lattice point belonging to $IR_{i+1}$. If the neighborhood of $u_1$ is of type B, then either there is an occupied lattice point belonging to $IR_{i+1}$ or there is an occupied lattice point on $IR_{i}$ and behind $x$. Both cases contradict the activation sequence. 
                \end{enumerate}

            For the case where $x$ belongs to $IR_0$, by defining idle ribbon $IR_{-1}$ as the directed line graph $(p_5,p_4)$, the above arguments still hold. As a result, $Q(n-1) \rightarrow P(n)$.

            To prove (2), we investigate the two requirements in definition \ref{def-properBoundaryRobots}:

            \begin{enumerate}
                
                \item First, we prove that the set of boundary robots satisfies definition \ref{def-properBoundaryRobots}.1 by contradiction. Suppose not, then there exists a pair of lattice points occupied by boundary robots $u,v$ such that there are two unoccupied lattice points adjacent to both $u$ and $v$. Since $P(n)$ holds, $u_{n}$ must be one of the robots $u,v$. WLOG, we assume $u_{n} = u$.
                    
                \begin{figure}[H]
                    \centering
                    \includegraphics[scale=0.4]{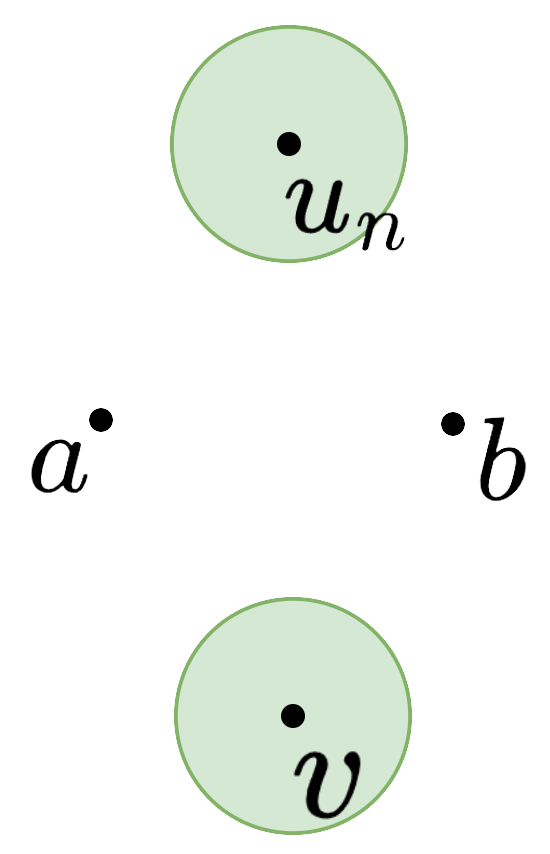}
                    \caption{The pair of robots and two unoccupied lattice points}
                \end{figure}
                
                % Notice that $v$ neighbors some unoccupied lattice point, so $v$ must be in $B_n$. 
                We investigate different cases where $v$ could be at:

                \begin{itemize}
                    \item [-] Case 1: $v$ is a stopped robot in $H_1$. Since $S$ has a smooth boundary, $a,b$ cannot be both not in $S$. Then at least one of the lattice points of $a,b$ is in $S$. WLOG, we assume $a \in S$ and $a$ is an unoccupied lattice point with hop count $k$. We inspect the neighborhood of $a$.

                            First, observe that $u_n$ and $v$ cannot occupy any of the lattice points with hop count $k+1$, otherwise $R(a)$ is not filled but $u_n$ occupies a vertex of a child ribbon of $R(a)$ and the assembly sequence is violated. If the neighborhood of $a$ is of type $A$, denotes the lattice point that is adjacent to $a$ with hop count $k-1$ as $m$, then $m$ must be unoccupied. As a result, the parent ribbon of $R(a)$ is not filled but $R(a)$ is half-filled. This contradicts the assembly sequence.

                            \begin{figure}[H]
                                \centering
                                \includegraphics[scale=0.5]{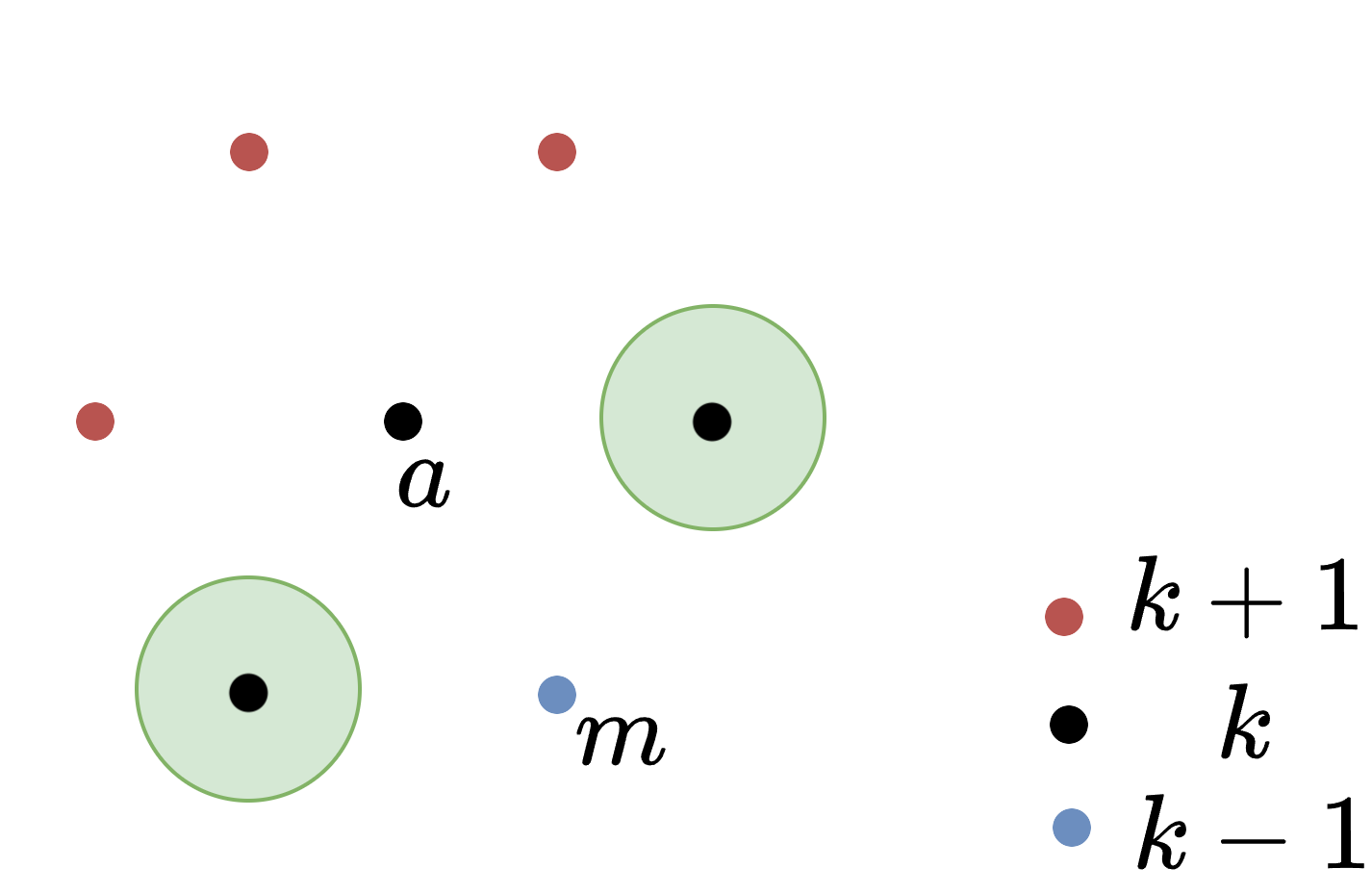}
                                \caption{Type $A$ neighborhood of lattice point $a$}
                            \end{figure}

                            If the neighborhood of $a$ is of type $B$, then $a$ is adjacent to two lattice points with hop count $k-1$, denoted as $m,n$. Then $u_n$ and $v$ must occupy $m$ or $n$ and another lattice point with hop count $k$.As a result, ribbon $R(a)$ and the parent ribbon of $R(a)$ are both half-filled. This contradicts the assembly sequence.

                            \begin{figure}[H]
                                \centering
                                \includegraphics[scale=0.5]{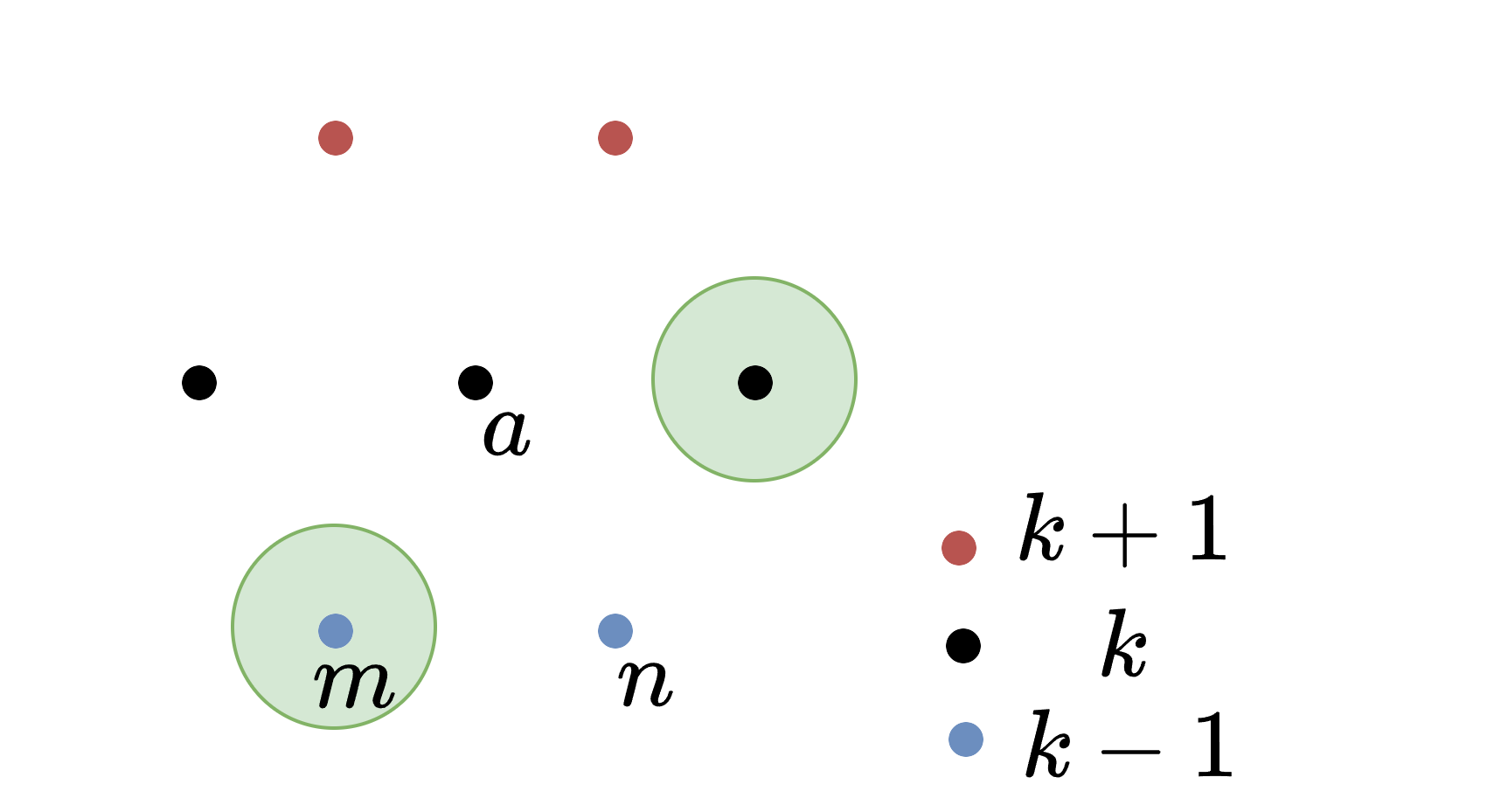}
                                \caption{Type $B$ neighborhood of lattice point $a$}
                            \end{figure}

                    \item [-] Case 2: $v$ is a seed robot. If $v$ is $v_0,v_1$ or $v_2$, by defining ribbon $R_{-1}$ as a directed line graph $(p2,p1)$ and $R_{-2}$ as $(O)$ (with a slight abuse of notation), the arguments in the case 1 still hold. If $v$ is $v_3$ or $v_4$, observe that $v$ is at least $\frac{3\sqrt{3}}{2}d$ from any lattice point in $H_1$, any robot occupying a lattice point in $H_1$ and robot $v$ cannot be adjacent to a same lattice point.
                    \item [-] Case 3: $v$ is an idle robot. Since any lattice point in $H_2$ is at least $2\sqrt{3}d$ from any lattice point in $H_1$, any robot occupying a lattice point in $H_1$ and an idle robot in $H_2$ cannot be adjacent to the same lattice point.
                \end{itemize}

                \item Then, we show that the set of boundary robots satisfies definition \ref{def-properBoundaryRobots}.2 by contradiction. Suppose not, then there exists a pair of lattice points occupying by robots $u,v$ such that there is an unoccupied lattice point between $u$ and $v$. Since $P(n)$ holds, $u_{n}$ must be one of the robots $u,v$. WLOG, we assume $u_{n} = u$.
                    
                \begin{figure}[H]
                    \centering
                    \includegraphics[scale=0.4]{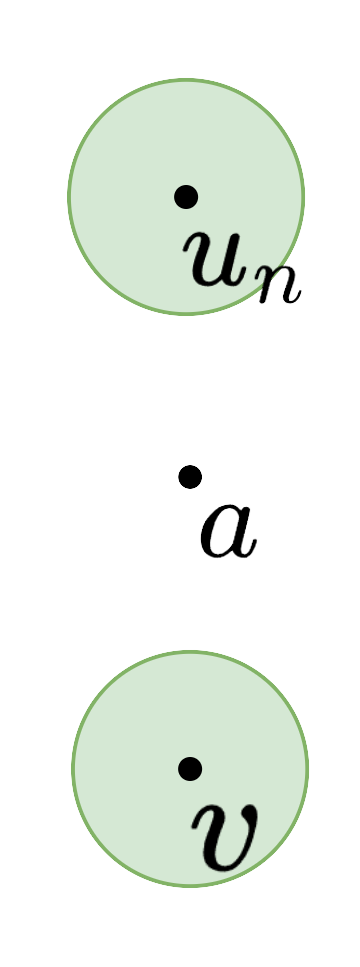}
                    \caption{The pair of robots and one unoccupied lattice point in between}
                \end{figure}
                
                % Notice that $v$ neighbors some unoccupied lattice point, so $v$ must be in $B_n$. 
                
                Now we investigate different cases where $v$ could be at:

                \begin{itemize}
                    \item [-] Case 1: $v$ is a stopped robot in $H_1$. Since $S$ has a smooth boundary, lattice point $a$ is in $S$. 
                    % The line segment connecting $u_{n}$ and $v$ cannot cross the boundary of $S$ (otherwise contradicts the proper ribbonization of $S$). Since the line segment is in $S$, lattice points $a$ is in $S$. 
                    Suppose $a$ occupies a lattice point with hop count $k$. We inspect the neighborhood of $a$.
                    First, observe that $u_n$ and $v$ cannot occupy any of the lattice points with hop count $k+1$, otherwise $R(a)$ is not filled but $u_n$ occupies a vertex of a child ribbon of $R(a)$ and the assembly sequence is violated. As a result, the neighborhood of $a$ must be of type $B$. Notice in this case $p_{n^+}(v)$, $p_{n^+}(u_n)$ and $a$ belongs to the same ribbon and $a$ is in between $p_{n^+}(v)$ and $p_{n^+}(u_n)$. As a result, $a$ is unoccupied when both $p_{n^+}(v)$ and $p_{n^+}(u_n)$ are occupied, which contradicts the assembly sequence.
                    \begin{figure}[H]
                        \centering
                        \includegraphics[scale=0.5]{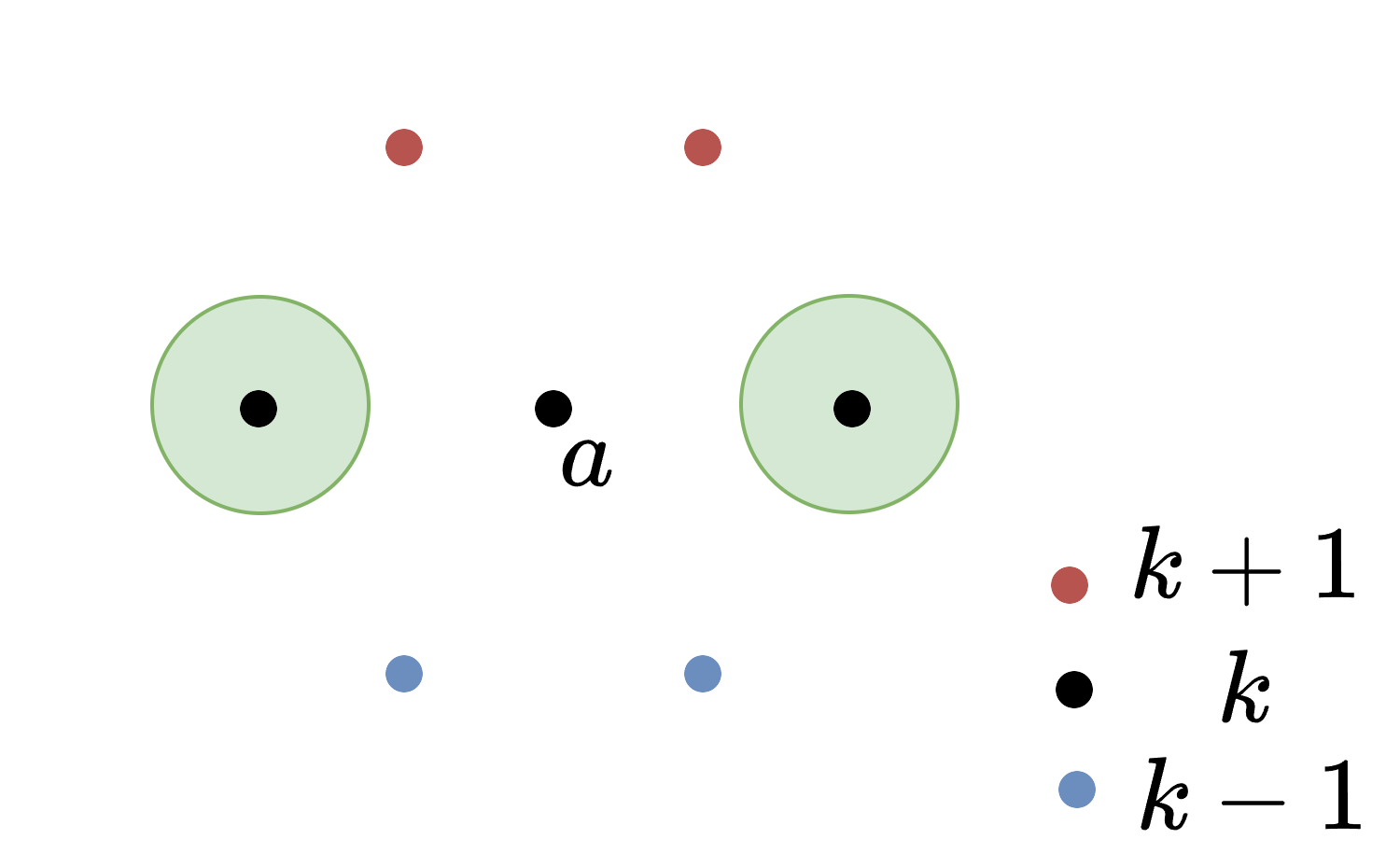}
                        \caption{Type $B$ neighborhood of lattice point $a$}
                    \end{figure}
                    
                    \item [-] Case 2 \& Case 3: $v$ is a seed robot or an idle robot. By similar arguments as in 1.Case 2 and 1.Case 3, we obtain contradictions. 
                \end{itemize}

            \end{enumerate}
            As a result, $P(n) \rightarrow Q(n)$.

        \end{itemize}
        Since $Q(0)$ is true, by induction, $P(n)$ and $Q(n)$ are true for all $1 \leq n \leq N$.
    \end{proof}

    \begin{definition}
        \label{def-EFP}
        At sub-epoch $T=n^-$ with $n \leq N$, define the \textbf{edge-following path} $EFP_{n^-}$ as the path consisting of all points that are at a distance $d$ from the nearest non-moving robot. The same applies to $EFP_{n^+}$ at $T=n^+$ .
    \end{definition}

    \begin{definition}
        For an edge-following path $EFP$:
        \begin{enumerate}
            \item A \textbf{path segment} of the $EFP$ is a contiguous segment of the $EFP$ containing more than one point.
            \item An \textbf{arc} of the $EFP$ is a path segment of the $EFP$ connecting two adjacent lattice point with a central angle of $\frac{\pi}{3}$.
        \end{enumerate}
    \end{definition}

    Since an edge-following path consists of arcs between adjacent edge-following points, we could define a robot's contribution to the path as follows:

    \begin{definition}    
        Given an edge-following path:
        \begin{enumerate}
            \item A point of the edge-following path is \textbf{contributed} by a non-moving robot $u$ if the center of $u$ is at a distance $d$ from the point. A point is said to be \textbf{induced} by robot $u$ if the point is not contributed by any other robot.
            \item An arc of the edge-following path is \textbf{contributed} by a non-moving robot $u$ iff all points on the arc are at a distance $d$ from the center of $u$. The arc is said to be \textbf{induced} by the robot if the arc is not contributed by any other robot.
            % \item A path segment of the edge-following path is \textbf{contributed} by a set of robots a robot $u$ iff all the arcs of the path segment are contributed by $u$. The path segment is said to be \textbf{induced} by the robot if the path segment is not contributed by any other robot.
            % \item A path segment $P$ is \textbf{contributed by a ribbon} $R$ iff all path segments of $P$ are contributed by a robot belongs to ribbon $R$.
        \end{enumerate}
    \end{definition}

    Observe that a lattice point could be induced by multiple robots. However, an arc can only be contributed by a single robot.

    \begin{proposition}
        \label{prop-EFPisCircle}
        During the additive stage, at each sub-epoch, the edge-following path is homeomorphic to a circle, and the edge-following path goes through all the points in the edge-following set.
    \end{proposition}
    \begin{proof}
        We shall prove the proposition by induction. Let $P(n)$ be "At sub-epoch $T=n^-$, $EFP_{n^-}$ is homeomorphic to a circle and the edge-following path goes through all the points of the edge-following set $EFS_{n^-}$". Let $Q(n)$ be "At sub-epoch $T=n^+$, $EFP_{n^+}$ is homeomorphic to a circle and the edge-following path goes through all the points of the edge-following set $EFS_{n^+}$." 

        \begin{itemize}
            \item [-] Base case: Observe that $EFP_{0^+}$ is homeomorphic to a circle and goes through all the edge-following points. $Q(0)$ is true.
                % \begin{figure}[H]
                %     \centering
                %     \includegraphics[scale=0.3]{figures/proofHoles/proofNonCanonicalWithHoles-EFP1.png}
                %     \caption{The path of $u_1$ when moving along the edge-following path}
                % \end{figure}

            \item [-] Inductive step: suppose $Q(n-1)$ is true, we want to show that $P(n)$ is true; suppose $P(n)$ is true, we want to show that $Q(n)$ is true. In other words, we want to show that after $u_n$ starts to move, and after $u_{n}$ stops, the edge-following paths are homeomorphic to a circle and go through all the edge-following points. We shall prove this in two steps: in part I, we show that $Q(n-1) \rightarrow P(n)$; in part II, we show that $P(n) \rightarrow Q(n)$.
                
                \begin{itemize}
                    \item[-] Part I:
                   
                    We shall investigate the edge-following path after $u_{n}$ starts to move. Let $d(s[n])$ denote the disk centered at $s[n]$ with radius $d$. After $u_n$ starts to move, the path segment(s) contributed by $u_n$ is replaced by the path segment transverse $d(s[n])$. We also notice that $s[n]$ is on the edge-following path $EFP_{n^-}$ and the edge-following points induced by $u_n$ are not on $EFP_{n^-}$ (except the endpoints $a,b$). Since the rest of the lattice points in the edge-following set remain unchanged, and by the induction hypothesis, $EFP_{n^-}$ goes through all the edge-following points in $EFS_{n^-}$. 

                    \begin{figure}[H]
                        \centering
                        \includegraphics[scale=0.32]{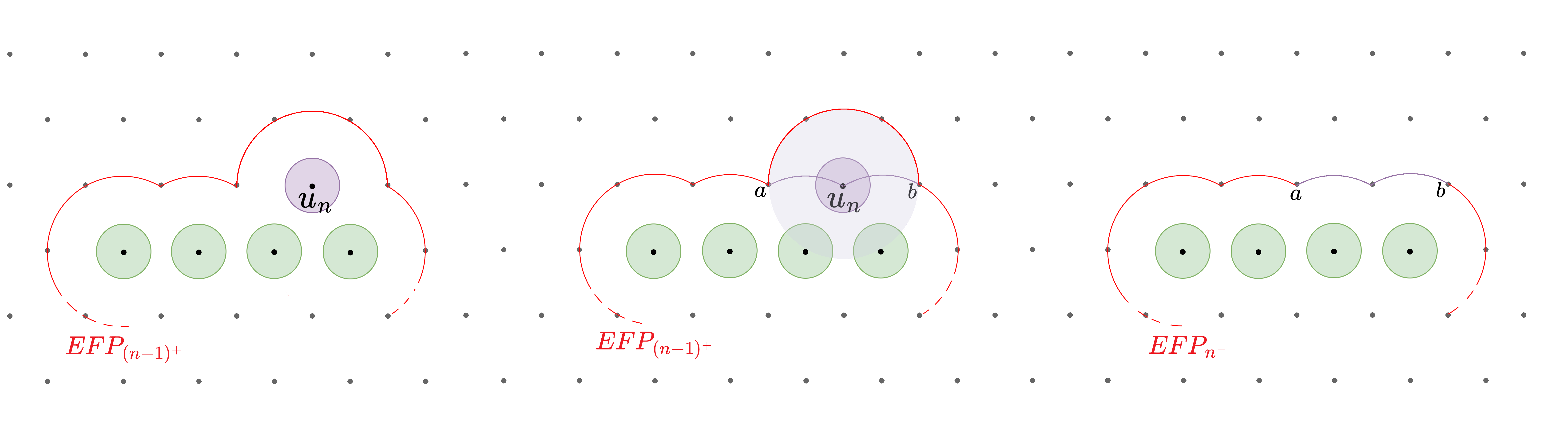}
                        \captionsetup{justification=centering}
                        %\caption{Left: the edge-following path at $T=n$; Middle: the \\ edge-following path $EFP_n$ intersect the path induced by $u_n$ \\ at point $a$ and $b$; Right: the edge-following path at $T=n+1$}
                    \end{figure}

                    However, the above arguments fail if $d(s[n])$ covers more than one path segment from $EFP_{n^-}$. We shall show that this could not happen by contradiction. Suppose not, then the disk $d(s[n])$ contains at least two path segments from $EFP_{n^-}$. Notice that $s[n]$ is on $EFP_{n^-}$, there is a path segment of $EFP_{n^-}$ that traverses $d(s[n])$, denoted as path segment $ab$. 
                    
                    Now, we investigate the neighborhood of $s[n]$. First, notice that there is at least one occupied lattice point in the neighborhood of $s[n]$ by the activation sequence. Also, notice that there must be at least three unoccupied lattice points; otherwise, there exists a pair of occupied lattice points on the opposite side of $s[n]$, which contradicts the properness of $B_{n^-}$ (by proposition \ref{prop-properBoundary}). As a result, the neighborhood of $s[n]$ has only three cases: 
                    \begin{itemize}
                        \item [-] case 1: only one lattice point is occupied;
                        \item [-] case 2: only two lattice points are occupied, and the two points are adjacent (otherwise proposition \ref{prop-properBoundary} is violated);
                        \item [-] case 3: only three lattice points are occupied, and the three points are next to each other (otherwise proposition \ref{prop-properBoundary} is violated)
                    \end{itemize}

                    The path segments $ab$ induced by the robots in the neighborhood of $s[n]$ are shown as follows:
                    \begin{figure}[H]
                        \centering
                        \includegraphics[scale=0.4]{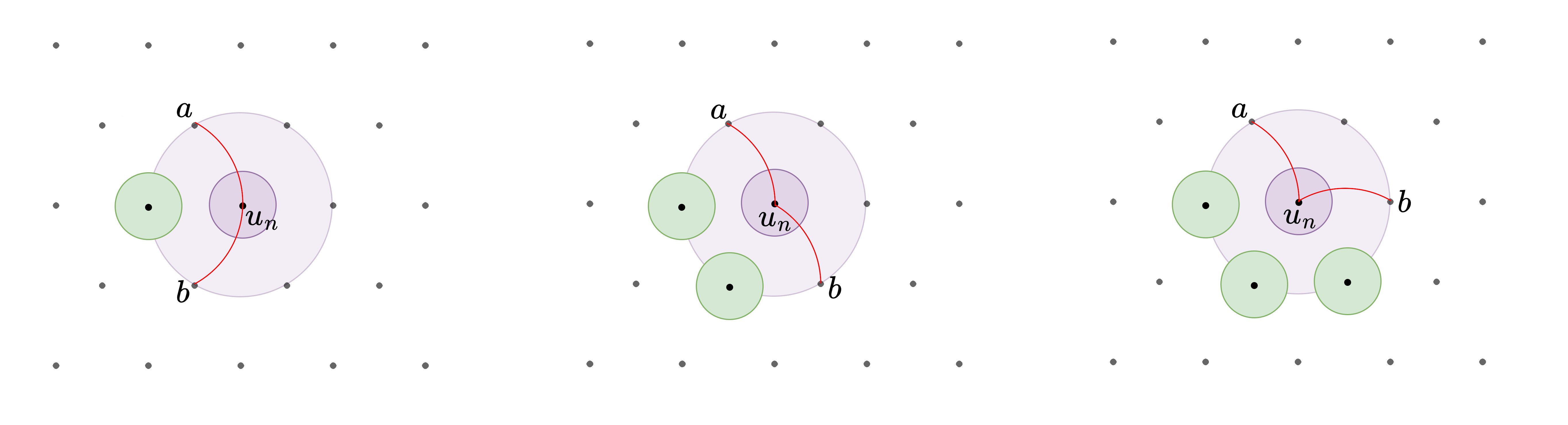}
                        \caption{Left: case 1; Middle: case 2; Right: case 3}
                    \end{figure}
                    
                    Observe that if there is another segment of $EFP_{n^-}$ in $d(s[n])$, the additional path segment cannot be induced by any robot in the neighborhood of $s[n]$. As a result, any additional path segment is induced by some robot $v$ in $B_{n^-}$ that is at a distance $\sqrt{3}d$ or $2d$ from robot $u_n$. There are two cases:

                    \begin{figure}[H]
                        \centering
                        \includegraphics[scale=0.4]{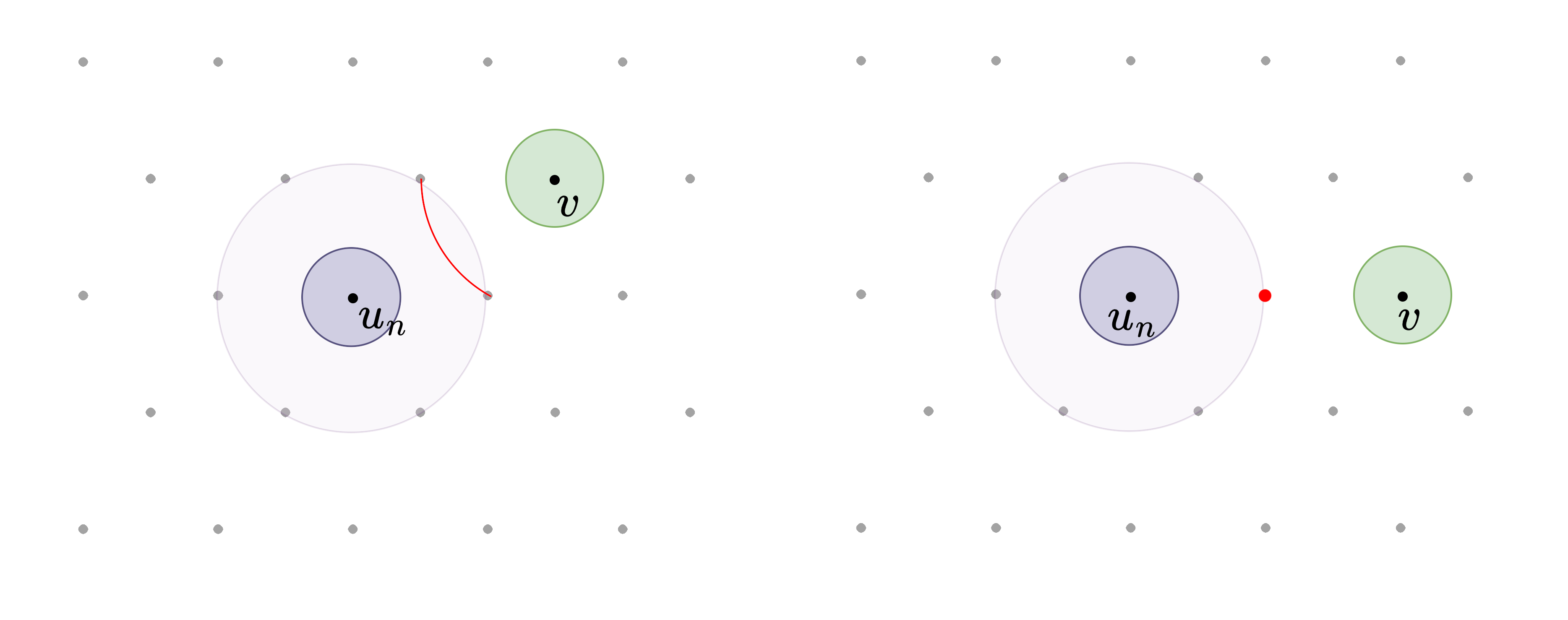}
                        %\caption{Left: $||p_(n)(v) - p_t(u_n)||=\sqrt{3}d$; Right: $||p_t(v) - p_t(u_n)||=2d$}
                    \end{figure}

                    Notice that both cases contradict the properness of $B_{(n-1)^+}$. As a result, $d(s[n])$ contains only one path segment of $EFP_{n^-}$. After $u_n$ starts to move, the edge-following path remains a circle in the topological sense, and the edge-following path $EFP_{n^-}$ goes through all the edge-following points in $EFS_{n^-}$. $P(n)$ is true.

                    \item [-] Part II:
                
                        After $u_n$ stops, some path segment(s) of $EFP_{n^-}$ is covered by the disc centered at $t[n]$ with radius $d$. We denote the disc center at $t[n]$ as $d(t[n])$. If only one path segment of $EFP_{n^-}$ is covered by $d(u_n)$, the segment is replaced by the segment of the circumference of $d(u_n)$ which is located outside the shape surrounded by $EFP_{n_-}$. We also notice that the edge-following points induced by $u_n$ are located on $EFP_{n^+}$ and the lattice point occupied by $u_n$ is not on $EFP_{n^+}$. Since all the other lattice points in the edge-following set remain unchanged, and by the induction hypothesis, $EFP_{n^+}$ goes through all the edge-following points after $u_n$ stops.
                        
                        % We denote the path segment(s) induced by $u_n$ as $\widetilde{ab}$. Notice that if $\widetilde{ab}$ consists of a single curve, $EFP_{n+1}$ is still homeomorphic to a circle as shown below.

                        \begin{figure}[H]
                            \centering
                            \includegraphics[scale=0.35]{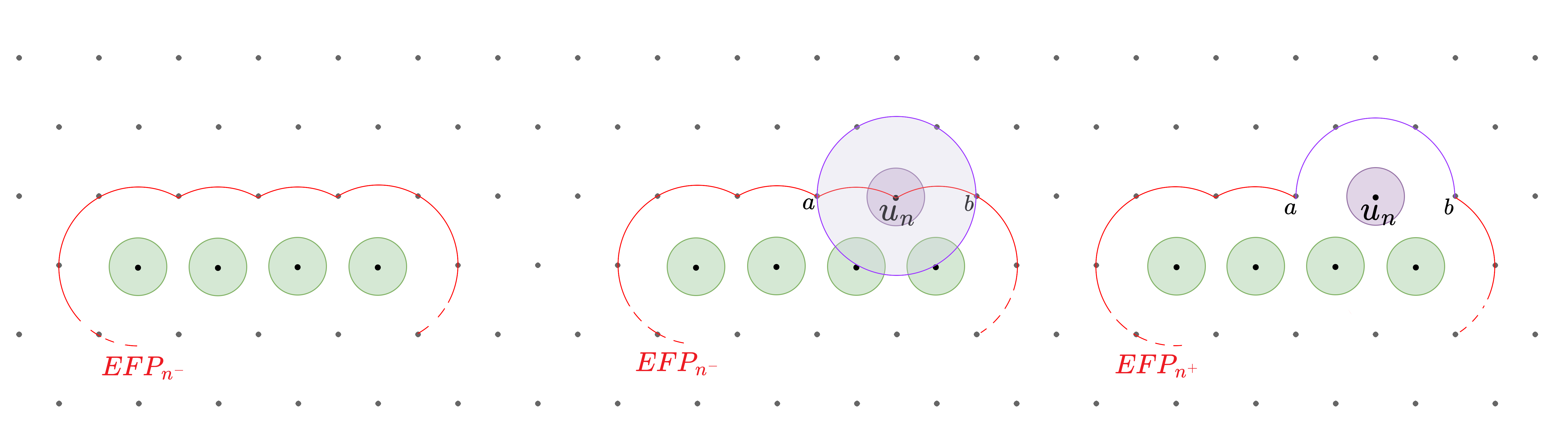}
                            \captionsetup{justification=centering}
                            \caption{The path segment $ab$ in red is replaced by the path segment $ab$ in purple.}
                            %\caption{Left: the edge-following path at $T=n$; Middle: the \\ edge-following path $EFP_n$ intersect the path induced by $u_n$ \\ at point $a$ and $b$; Right: the edge-following path at $T=n+1$}
                        \end{figure}

                        However, the above arguments fail if $d(t[n])$ covers more than one path segment of $EFP_{n^-}$. By similar arguments as in part I, by using the properness of $B_{n^-}$ and $B_{n^+}$, there is no robot in $B_{n^-}$ that could induce any additional path segment in the disk $d(t[n])$. As a result, after $u_n$ stops, the edge-following path $EFP_{n^+}$ remains a circle in the topological sense, and the edge-following path goes through all the edge-following points in $EFS_{n^+}$. $Q(n)$ is true.

                \end{itemize}

                %Lastly, observe that the neighborhood of $u_{n+1}$ contains at least two unoccupied lattice points of $EFS_n$ and the unoccupied lattice points are connected. As a result, removing $u_{n+1}$ from the swarm will shrink the edge-following path but will not affect its topological shape.

        \end{itemize}
        Since $Q(0)$ is true, by induction, $P(n)$ and $Q(n)$ are true for all $1 \leq n \leq N$.
    \end{proof}

    \begin{definition}
        At epoch $T = n \leq N$, the \textbf{assembly path} for active robot $u_n$ is defined as the path segment of $EFP_{n^-}$ from $s[n]$ to $t[n]$, directed clockwise.
    \end{definition}

    \begin{lemma}
        \label{thm-unobstructedPath}
        At epoch $T=n \leq N$, the assembly path for robot $u_n$ is well-defined and unobstructed.
    \end{lemma}
    \begin{proof}
        By proposition \ref{prop-startEndBothInEFS}, $s[n]$ and $t[n]$ are in the edge-following set. By proposition \ref{prop-EFPisCircle}, the edge-following path is homeomorphic to a circle and links all the lattice points in the edge-following set. As a result, the assembly path is well-defined. The assembly path is also unobstructed since any point on the edge-following path is at $d$ to the nearest unoccupied lattice point.
    \end{proof}

\subsection{Localization for the Additive Stage}
\label{sect-localization}
% During the execution of the algorithm, robots need their coordinates to establish their location and make corresponding decisions. We show in section 4.5 that every non-moving robot can correctly obtain its coordinates, and every moving robot can establish its coordinates as long as the moving robot is in the half-plane $H_1$. 

     During the execution of the algorithm, each active robot $u_n$ in the additive stage moves from lattice point $s[n]$ to lattice point $t[n]$. To follow the edge-following path, $u_n$ maintains a distance $d$ to the nearest robot and rotates clockwise. Although the movement only requires local distance information, the moving robot $u_n$ needs to know its coordinate at $s[n]$ and $t[n]$. The process of obtaining the coordinates is called localization. In this section, we formally define localization and prove the localizability of the MRS at the initial configuration and during the additive stage.

    \begin{definition}
        A robot $u$ is \textbf{localizable} if
        \begin{enumerate}
            \item it is one of the seed robots; or
            \item it is localizable at a previous time step and it has not moved since then; or 
            \item it could determine its coordinates using trilateration. I.e., there exist three localizable robots such that:
                    \begin{itemize}
                        \item the three robots are less than or equal to $2d$ from the robot $u$, and
                        \item the three robots are noncollinear.
                    \end{itemize}
        \end{enumerate}
    \end{definition}

    \begin{proposition}
        \label{prop-IdleLocalizable}
        All idle robots are localizable.
    \end{proposition}
    \begin{proof}
        We first prove that each idle robot is localizable at $T = 0$. Let $P(n)$ denote "all the idle robots belonging to $IR_n$ is localizable at $T=0$".
        \begin{itemize}
            \item [-] Base cases:
                        \begin{itemize}
                            %\item If we include $IR_{-1}$ as the canonical ribbon induced by $p_4,p_5$, then $P(1)$ is vacuously true.
                            \item $\forall u \in IR_{0}$,  since $u$ neighbors the clique of three seed robots $v_0, v_3, v_4$, $P(0)$ follows.
                            \item $\forall u \in IR_{1}$,  $u$ neighbors a $v$ robot in $IR_{0}$. $v$ neighbors a robot $w \in IR_{0}$ and a seed robot $k$. By inspecting the two types of neighborhood of robot $v$, $v,w,k$ are non-collinear. Since $v,w,k$ are localizable, $P(1)$ follows.
                        \end{itemize}
            
            \item [-] Inductive step: suppose $P(0), P(1), \ldots, P(n)$ are true, we shall show that $P(n+1)$ is true.
              
                        $\forall u \in IR_{n+1}$, $u$ neighbors an idle robot $v$ in $IR_{n}$. Since every idle ribbon contains at least two robots, there is a robot $w \in IR_{n}$ neighboring $v$. Notice that $v$ neighbors at least one robot in $IR_{n-1}$, and we denote the robot as $k$. By inspecting the type A and the type B neighborhood of robot $v$, $k,w,v$ cannot be collinear. Since $P(n)$ and $P(n-1)$ are true, $k,w,v$ are localizable. Finally, observe that the centers of $k,w,v$ are less than $2d$ from the center of robot $u$, $P(n+1)$ follows.
        \end{itemize}
        By induction, all idle robots are localizable at $T=0$. Since an idle robot remains stationary until it becomes active, the proposition follows.
    \end{proof}

    \begin{proposition}
        \label{prop-3robotsForLocalization}
        Given a proper ribbonization $\hat{S}$, an active robot $u_n$ is localizable if it neighbors a seed robot or a robot that belongs to a filled ribbon.
    \end{proposition}
    \begin{proof}
        We shall do induction on the epoch $T$. Let $P(n)$ be "$u_n$ is localizable if it neighbors a seed robot or a robot that belongs to a filled ribbon."
        \begin{itemize}
            \item [-] Base case: Suppose the active robot $u_1$ neighbors robot $v$ and $v$ is a seed robot. Then $v$ neighbors two seed robots. The three robots form a clique (noncollinear) and the centers of the three robots are within $2d$ from the center of robot $u$. Since all the seed robots are localizable, $P(1)$ is true.
            \item [-] Inductive step: Suppose $P(n-1)$ is true, we will show that $P(n)$ is true. Suppose the active robot $u_n$ neighbors robot $v$.
            \begin{itemize}
                \item [-] If $v$ is a seed robot, then $v$ neighbors two seed robots. The three robots form a clique (noncollinear) and the centers of the three robots are within $2d$ from the center of robot $u$. $P(n)$ is true.
                
                \item [-] If $v$ is not a seed robot. Suppose $v$ belongs to a filled ribbon $R$. Then $v$ neighbors a robot $k$ with gradient value $g(k) = g(v)-1$. Since every filled ribbon contains at least two localized robots, there is a robot $w \in R$ neighboring $v$. By inspecting the type A and the type B neighborhood of robot $v$, we notice that $k,w,v$ cannot be collinear. Since the centers of $k,w,v$ are less than $2d$ from the center of robot $u$ and $k,w,v$ are localizable by induction hypothesis, $P(n)$ is true.
            \end{itemize}
        \end{itemize}

    \end{proof}

    \begin{lemma}
        \label{thm-localizable}
        At epoch $T=n$, robot $s_n$ is localizable at lattice point $s[n]$ and lattice point $t[n]$.
    \end{lemma}
    \begin{proof}
        By proposition \ref{prop-IdleLocalizable}, robot $s_n$ is localizable at lattice point $s[n]$. By the ribbon order, $t[n]$ is adjacent to a lattice point occupied by a seed robot or a vertex of a filled ribbon. Robot $s_n$ is localizable at lattice point $t[n]$ by proposition \ref{prop-3robotsForLocalization}.
    \end{proof}

    % \begin{theorem}
    %     \label{thm-additiveStageCorrect}
    %     At epoch $1 \leq T \leq N$, robot $s_n$ executing the edge-following movement stops at $t[n]$.
    % \end{theorem}
    % \begin{proof}
    %     By lemma \ref{thm-unobstructedPath} and lemma \ref{thm-localizable}, the theorem follows.
    % \end{proof}

    \subsection{Properties of the MRS At the End of the Additive Stage}
    \label{sect-propertiesOfEOA}

    In this section, we investigate the properties of the MRS after robot $u_{N}$ stops. Since the MRS grows to its maximum size in the formation half-plane after $u_{N}$ stops, we say the MRS is in its "culminating" state. The "culminating" state plays an important role as it serves as the starting configuration of the subtractive stage. We shall investigate the "culminating" state, its corresponding boundary and edge-following path.

    \begin{definition}
        \leavevmode
        \makeatletter
        \@nobreaktrue
        \makeatother
        \begin{enumerate}
            \item The \textbf{culminating set of boundary robots}, $\bm{\widehat{B}}$, is the collection of all boundary robots after $u_{N}$ stops and before $u_{N+1}$ moves, i.e., $\widehat{B} = B_{N^+}$.
            %The \textbf{culminating set of boundary points} is the collection of center points of the boundary robots.
            \item The \textbf{culminating edge-following set}, $\bm{\widehat{EFS}}$, is the set that contains all the edge-following points after $u_{N}$ stops and before $u_{N+1}$ moves, i.e., $\widehat{EFS} = EFS_{N^+}$.
            \item The \textbf{culminating edge-following path}, $\bm{\widehat{EFP}}$, is the path consisting of all points that are at a distance $d$ from the nearest robot at the end of the additive stage, i.e., $\widehat{EFP} = EFP_{N^+}$.
        \end{enumerate}

        % \begin{figure}[H]
        %     \centering
        %     \includegraphics[scale=0.28]{figures/proofHoles/proofNonCanonicalWithHoles-EFS&BoundaryPointsLight.png}
        %     \includegraphics[scale=0.28]{figures/proofHoles/proofNonCanonicalWithHoles-culminatingEFP.png}
        %     \caption{Left: The black dots indicate the culminating set of boundary points, and the red dots indicate the edge-following set $\widehat{EFS}$. Right: the culminating edge-following path $\widehat{EFP}$}
        % \end{figure}
    
    \end{definition}

    Notice that the culminating set of boundary robots $\widehat{B}$ is proper by proposition \ref{prop-properBoundary}. The culminating edge-following path is homeomorphic to a circle and goes through all the edge-following points by proposition \ref{prop-EFPisCircle}. 
    
    %Recall that in the subtractive stage, a recycle robot rotates around the perimeter of the swarm and stops in $H_2$. To facilitate the path planning for recycle robots, we shall investigate the stopped robots that contribute to the culminating edge-following path: 

    \begin{definition}
        \leavevmode
        \makeatletter
        \@nobreaktrue
        \makeatother
        \begin{enumerate}
            \item At sub-epoch $T=n^-$ with $0 < n \leq N$, the \textbf{affiliated path} $\bm{AP_{n^-}}$ is the path segment of the edge-following path $EFP_{n^-}$ starts from the lattice point $p_6$ and ends at the lattice point $p_3$, oriented clockwise. Similarly, the \textbf{affiliated path} $\bm{AP_{n^+}}$ is the corresponding path segment of the edge-following path $EFP_{n^+}$ at sub-epoch or $T = n^+$ with $0 \leq n \leq N$.
            \item Let the \textbf{culminating affiliated path} $\bm{\widehat{AP}}$ be the path segment of the culminating edge-following path $\widehat{EFP}$ starts from the lattice point $p_6$ and ends at the lattice point $p_3$, oriented clockwise. I.e., $\widehat{AP} = AP_{N^+}$.
            % $\widehat{AP}$ inherits its direction from the directed culminating edge-following path.
        \end{enumerate}

        \begin{figure}[H]
            \centering
            \captionsetup{justification=centering}
            \includegraphics[scale=0.3]{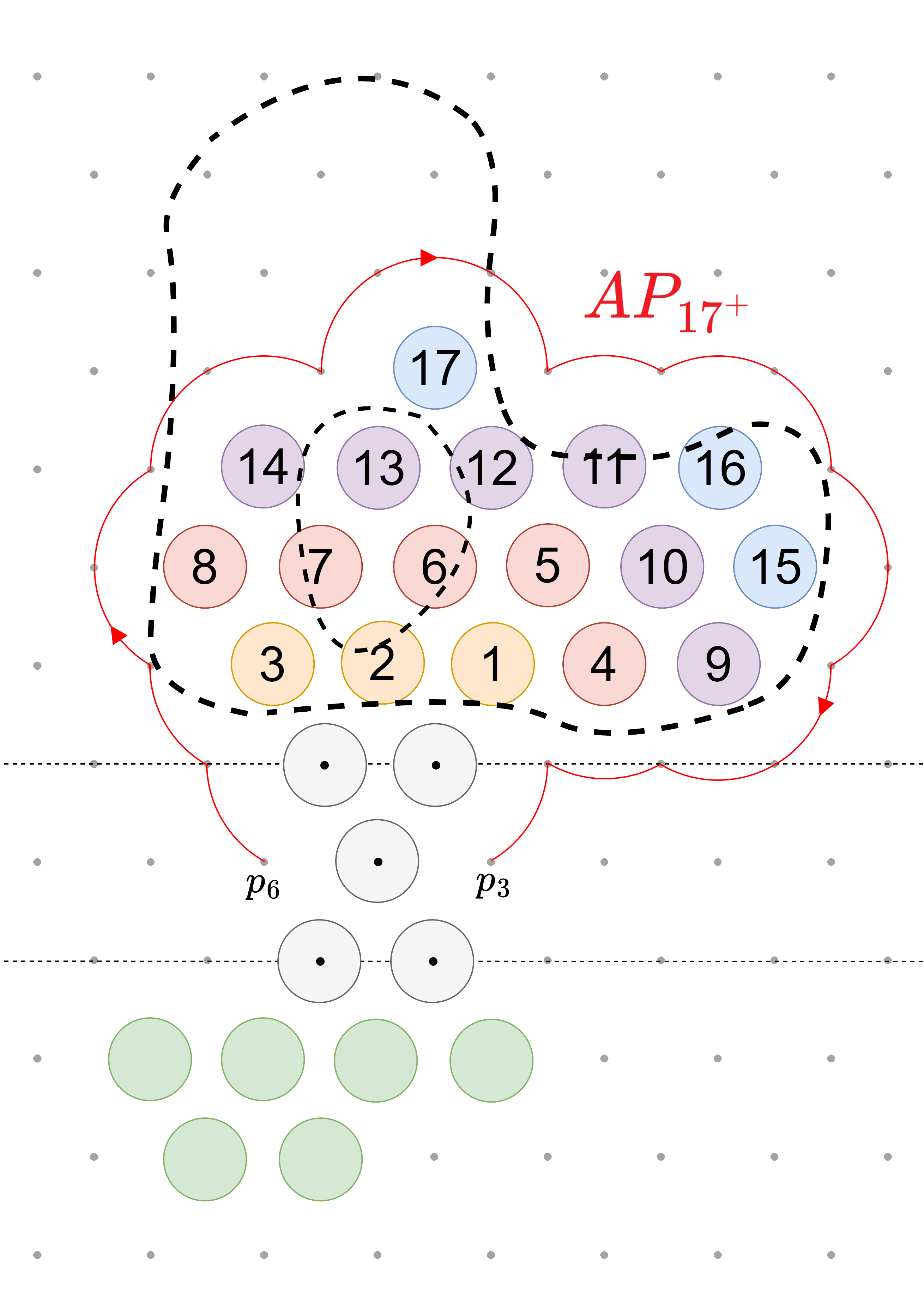}
            \includegraphics[scale=0.3]{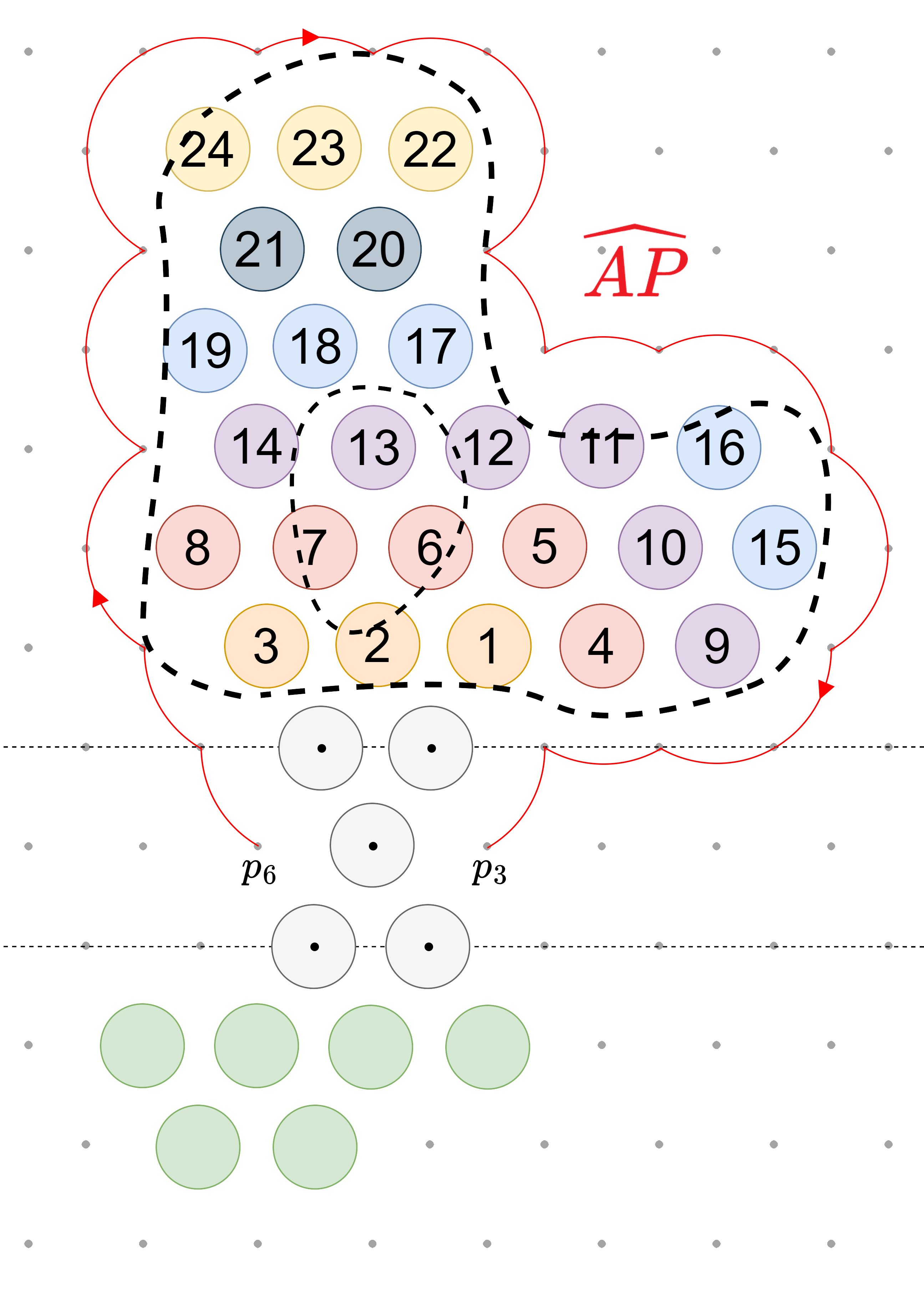}
            \caption{Left: the affiliated path $AP_{17^+}$; Right: the culminating affiliated path $\widehat{AP}$}
        \end{figure}
    \end{definition}

    % \begin{definition}

    %     \begin{figure}[H]
    %         \centering
    %         \captionsetup{justification=centering}
    %         \includegraphics[scale=0.4]{figures/proofHoles/proofNonCanonicalWithHoles-culminatingAP.png}
    %         \caption{The culminating assembly path}
    %     \end{figure}
    % \end{definition}

    To facilitate the path planning for the subtractive stage, we define the "merging point" for each ribbon. Loosely speaking, the "merging point" is a special point on the affiliated path that marks the entrance of a recycle robot to the affiliated path.

    \begin{definition}
        For each ribbon $R$, suppose the first vertex of $R$ is $x$, then a lattice point $m \notin S$ is a \textbf{merging point} of ribbon $R$ if
        \begin{enumerate}
            \item $||m-x||=d$; and
            \item $(\exists y \in V(P(R)))$ $||y-x||=d \land ||y-m||=d$.
            % \item $x$ is adjacent to the first vertex of ribbon $R$; and
            % \item $x$ is adjacent to a vertex of the parent ribbon of $R$. 
        \end{enumerate}
        The merging point of $R$ is denoted as $m(R)$.

        % For each ribbon $R$, suppose the first vertex of $R$ is $x$, then a lattice point $m \notin S$ is a \textbf{merging point} of ribbon $R$ if
        % \begin{enumerate}
        %     \item $m$ is adjacent to $x$; and
        %     \item $\exists y$ of the parent ribbon of $R$ such that $y$ is adjacent to $x$ and $m$.
        %     % \item $x$ is adjacent to the first vertex of ribbon $R$; and
        %     % \item $x$ is adjacent to a vertex of the parent ribbon of $R$. 
        % \end{enumerate}

        % \begin{figure}[H]
        %     \centering
        %     \captionsetup{justification=centering}
        %     \includegraphics[scale=0.3]{figures/proofHoles/proofNonCanonicalWithHoles-mergingPoint.png}
        %     \caption{Illustration of the culminating assembly path and merging points (as blue dots)}
        % \end{figure}
    \end{definition}
    
    \begin{proposition}
        Given a proper ribbonization of $S$, every ribbon has exactly one merging point.
    \end{proposition}
    \begin{proof}
        For arbitrary ribbon $R$, suppose the first vertex of $R$ is $x$ and $h(x) = k$. We investigate the two types of neighborhood of $x$:
    
        \begin{figure}[H]
            \centering
            \includegraphics[scale=0.5]{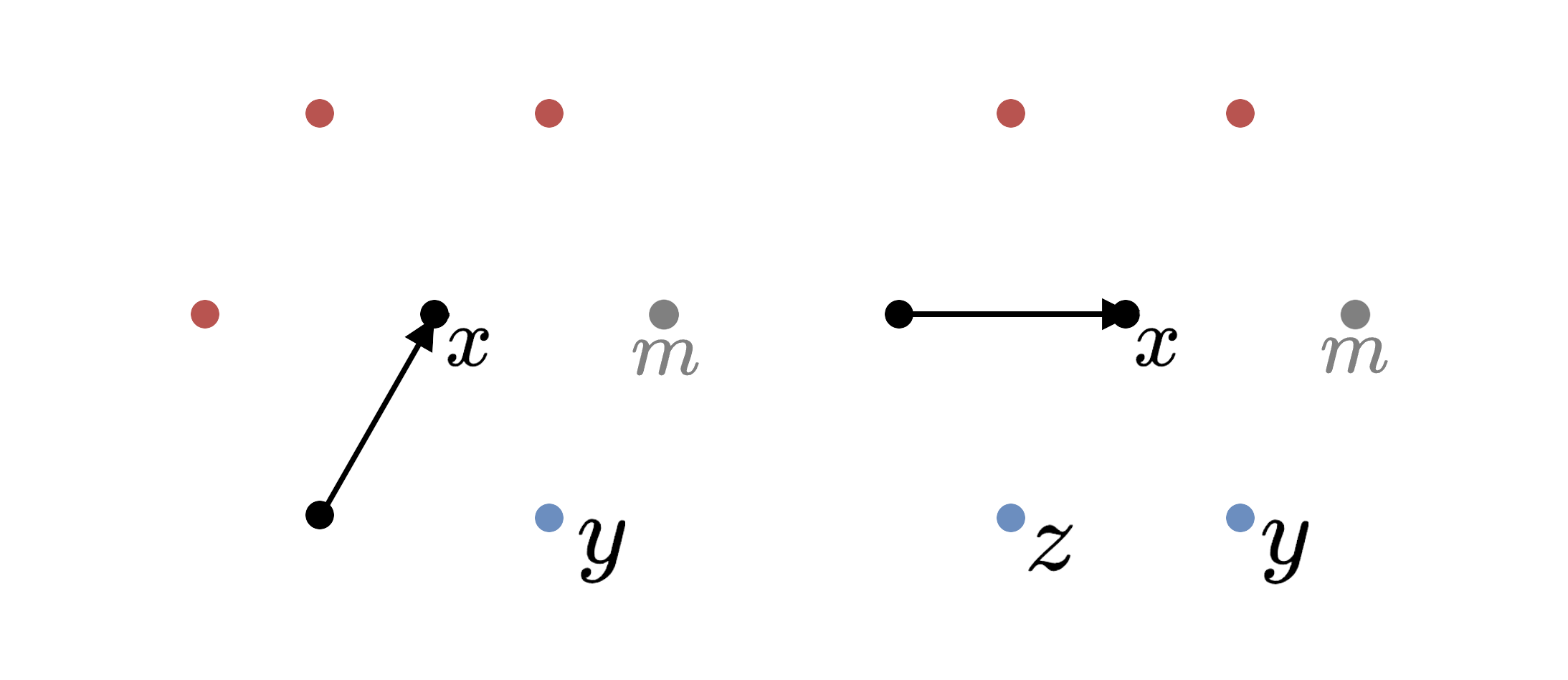}
            \caption{Left: type A neighborhood of $x$; Right: type B neighborhood of $x$}
        \end{figure}
    
        \begin{itemize}
            \item [-] If $x$ has the type $A$ neighborhood, denote the lattice point in the neighborhood with a hop count $k-1$ as $y$. Two lattice points in the neighborhood of $x$ are adjacent to both $x$ and $y$, and we denote the lattice point on the counter-clockwise side of $x$ with respect to $y$ as $a$, and the lattice point on the clockwise side of $x$ as $b$. Since $R$ has a minimum length of two, $a$ is a vertex of $R$ and $a$ is behind $x$. As a result, $b$ is not in $S$, and $b$ is a merging point of $R$. Now, we show that $b$ is the only merging point of $R$.  
            
            Suppose not, then there is a lattice point $w$ in the rest of the lattice points in the neighborhood of $x$ such that $h(w) = k-1$. This contradicts proposition \ref{prop-typeOfNeighborhood}. As a result, $b$ is the only merging point of $R$.

            \item [-] If $x$ has the type $B$ neighborhood, denote the two lattice points in the neighborhood with a hop count $k-1$ as $y$ and $z$. WLOG, we let $y$ on the counter-clockwise side of $z$ with respect to $x$. By similar arguments as in the first case, the lattice point adjacent to both $x$ and $y$ is the only merging point of $R$.
        \end{itemize}

        % Notice that in both cases, since $x$ is the first vertex of ribbon $R$, lattice point $m \notin S$. Since $m,x,y$ forms a clique of three adjacent lattice points, $m$ is a merging point. If there is a merging point $m'$ other than $m$, then $m' \notin S$. $h(m')$ $h(m') \neq k+1$ since $m'$ and $x$ must be adjacent to the same lattice point with hop count $k-1$ (i.e., a lattice point in blue). As a result, $h(m') = k-1$. Since the the blue lattice point is a vertex of the parent ribbon of $R$, $m' \in S$. We reach a contradiction.     
    \end{proof}

    \begin{definition}
        \leavevmode
        \makeatletter
        \@nobreaktrue
        \makeatother
        \begin{enumerate}
            \item Given two points $x,y$ on the affiliated path, $x$ is \textbf{in front of} $y$ if the path segment of the edge-following path from $x$ to $y$ has the same direction as the affiliated path. %We also say that $x$ is in front of $y$ if $x=y$. 
            \item Given two points $x,y$ on the assembly path, $x$ is \textbf{strictly in front of} $y$ if $x$ is in front of $y$ and $x \neq y$.
            
            % \item Given point $x$ and path segment $P$ on the assembly path, the path segment $P$ is \textbf{in front of} the point $x$ iff all the points of $P$ is \textbf{in front of} $x$. Similarly, the point $x$ is \textbf{in front of} the path segment $P$ iff $x$ is in front of all the points of $P$. 
            % \item Given two path segment $P,Q$ on the assembly path, the path segment $P$ is \textbf{in front of} the path segment $Q$ iff all the points of $P$ is \textbf{in front of} all the points of $Q$. 
        \end{enumerate}

    \end{definition}

    We say $y$ is \textbf{behind} $x$ when $x$ is in front of $y$, and $y$ is \textbf{strictly behind} $x$ when $x$ is strictly in front of $y$. 

    \begin{proposition}
        \label{prop-mergingPointInfrontLemma}
        For ribbon $R$ with merging point $m$, if $m$ is on the affiliated path, then $m$ is in front of all the points on the affiliated path that are contributed by robots from $R$.
    \end{proposition}
    \begin{proof}
        We shall prove the proposition by induction on epoch $T$. Let $P(n)$ denotes "$\forall x \in AP_{n^-}$, if $x$ is contributed by a robot $u \in R$, then $m$ is in front of $x$". Let $Q(n)$ denotes "$\forall x \in AP_{n^+}$, if $x$ is contributed by a robot $u \in R$, then $m$ is in front of $x$".

        \begin{itemize}
            \item [-] Base case: $Q(0)$ is vacuously true since there is no point on $AP_{0^+}$ which is contributed by robots of $R$.

            \item [-] Inductive step: suppose $Q(n-1)$ is true, we want to show $P(n)$ is true; and suppose $P(n)$ is true, we want to show that $Q(n)$ is true. Since $s[n]$ is in the idle half-plane $H_2$, the non-moving robots in $H_1$ and the affiliated path remain unchanged after $u_n$ starts to move. As a result, $Q(n-1) \rightarrow P(n)$. We are left to show $P(n) \rightarrow Q(n)$.
            In other words, we want to show that all the points contributed by robots of $R$ are behind $m$ after $u_{n}$ stops. We shall discuss the following three cases:
            
            % Suppose $u_n$ belongs to ribbon $R$ after stops and $u_n$ stops at lattice point $x$. Denote the parent ribbon of $R$ as $R_p$.  

            \begin{itemize}
                \item [-] Case 1: $t[n]$ is the first vertex of ribbon $R$.
                \begin{figure}[H]
                    \centering
                    \includegraphics[scale=0.4]{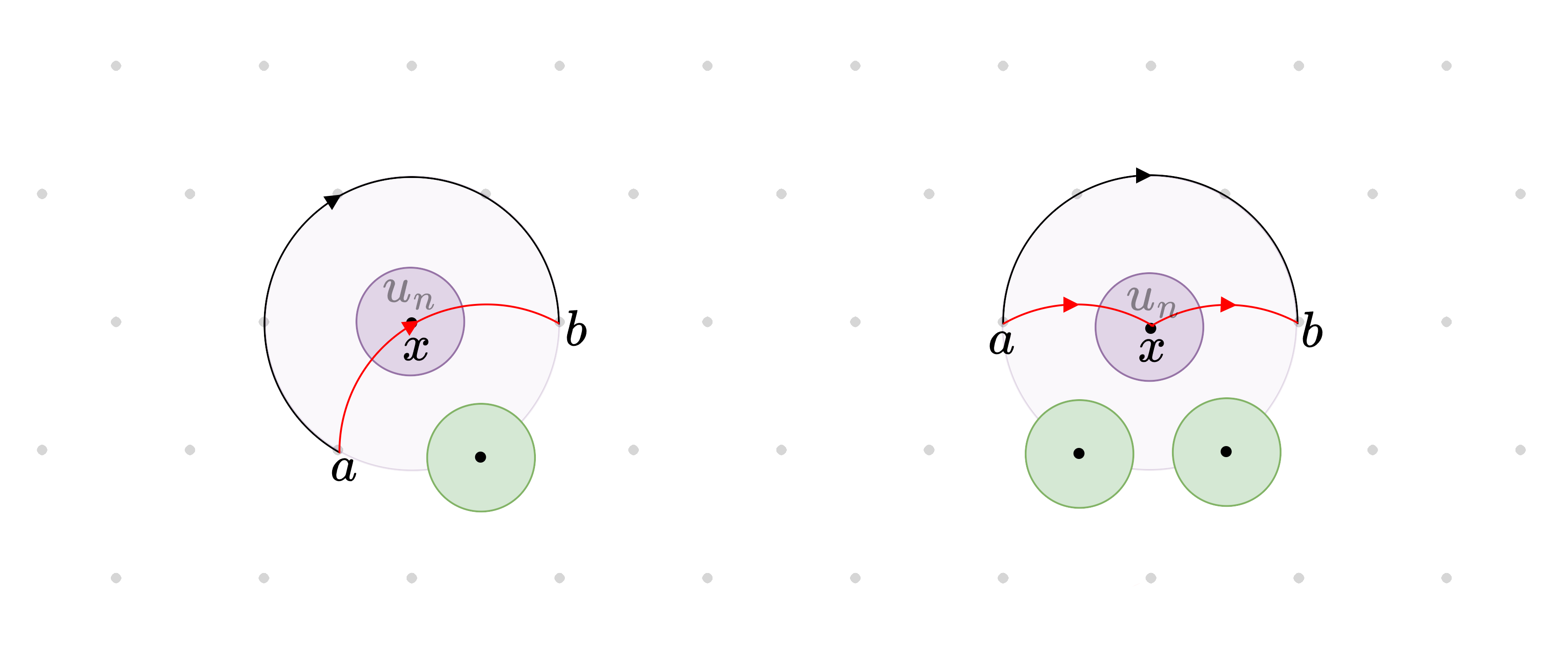}
                    \caption{Left: $x$ has type B neighborhood; Right: $x$ has type A neighborhood }
                \end{figure}
                Recall in the proof of proposition \ref{prop-EFPisCircle}, after $u_n$ stops, the path segment $ab$ that traverses the disk $d(r[n])$ is replaced by the path segment induced by $u_n$. In this case, point $b$ is the merging point of ribbon $R$. Points on the path segment $axb$ are contributed by a robot(s) from $R^{[p]}$. After $u_n$ stops, the path segment $axb$ (in red) is replaced by path segment $ab$ (in black) with the merging point $m=b$ in the front. As a result, the merging point of $R$ is in front of all points of $AP_{n^+}$ that are contributed by robot $u_n$. Since $AP_{n^+}$ and $AP_{n^-}$ only different by the path segment between $a$ and $b$, $P(n) \rightarrow Q(n)$. 
                
                \item [-] Case 2: $t[n]$ is a vertex of $R$ but $t[n]$ is not the first vertex of $R$.
                \begin{figure}[H]
                    \centering
                    \includegraphics[scale=0.4]{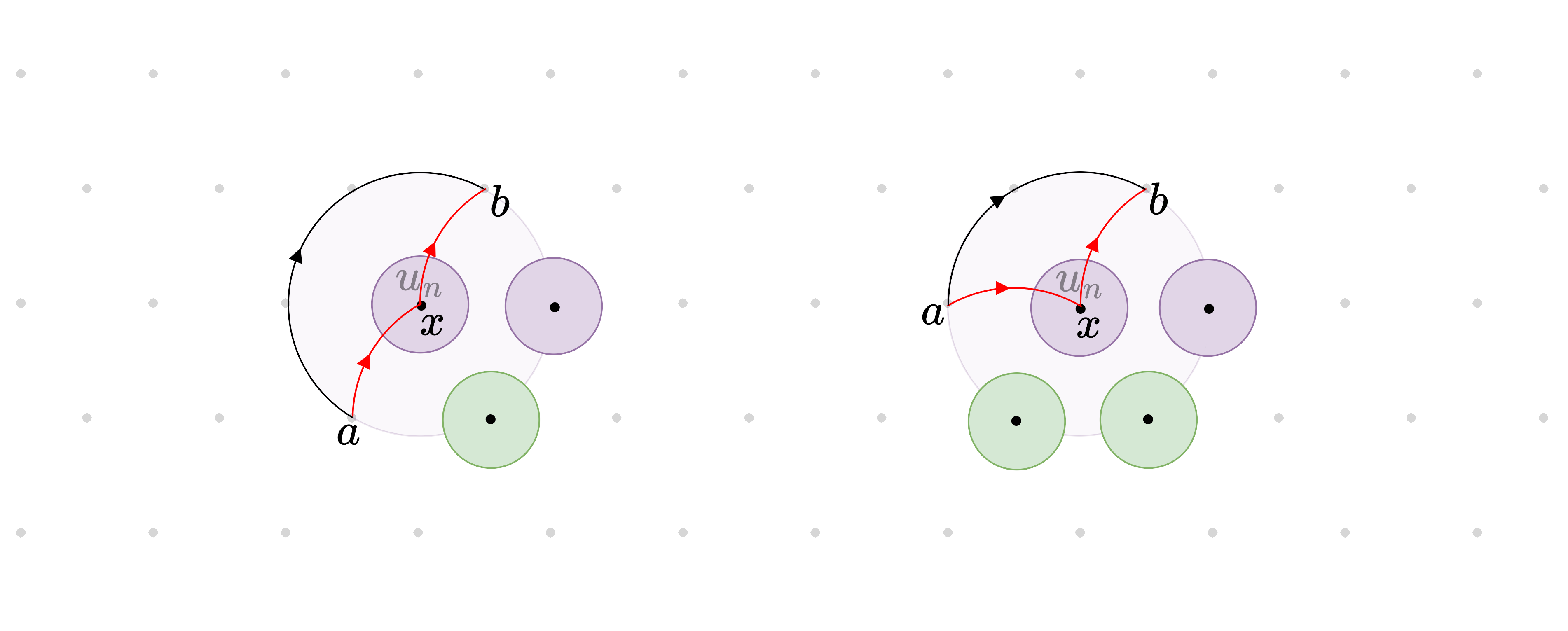}
                    \caption{Left: $x$ has type A neighborhood; Right: $x$ has type B neighborhood}
                \end{figure}
                Notice that points on the arc $ax$ are contributed by a robot that belongs to $R_p$ and points on the arc $xb$ are contributed by a robot that belongs to $R$. After $u_n$ stops, the path segment $axb$ is replaced by path segment $ab$ that is contributed by $u_n$. By the induction hypothesis, $b$ is behind the merging point of ribbon $R$. Since all points on $ab$ are behind $b$, all points on $ab$ are behind the merging point of $R$. As a result, the merging point of $R$ is in front of all points of $AP_{n^+}$ that are contributed by robot $u_n$. Since $AP_{n^-}$ and $AP_{n^+}$ only different by the path segment between $a$ and $b$, $Q(n)$ follows. 

                \item [-] Case 3: $t[n]$ is not a vertex of $R$. If $t[n]$ is a vertex of an ancestor ribbon of $R$, $R$ is unfilled and $Q(n)$ is vacuously true. If $t[n]$ is not a vertex of an ancestor ribbon of $R$, then after $u_n$ stops, the path segment of $AP_{n^-}$ that traverses the disk $d(r[n])$ is replaced by the path segment induced by $u_n \notin R$. As a result, $P(n) \rightarrow Q(n)$.

            \end{itemize}
            
        \end{itemize}
        Since $Q(0)$ is true, by induction, $P(n)$ and $Q(n)$ are true for all $1 \leq n \leq N$.
    \end{proof}

    \begin{proposition}
        \label{prop-mergingPointInFront}
        For ribbon $R$ with merging point $m$, $m$ is in front of all the points on the culminating affiliated path $\widehat{AP}$ that are contributed by robots from $R$.
    \end{proposition}
    \begin{proof}
        First, note that $m$ is on the culminating edge-following path $\widehat{EFP}$ since $m$ is adjacent to a non-moving robot in $S$ at sub-epoch $T = N^+$. As a result, $m$ is on $\widehat{AP}$. By proposition \ref{prop-mergingPointInfrontLemma}, the proposition follows.

        % By proposition \ref{prop-mergingPointInfrontLemma}, $\forall x \in \widehat{AP}$, if $x$ is contributed by a robot $u \in R$, then $m$ is in front of $x$. Since $u_{EOA}$ does not occupy the first lattice point of ribbon $R$ after stops (otherwise, there will be a ribbon with length one, which contradicts the proper ribbonization of $S$), by the same arguments in case 2 of the proof in proposition \ref{prop-mergingPointInfrontLemma}, we conclude that "$\forall x \in \widehat{AP}$, if $x$ is contributed by a robot $u \in R$, then $m$ is in front of $x$".
    \end{proof}

    \begin{corollary}
        \label{cor-mergingPoint}
        For ribbon $R_i$ with merging point $m$, suppose ribbon $R_j$ is a descendant ribbon of $R_i$ and $R_j$ has merging point $n$, then $m$ is in front of $n$ (w.r.t $\hat{AP}$).
    \end{corollary}
    \begin{proof}
        First, notice that by the definition of merging point, $n$ is contributed by a robot of $R_i$. By proposition \ref{prop-mergingPointInFront}, $m$ is in front of all the points on the affiliated path that are contributed by robots from $R_i$. The corollary follows.
    \end{proof}

\subsection{Reactivation Sequence and Reassembly Sequence}
\label{sect-extendedSequence}

During the subtractive stage, extra robots of each layer move outside $S$ while others re-arrange themselves along the layer to fill up the lattice points in $S \setminus D$. The robots moving out of $S$ move along the perimeter of the MRS and join the idle robots in the idle half-plane $H_2$. We call such robots "recycled robots". The remaining robots of each layer move within the layer and redistribute to occupy lattice points in $S \setminus D$, and we call such robots "rearranged robots". 

As the counterpart of the "activation sequence" and the "assembly sequence", we use the "reactivation sequence" and the "reassembly sequence" to denote which robot in $S$ should become active and where it shall stop during the subtractive stage. At epoch $T = N+k$:

\begin{itemize}
    \item [-] the reactivation sequence $s_r[k]$: the $k$th entry denotes the lattice point at which the robot $u_{(N+k)}$ start to move;
    \item [-] the reassembly sequence $t_r[k]$: the $k$th entry denotes the lattice point at which the robot $u_{(N+k)}$ should stop.
\end{itemize}    

We shall use the orderings induced by the ribbon structure to formulate the reactivation sequence and the reassembly sequence.

% Loosely speaking, every stopped robot in $S$ is reactivated once and the robots of the outermost ribbon are reactivated first. The extra robots of each ribbon move out of the shape and reassemble in the idle half-plane, and the remaining robots of the ribbon rearrange themselves to occupy the vertices in $S \setminus D$ of the same ribbon. 

\begin{definition}
    The \textbf{reactivation sequence} $s_r[k]$ is defined recursively such that each element $s_r[k]$ is the smallest lattice point in $S$ (w.r.t the ribbon order) on the largest ribbon (w.r.t the complete tree order) which has not been included in ${s_r}_{[1:k-1]}$. 

    % \begin{enumerate}
    %     \item $s_r[\text{ }]$ contains all and only lattice points in $S$;
    %     \item $\forall s_r[i], s_r[j] \in s$ with $ 1 \leq i < j$, $R(s_r[i])>R(s_r[j])$ in the complete tree order and $s_r[i]<s_r[j]$ in the ribbon order.
    % \end{enumerate} 
    
\end{definition}

Notice that every robot in $H_1$ occupies a lattice point in $S$, and the lattice point appears in the reactivation sequence. At epoch $T= n = N+k$, the robot occupying lattice point $s_r[k]$ is the active robot, denoted as robot $u_n$. Also, notice that every lattice point in $S$ appears exactly once in the reactivation sequence. As a result, every stopped robot becomes active exactly once during the subtractive stage. 

To formulate the reassembly sequence, we investigate the recycled robots and the rearranged robots separately:

\begin{definition}
    For robot $u$ that occupies lattice point $x \in S$ at sub-epoch $T = N^+$,
    \begin{itemize}
        \item [-] $u$ is a \textbf{recycled robot} if $ID(x) < c(R(x))$; or
        \item [-] $u$ is a \textbf{rearranged robot} if $ID(x) \geq c(R(x))$.
    \end{itemize} 
\end{definition}

% A recycle robot $u$ will exit the shape $S$ and stop in the idle half-plane $H_2$; and a rearranged robot $u$ will reposition itself inside $S \setminus D$.

\begin{definition}
    \leavevmode
    \makeatletter
    \@nobreaktrue
    \makeatother
    \begin{enumerate}
            \item The \textbf{rearrangement sequence} is a sequence of lattice points in $S \setminus D$, defined recursively such that the $k$th element is the smallest lattice point in $S \setminus D$ (w.r.t the ribbon order) on the largest ribbon (w.r.t the complete tree order) which has not yet been included in the rearrangement sequence.

            \item The \textbf{recycle sequence} is a sequence containing lattice points in $H_2$, defined recursively such that the $k$th element is the smallest lattice point in $s_{[1:N]}$ (w.r.t the idle ribbon order) on the smallest idle ribbon (w.r.t the idle tree order) which has not yet been included in the recycle sequence.
            %that indicates which idle robot should transit into the moving state. We require that each robot transits into the moving state satisfies
            % \begin{enumerate}
            %     \item $IR(x)$ is an ancestor idle ribbon of $IR(y)$; or
            %     \item $x,y$ are vertices of the same idle ribbon and $x$ is in front of $y$.
            % \end{enumerate}
    \end{enumerate}
\end{definition}

Then we define the reassembly sequence using the rearrangement sequence and recycle sequence:
\begin{definition}
    The re-assembly sequence $t_r[k]$ for $ 1\leq k \leq N$ is generated recursively such that each element $t_r[k]$ satisfies:
    \begin{itemize}
        \item[-] if the lattice position $s_r[k]$ is among the first $c(R(s_r[k]))$ lattice points of the ribbon $R(s_r[k])$, $t_r[k]$ is the first lattice point in the recycle sequence that has not been include in ${t_r}_{[1:k-1]}$;
        \item[-] if $s_r[k]$ is not among the first $c(R(s_r[k]))$ lattice points of the ribbon $R(s_r[k])$, $t_r[k]$ is the first lattice point in the rearrangement sequence that has not been include in ${t_r}_{[1:k-1]}$.
    \end{itemize}
\end{definition}

% Then we define the reassembly sequence using the rearrangement sequence and recycle sequence:
% \begin{definition}
%     The \textbf{reassembly sequence} $t_r$ is defined such that:
%     \begin{enumerate}
%         \item if $ID(s_r[1]) < c(R(s_r[1]))$, $t_r[1]$ is the first lattice point in the recycle sequence;
%         \item if $ID(s_r[1]) \geq c(R(s_r[1]))$, $t_r[1]$ is the first lattice point in the rearrangement sequence.
%     \end{enumerate}

%     and for $k >1$,
%     \begin{enumerate}
        
%         \item if $ID(s_r[k]) < c(R(s_r[k]))$, $t_r[k]$ is the next lattice point in the recycle sequence;
%         \item if $ID(s_r[k]) \geq c(R(s_r[k]))$, $t_r[k]$ is the next lattice point in the rearrangement sequence.
%     \end{enumerate} 
% \end{definition}

Loosely speaking, the robot at the front of the outermost ribbon becomes active first. If the robot is a recycled robot, it moves out of the shape $S$ and stops in $H_2$ layer by layer; otherwise, the robot is a rearranged robot and it moves to the front of the ribbon and stops at the smallest unoccupied vertex. This property is captured by the following proposition:

\begin{proposition}
    \label{prop-rearrangedOccupyAllPointsOfRibbon}
    At sub-epoch $T = n^-$ with $n=N+k$, suppose $s_r[k] \in V(R)$,
    \begin{itemize}
        \item [-] if the active robot $u_n$ is a recycle robot, then $t_r[k] \in H_2$;
        \item [-] if the active robot $u_n$ is a rearranged robot, then $t_r[k]\in V(R)$ and $t_r[k] \in S \setminus D$, and all vertices of ribbon $R$ between $t_r[k]$ (inclusive) and $s_r[k]$ (exclusive) are unoccupied, that is
        \[ (\forall x \in \{ x\in V(R) | s_r[k] < x \leq t_r[k]\} ) x \notin {t_r}_{[1:k-1]}\]
    
    \end{itemize}
    %  An active recycled robot . An active rearranged robot , and all vertices on the ribbon between its starting position (exclusive) and stopping position (inclusive) are unoccupied. 
\end{proposition}
\begin{proof}
    Since the complete tree order is a total order of ribbons, the robots in $S$ become active in a ribbon-by-ribbon manner. Let $E$ be the sequence of ribbons induced by the complete tree order such that $E[1]$ is the largest ribbon in the complete tree order and $E[n+1]$ is the next largest ribbon which is smaller than $E[n]$. In other words, the ribbons in sequence $E$ are arranged from the outermost ribbon to the ribbon closest to the seed robots according to the complete tree order.
    
    To prove the proposition, we shall perform induction on the sequence of ribbon $E$. Let $P(i)$ be "for all active robot $u_n$ with starting position $s_r[k] \in V(E[i])$,
    \begin{enumerate}
        \item if $ID(s_r[k]) < c(E[i])$, then $t_r[k] \in H_2$; 
        \item if $ID(s_r[k]) \geq c(E[i])$, then $t_r[k] \in V(R)$ and $t_r[k] \in S \setminus D$. In additionally, all lattice points on ribbon $R$ between $s_r[k]$ (exclusive) and $t_r[k]$ (inclusive) are not in ${t_r}_{[1:k-1]}$."
    \end{enumerate}

    \begin{itemize}
    
        \item [-] Base case: for robots belonging to $E[1]$, the first $c(E[1])$ robots are the recycle robots and by the recycle sequence, they stop and occupy the first $c(E[1])$ lattice points of the recycle sequence. P(1) follows. The remaining robots of $E[1]$ are rearranged robots and they stop and occupy the first $len(E[1]) - c(E[1])$ lattice points of the rearrangement sequence. Since the first $len(E[1]) - c(E[1])$ lattice points of the rearrangement sequence are exactly the vertex of $E[1]$ in $S \setminus D$ with the same ordering, P(1) follows.
        
        \item [-] Inductive step: suppose $P(1),P(2),...,P(n)$ is true, we want to show that $P(n+1)$ is true. First notice that $s_r[\sum_{m=1}^{n}len(E[m])+1]$ to $s_r[\sum_{m=1}^{n+1}len(E[m])]$ are vertices of ribbon $E(n+1)$. For $m = 1,2,...,n$, there are $len(E[m]) - c(E[m])$ rearranged robots stops and occupies $len(E[m]) - c(E[m])$ lattice points in $S\setminus D$ of ribbon $E[m]$. 
            \begin{itemize}
                \item [-] if $c(E[n+1]) = 0$, there are no recycle robots and the rearranged robots at $s_r[\sum_{m=1}^{n}len(E[m])+1]$ to $s_r[\sum_{m=1}^{n}len(E[m])+len(E[n+1])]$ stops at the $(\sum_{m=1}^{n}(len(E[m])- c(E[m]))+1)$th to the $(\sum_{m=1}^{n}(len(E[m]) - c(E[m]))+len(E[n+1]))$th lattice points of the rearrangement sequence, and the lattice points are exactly the vertices of ribbon $E[n+1]$ with the same ribbon order. As a result, the rearranged robots shall start and stop at the same lattice point and $P(n+1)$ follows. 
                
                \item [-] if $c(E[n+1]) \neq 0$,the recycle robots at $s_r[\sum_{m=1}^{n}len(E[m])+1]$ to $s_r[\sum_{m=1}^{n}len(E[m])+c(E[n+1])]$ stops at the $(\sum_{m=1}^{n}c(E[m])+1)$th to $(\sum_{m=1}^{n+1}c(E[m]))$th lattice points of the recycle sequence. $P(n+1).1$ follows. The rearranged robots at $s_r[\sum_{m=1}^{n}len(E[m])+c(E[n+1])+1]$ to $s_r[\sum_{m=1}^{n+1}len(E[m])]$ stops at the $(\sum_{m=1}^{n}(len(E[m]) - c(E[m]))+1)$th to the $(\sum_{m=1}^{n+1}(len(E[m]) - c(E[m])))$th lattice points of the rearrangement sequence, and the lattice points are exactly the vertices in $S \setminus D$ of ribbon $E[n+1]$ with the same ordering. $P(n+1)$ follows.
            \end{itemize}
        
    \end{itemize}

\end{proof}

\begin{proposition}
    \label{prop-boundaryPointsAreOccupied}
    During the subtractive stage, by following the re-activation sequence and the re-assembly sequence, all boundary points of $S$ should be occupied before any leading robot of some ribbon starts to move, that is 
    \[ (\forall N< T \leq 2N) (\exists R \in \hat{S}) (p_{{(T-1)}^+}(u_{T-N}) = F(R)) \Rightarrow (\forall \text{ boundary point } x, x \notin {s_r}_{[1:T-N-1]} \lor x \in {t_r}_{[1:T-N-1]}) \] 
\end{proposition}
\begin{proof}

     Let sequence $I[k]$ be an index sequence of selects the first vertex of each ribbon in the reactivation sequence $s_r$. Notice that $I[k]$ follows the complete tree order in the sense that each element appears earlier in the sequence $I[k]$ is smaller with respect to the complete tree order. Let $Q$ denote the set containing all boundary points of $S$.
     
     We shall prove the proposition by induction. Let $P(n)$ denotes
     \[ (\forall x \in Q) (x \notin {s_r}_{[1:I[n]-1]}) \lor (x \in {t_r}_{[1:I[n]-1]}) \]

     \begin{itemize}
        \item [-] Base case: $P(1)$ is vacuously true since all lattice points in $S$ are occupied.
        \item [-] Inductive step: Suppose $P(n)$ is true, we want to show that $P(n+1)$ is true. In other words, we want to show that all lattice points in $Q$ are occupied before the robot at lattice point $s_r[I[n]]$ starts to move.
        
        By the induction hypothesis, all points in $Q$ are occupied before the robot at lattice point $s_r[I[n]]$ starts to move. In the reactivation sequence, all lattice points after lattice point $s_r[I[n]]$ and before lattice point $s_r[I[n+1]]$ belong to ribbon $R(s_r[I[n]])$. The rearranged robots belonging to ribbon $R(s_r[I[n]])$ stop at lattice points of the same ribbon by proposition \ref{prop-rearrangedOccupyAllPointsOfRibbon}. Since there are exactly the same number of rearranged robots and the number of lattice points in $S \setminus D$ of ribbon $R(s_r[I[n]])$, all lattice points in $S \setminus D$ of the ribbon $R(s_r[I[n]])$ are occupied before the robot at $s_r[I[n+1]]$ starts to move. Since $Q$ does not contain any lattice point in $D$ by the proper ribbonization of $S$, $P(n+1)$ follows.

     \end{itemize}

\end{proof}

\subsection{Path for the Subtractive Stage}
\label{sect-pathPlanningForSubtractive}

At epoch $T= n = N+k$ of the subtractive stage, the robot at lattice point $s_r[k]$ starts to move. By following the reactivation sequence and the reassembly sequence, the active robot $u_{u}$ starts its movement from lattice point $s_r[k]$ and is expected to stop at lattice point $t_r[k]$. In this section, we investigate the path for each active robot $u_n$ of the subtractive stage. There are two types of path: for the first $c(R(s_r[k]))$ robots of ribbon $R(s_r[k])$ (i.e., the recycled robots), they move out the shape $S$ and reassemble in $H_2$; for the remaining robots of ribbon $R$ (i.e., the rearranged robots), they redistribute on the ribbon $R(s_r[k])$ and occupy all lattice points of the ribbon in $S \setminus D$. In section \ref{subsect-pathRecycle}, we investigate the paths for the recycled robots and prove the unobstructiveness of the paths. In section \ref{subsect-pathRearrange}, we investigate the paths for the rearranged robots and prove the unobstructiveness.

% \begin{figure}[H]
%     \centering
%     \includegraphics[scale=0.3]{figures/proofHoles/proofWithHoles-stage2initial.png}
%     \caption{Initial configuration of subtractive stage}
% \end{figure}

\subsubsection{Motion of Recycled Robots}
\label{subsect-pathRecycle}

For a recycle robot $u_n$ with $n = N + k$, suppose $s_r[k]$ is a vertex of ribbon $R$ and $R$ has a parent ribbon $R^{[p]}$. The active robot first rotates around the parent ribbon $R^{[p]}$ while keeping a distance $d$ from the nearest robot on $R^{[p]}$ until it reaches the first vertex of $R$. From the first vertex $x$, the robot continues its rotation and reaches the merging point $m(R)$ of the ribbon $R$. Finally, $u_n$ moves clockwise along the perimeter of the MRS and stops in $H_2$. 

% \begin{figure}[H]
%     \centering
%     \includegraphics[scale=0.3]{figures/proofHoles/proofNonCanonicalWithHoles-shuffle-recycle2.png}
%     \caption{Illustration of the recycle path of recycle robot $u_n$}
    
% \end{figure}

\begin{figure}[H]
    \centering
    \includegraphics[scale=0.4]{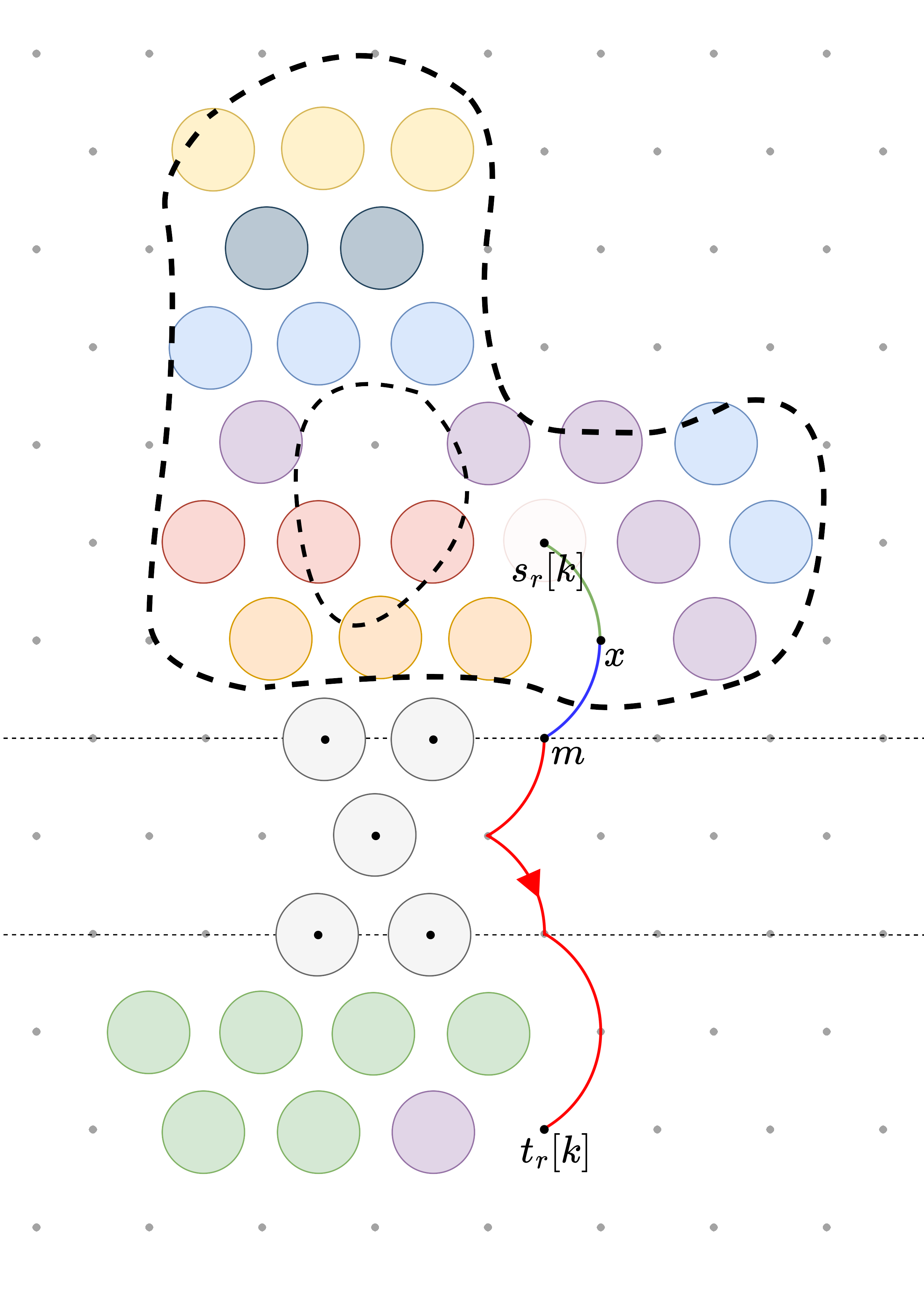}
    \caption{Illustration of the recycle path of recycle robot $u_n$: the interior path is denoted with green; the bridging path is denoted with blue; the exterior path is denoted with red}
\end{figure}

\begin{definition}
    At epoch $T=n$, for a recycled robot $u_n$ with $n = N + k$, suppose $u_n$ belongs to ribbon $R$ before activation, define the \textbf{recycle path} of $u_n$ that consists of three path segments as follows:
    \begin{enumerate}
        \item \textbf{the interior path}: starts from its starting lattice point $s_r[k]$, and rotates clockwise around the nearest robot of the parent ribbon $R^{[p]}$ with distance $d$ until it reaches the first vertex $x$ of ribbon $R$;
        \item \textbf{the bridging path}: starts from the first vertex of ribbon $R$ and stops at the merging point $m(R)$ of ribbon $R$ by rotating clockwise around the nearest robot of the parent ribbon $R^{[p]}$ with distance $d$;
        \item \textbf{the exterior path}: starts from the merging point $m(R)$ and stops at $t_r[k]$ by rotating clockwise around the perimeter of the MRS while keeping distance $d$ to the nearest boundary robot.
    \end{enumerate}

\end{definition}

Next, we shall show that the recycle path is unobstructed by investigating the three path segments of the recycle path individually. 

% We call the path segment from $s_r[n]$ to merging point $m$ as the \textbf{interior path}, and call the path segment from merging point $m$ to $t_r[n]$ as the \textbf{exterior path}. We shall show that both the interior path and the exterior path is unobstructed.

\begin{proposition}
    \label{prop-interiorPath}
    For any recycled robot $u_n$ with $n = N + k$, its interior path is unobstructed.
\end{proposition}
\begin{proof}
    Notice that all lattice points in $R$ before $s_r[k]$ are unoccupied since all robots that occupied such lattice points have moved out of $S$ and stopped in $H_2$ by proposition \ref{prop-rearrangedOccupyAllPointsOfRibbon}. It is left to show that the curve is unobstructed between each adjacent lattice point on the path. Suppose two adjacent lattice points $x,y$ are vertices of ribbon $R$, the curve between $x$ and $y$ with distance $d$ from robot $v$ (belonging to the parent ribbon of $R$) is shown below.

    \begin{figure}[H]
        \centering
        \includegraphics[scale=0.6]{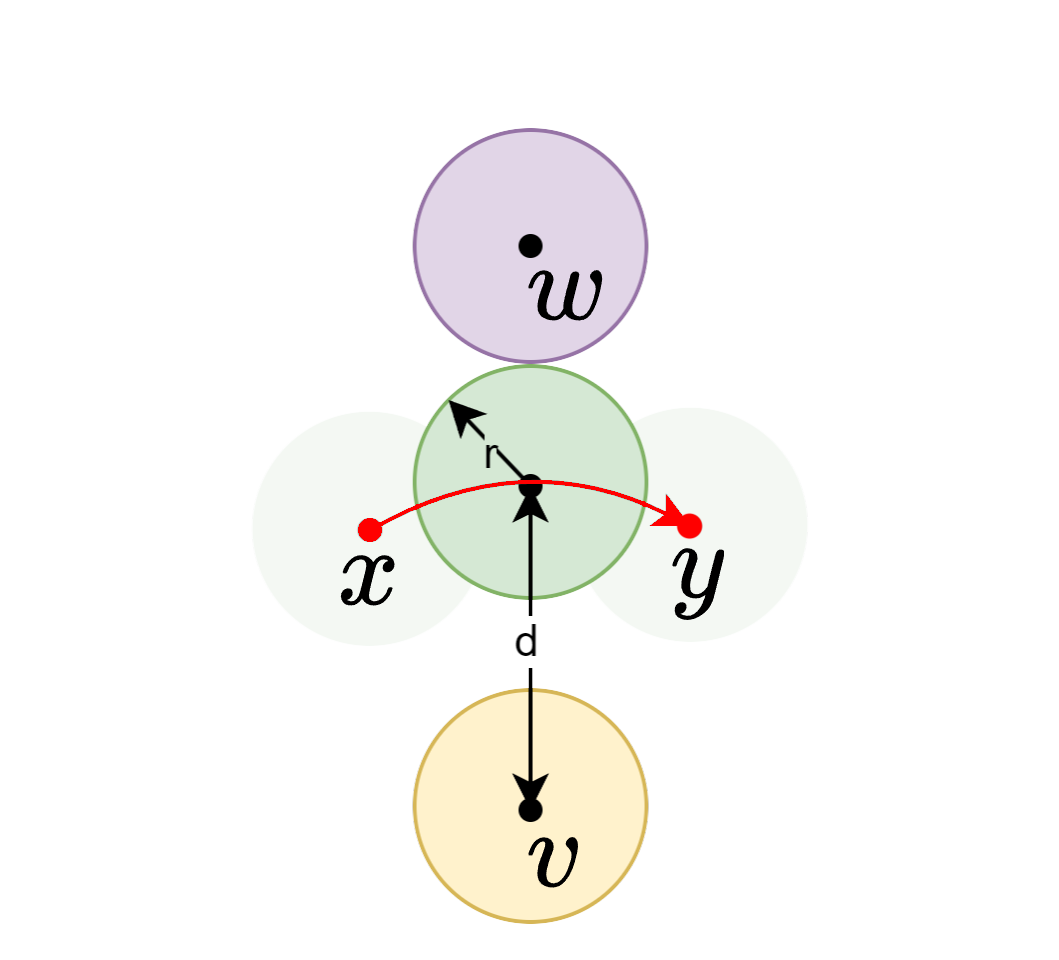}
        \caption{Path (in red) between lattice point $x$ and $y$}
    \end{figure}

    Notice that any robot at a distance larger than or equal to $2r$ from all points of the curve $xy$ could not obstruct the path. Since $d = \frac{2r}{\sqrt{3}-1}$, the nearest robot to the curve is $2*d*\frac{\sqrt{3}}{2} - d = 2r$ from the curve. The proposition follows.
\end{proof}

\begin{proposition}
    \label{prop-bridgingPath}
    For any recycle robot $u_n$ with $n = N + k$, its bridging path is unobstructed.
\end{proposition}
\begin{proof}
    % Suppose $u_n$ belongs to ribbon $R$ before activation, and suppose the first vertex of ribbon $R$ is $x$ and the merging point of ribbon $R$ is $m$. The neighborhood of $x$ is shown as follows:
    % \begin{figure}[H]
    %     \centering
    %     \includegraphics[scale=0.5]{figures/proofHoles/proofCanonical-mergingPoint.png}
    %     \caption{Two types of neighborhood of $x$: points $y,z$ are vertices of the parent ribbon of $R$}
    % \end{figure}
    % In both types of neighborhood, the arc from $x$ to $m$ centered at $y$ is the bridging path of $u_n$. By the similar arguments as in the proof of proposition \ref{prop-interiorPath}, the bridging path is unobstructed.

    Suppose $u_n$ belongs to ribbon $R$ before moving, and suppose the first vertex of ribbon $R$ is $x$ and the merging point of ribbon $R$ is $m(R)$. The arc from $x$ to $m(R)$ can be viewed as an extension of the interior path. Since $m$ is unoccupied, the bridging path is unobstructed by the same arguments as in the proof of proposition \ref{prop-interiorPath}.

\end{proof}

\bigskip
To prove the correctness of the exterior path, notice that the recycled robots rotate around the perimeter of the MRS and stop in the idle half-plane layer by layer. This procedure could be viewed as the recycle robots repeating the additive stage and assembling in the half-plane $H_2$. To reuse the results in section \ref{sect-motionPlanning}, we redefine the edge-following set and corresponding edge-following path for the subtractive stage. By the similar arguments as in section \ref{sect-motionPlanning}, the redefined edge-following path is homeomorphic to a circle, and the exterior path is exactly a path segment of the redefined edge-following path. We conclude that the exterior path is unobstructed.

% For we define virtual robots as a place holder, so that all lattice points in $S \cup D$ are occupied when we define the boundary of the swarm:
To reuse the results from the additive stage, we define virtual robots as placeholders so that there is no unoccupied lattice point within $S$:

\begin{definition}
    A lattice point $x \in S$ is occupied by a \textbf{virtual robot} iff $x$ is not occupied by any non-moving robot. We say a lattice point is \textbf{virtually-occupied} if it is occupied by a virtual robot.
\end{definition}

\begin{definition}
    \leavevmode
    \makeatletter
    \@nobreaktrue
    \makeatother
    \begin{enumerate}
        \item a \textbf{virtual boundary robot} is a virtual robot occupies a boundary point of $S$; and
        \item the \textbf{extended set of boundary robot} $\overline{B_{n^-}}$ and $\overline{B_{n^+}}$ is the collection of all boundary robots and virtual boundary robots at sub-epoch $T = n^-$ and $T = n^+$, respectively.
    \end{enumerate}
    % At system step $T = n > EOA+k$ with corresponding moving robot $u_n$,   The \textbf{extended set of boundary points} is the collection of center points of all boundary robots and pseudo-boundary robots.
\end{definition}

\begin{definition}
    At sub-epoch $T = n^-$ with $n > N+k$, a lattice point $x \in L$ is an \textbf{extended edge-following point} iff:
    \begin{enumerate}
        \item $x$ is not occupied or virtually-occupied; and
        % \item $\exists y \in L$, such that $u$ is a seed robot or $u$ is robot in the idle or stopped state and $p_t(u)$ is adjacent to $x$; or
        \item $\exists y \in L$ such that $y$ is occupied or pseudo-occupied and $y$ is adjacent to $x$.
        % \item $\exists \text{ lattice point } a \in L$, such that $p_t(u)$ is adjacent to $a$. 
    \end{enumerate}
    The \textbf{extended edge-following set}, denoted as $\overline{EFS_{n^-}}$, is the set that contains all the extended edge-following points. The same definition applies to the extended edge-following points at $T = n^+$ and the corresponding extended edge-following set $\overline{EFS_{n^+}}$.
\end{definition}

Notice that the extended edge-following points are the same as edge-following points if we consider all lattice points in $S$ occupied.

% \begin{figure}[H]
%     \centering
%     \captionsetup{justification=centering}
%     \includegraphics[scale=0.35]{figures/proofHoles/proofNonCanonicalWithHoles-extended EFS1.png}
%     \caption{Illustration of the swarm at system step $T = n$ with recycled robot $u_n$. \\ The black dots indicate the extended set of boundary points, and the red dots indicate \\the extended edge-following set $\overline{EFS_n}$. Notice that $u_n$ is the moving robot and\\ there are two virtual robots at lattice points $v$ and $w$.}
% \end{figure}

\begin{proposition}
    \label{prop-startEndBothInEFS2}
    At epoch $T = n = N + k$ and suppose $s_r[k]$ belongs to ribbon $R$, then $m(R) \in \overline{EFS_{n^-}}$ and $t_r[k] \in \overline{EFS_{n^-}}$. 
\end{proposition}
\begin{proof}
    Since the merging point $m(R) \notin S$ and $m$ is adjacent to the first vertex of $R$, $m(R)$ belongs to the extended edge-following set $\overline{EFS_{n^-}}$. By the reassembly sequence, $t_r[k] \in H_2$, and $t_r[k]$ is adjacent to a lattice point on the parent ribbon of $R(t_r[k])$. As a result, $t_r[n]$ is in $\overline{EFS_{n^-}}$. 
\end{proof}

\begin{proposition}
    \label{prop-properBoundaryRobots2}
    Given a proper ribbonization $\hat{S}$ with the previously defined initial configuration, activation sequence, assembly sequence, reactivation sequence and reassembly sequence, the extended set of boundary robots is proper for $N< n \leq 2N$.
\end{proposition}
\begin{proof}
    We shall prove this by induction. Let $P(n)$ be "the extended set of boundary robots $\overline{B_n^-}$ is proper". Let $Q(n)$ be "the extended set of boundary robots $\overline{B_n^+}$ is proper".

    First, notice that the culminating set of boundary robots $\widehat{B}$ is proper. Since the extended set of boundary robots agrees with the culminating set of boundary robots in $H_1$, the base case $P(N+1)$ is true.
    
    Now suppose $P(n)$ is true, we want to show that after $u_{n}$ stops, the extended set of boundary robots is proper. By the same arguments in the proof of proposition \ref{prop-properBoundary}, $Q(n)$ follows. 
    
    Suppose $Q(n)$ is true, notice that after $u_{n+1}$ becomes active, the extended set of boundary robots remains unchanged since $s_r[n+1-N] \in H_1$. As a result, $P(n+1)$ is true.
    
    Since $P(N+1)$ is true, by induction, $P(n)$ and $Q(n)$ are true for all $N<n \leq 2*N$.
\end{proof}

\begin{definition}
    \label{def-EFP2}
    At sub-epoch $T = n^-$ with $n = N+k$, define the \textbf{extended edge-following path} $\overline{EFP_{n^-}}$ as the path consisting of all points that are at a distance $d$ from the nearest occupied or virtually-occupied lattice point. Same applies for $\overline{EFP_{n^+}}$ at $T = n^+$.
\end{definition}

By the same arguments as in the proof of proposition \ref{prop-EFPisCircle}, we obtain the following proposition:

\begin{proposition}
    \label{prop-EFPisCircle2}
    The extended edge-following path is always homeomorphic to a circle, and the extended edge-following path goes through all the points of the extended edge-following set.
\end{proposition}

\begin{definition}
    \label{def-exteriorEdgeFollowingPath}
    At epoch $T= n = N + k$, suppose $s_r[k]$ belongs to ribbon $R$ before activation and the merging point of $R$ is $m$, define the \textbf{exterior edge-following path} for $u_n$ as the path segment of the extended edge-following path $\overline{EFP_{n^-}}$ that starts from $m$  and ends at $t_r[k]$, directed clockwise.
\end{definition}

Notice that the above definition is well-defined since the merging point $m$ and the stopping point $t_r[n]$ are on the edge-following path $EFP_{n^-}$ (by proposition \ref{prop-startEndBothInEFS2}) and the extended edge-following path is homeomorphic to a circle (by proposition \ref{prop-EFPisCircle2}).

\begin{proposition}
    \label{prop-exteriorEFPNoVirtualRobot}
    At epoch $T=n = N + k$, suppose $s_r[k]$ belongs to ribbon $R$. Then, except for the merging point of $R$, any point on the exterior edge-following path of $u_n$ is not contributed by a virtual robot.
\end{proposition}
\begin{proof}
    Recall $p_3$ is the lattice point on the right of the origin $O$. 
    Observe that the exterior edge-following path goes through lattice point $p_3$. Since virtual robots (if they exist) are in half-plane $H_1$, we only need to show that the path segment of the exterior edge-following path from $m(R)$ to $p_3$ satisfies the proposition. We denote the path segment from $m(R)$ to $p_3$ as $P$. Notice that $P$ is a path segment of the culminating affiliated path $\widehat{AP}$.

    By proposition \ref{prop-boundaryPointsAreOccupied}, no point on $P$ is contributed by any virtual robot at the time when the first robot of ribbon $R$ starts to move. Let $T = l$ denote the epoch when the first robot of ribbon $R$ starts to move. By the reactivation sequence, all active robots at epoch $l < T < N+k$ are from ribbon $R$. As a result, a point on $P$ is either contributed by a boundary robot or a virtual robot occupying some vertex of $R$. By proposition \ref{prop-mergingPointInFront}, $m(R)$ is in front of all points of $\widehat{AP}$ that are contributed by robots/virtual robots of $R$. As a result, no point on $P$ (except for the merging point itself) is contributed by a virtual robot. The proposition follows.

    % \begin{figure}[H]
    %     \centering
    %     \includegraphics[scale=0.3]{figures/proofHoles/proofNonCanonicalWithHoles-recyclePathAP.png}
    % \end{figure}
\end{proof}

\begin{proposition}
    \label{prop-exteriorPath}
    For the recycle robot $u_n$ with $n = N + k$, its exterior path is well-defined and unobstructed.
\end{proposition}
\begin{proof}
    First, notice that the merging point $m$ and the stopping point $t_r[k]$ are on the extended edge-following path $\overline{EFP_{n^-}}$ (by proposition \ref{prop-startEndBothInEFS2}). By proposition \ref{prop-exteriorEFPNoVirtualRobot}, the exterior path of $u_n$ (except for the merging point) is not contributed by any virtual robot. As a result, the exterior path for $u_n$ is exactly the exterior edge-following path which is the path segment of the extended edge-following path $\overline{EFP_{n^-}}$ that starts from $m$ oriented clockwise and ends at $t_r[k]$. By proposition \ref{prop-EFPisCircle2}, the exterior path starting from $m$ and ending at $t_r[k]$ is well-defined and unobstructed.
\end{proof}

\subsubsection{Motion of Rearranged Robots}
\label{subsect-pathRearrange}

At sub-epoch $T= n^-$ with $n = N+k$, the rearranged robot $u_n$ moves from $s_r[k]$ and stops at $t_r[k]$. 

\begin{figure}[H]
    \centering
    \captionsetup{justification=centering}
    \includegraphics[scale=0.4]{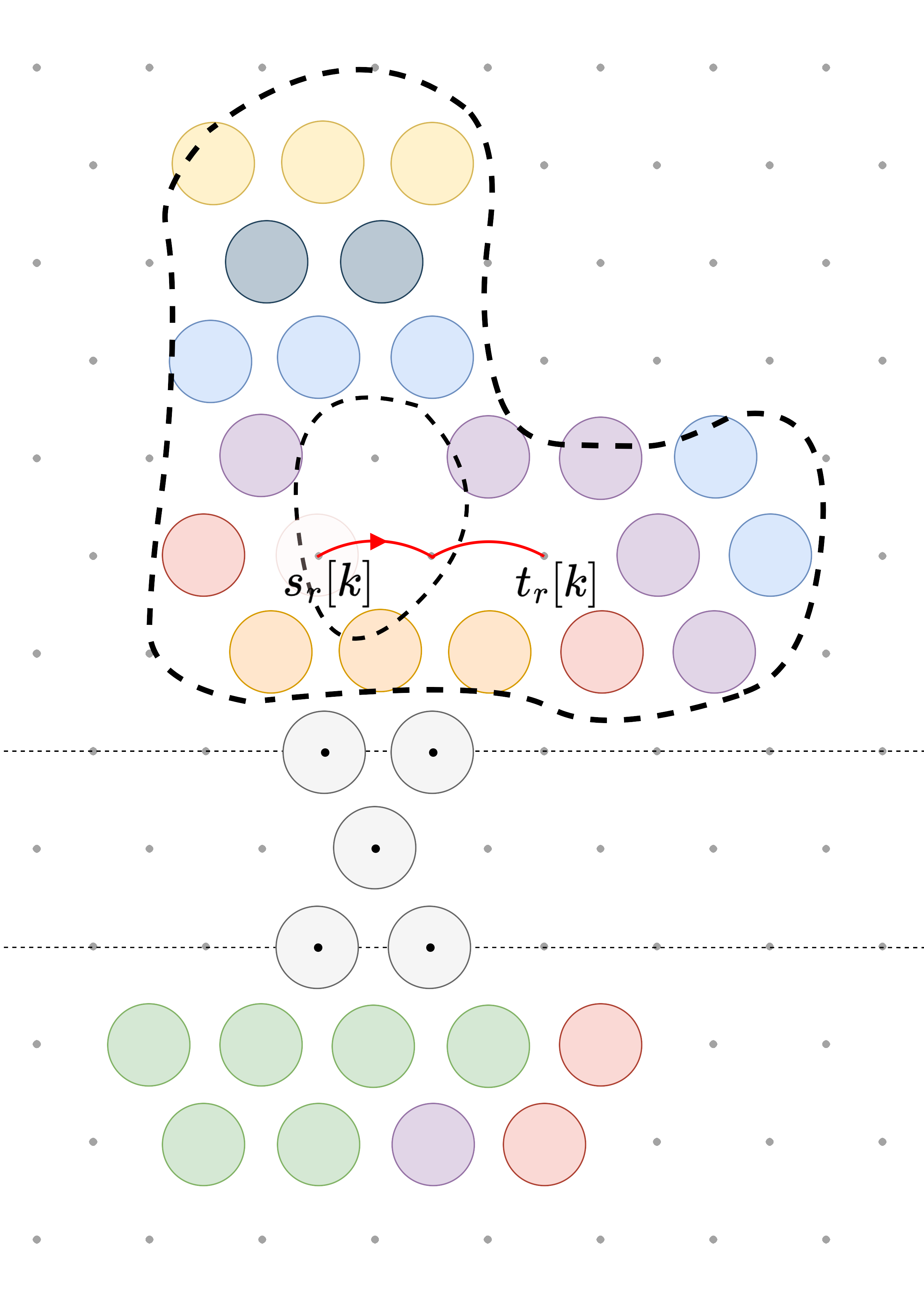}
    \caption{Illustration of the start point $s_r[k]$ and target point $t_r[k]$ for rearranged robot $u_{EOA+k}$}
\end{figure}

\begin{definition}
    For any rearranged robot $u_n$ with $n = N + k$, suppose $u_n$ belongs to ribbon $R$ before activation, define the \textbf{rearrangement path} of $u_n$ such as:
    % \begin{figure}[H]
    %     \centering
    %     \captionsetup{justification=centering}
    %     \includegraphics[scale=0.4]{figures/proofHoles/proofNonCanonicalWithHoles-rearrangedRobotPath.png}
    % \end{figure}    
    \begin{enumerate}
        \item the path starts from lattice point $s_r[k]$ and ends at lattice point $t_r[k]$ of ribbon $R$; and
        \item consecutively visiting vertices of $R$ between $s_r[k]$ and $t_r[k]$ while maintains distance $d$ from the nearest robot of the parent ribbon of $R$;
    \end{enumerate}
\end{definition}

\begin{proposition}
    \label{prop-rearrange}
    For any rearranged robot $u_n$, its rearrangement path is unobstructed. 
\end{proposition}
\begin{proof}
    Let $k= T - N$, then by proposition \ref{prop-rearrangedOccupyAllPointsOfRibbon}, $s_r[k]$ and $t_r[k]$ are vertices of the same ribbon $R(s_r[k])$, and all lattice points on ribbon $R(s_r[k])$ between $s_r[k]$ (exclusive) and $t_r[k]$ (inclusive) are unoccupied. By the same arguments in the proof of proposition \ref{prop-interiorPath}, the result follows. 

\end{proof}

\begin{lemma}
    \label{prop-subtractiveStageCorrectness}
    For any active robot in the subtractive stage, its path is unobstructed.
\end{lemma}
\begin{proof}
    If the active robot is a recycled robot, the corresponding recycle path is unobstructed since all three path segments are unobstructed by proposition \ref{prop-interiorPath}, proposition \ref{prop-bridgingPath}, and proposition \ref{prop-exteriorPath}. If the active robot is a rearranged robot, the corresponding rearranged path is unobstructed by proposition \ref{prop-rearrange}.
\end{proof}

\subsection{Localization for the Subtractive Stage}
\label{sect-localization2}

During the subtractive stage of the algorithm, each active robot $u_n$ with $n= N+k$ shall move from lattice point $s_r[k]$ to lattice point $t_r[k]$. We use $R$ to denote the ribbon that $s_r[k]$ belongs to, and recall that we use $R^{[p]}$ to denote the parent ribbon of $R$. In this section, we prove the localizability of the robots of the subtractive stage. Notice that all non-moving robots are localizable at the end of the additive stage by proposition \ref{prop-IdleLocalizable} and lemma \ref{thm-localizable}. We are left to show the localizability of every moving robot $u_n$ on its path from $s_r[k]$ to $t_r[k]$.

\subsubsection{Localization for Rearranged Robots}
If $u_n$ is a rearranged robot, then it starts from $s_r[k]$ and rotates clockwise around the nearest robot of ribbon $R^{[p]}$ with a distance $d$ until it stops at $t_r[k]$. Notice that $u_n$ only needs to maintain a certain distance from the nearest robot of $R^{[p]}$. As a result, $u_n$ only needs to know its coordinates when it is at $s_r[n]$ and $t_r[n]$.
 
\begin{lemma}
\label{prop-rearrangeLocalizable}
    At epoch $T = n = N + k$, if $u_n$ is a rearranged robot and $s_r[k]$ is a vertex of ribbon $R$, then $u_n$ is localizable at $s_r[k]$ and $t_r[k]$.
\end{lemma}
\begin{proof}
    Notice that $s_r[k]$ is adjacent to a seed robot or a robot of $R^{[p]}$. By the definition of the reactivation sequence, the parent ribbon $R^{[p]}$ is filled at $T=n$. As a result, $u_n$ is localizable at $s_r[k]$ by \ref{prop-3robotsForLocalization}. The same applies to lattice point $t_r[k]$.
\end{proof}

\subsubsection{Localization for Recycled robots}
If the active robot $u_n$ is a recycled robot, it starts from $s_r[k]$ and rotates clockwise around the nearest robot of ribbon $R^{[p]}$ at a distance $d$ until it reaches the first vertex of $R$. Then, robot $u_n$ continuously rotates until it reaches the merging point $m(R)$. Finally, it rotates clockwise around the nearest boundary robot until it stops at $t_r[k]$. In other words, $u_n$ first follows the interior path until it reaches the first vertex $x$, then it follows the bridging path until it reaches the merging point $m(R^{[p]})$. Finally, it follows the exterior path until reaching the stopping point $t_r[k]$. Notice that $u_n$ only needs to maintain a certain distance from the nearest robot\footnote{In the case of interior path and bridging path, $u_n$ maintains a certain distance from the nearest robot of a particular ribbon.} within each path segment. As a result, $u_n$ only needs to know its coordinates when it is at the endpoints of the three path segments, i.e., at $s_r[k]$, $x$, $m(R)$, and $t_r[k]$. 

% Although the edge-following movement only requires $u_n$ to maintain certain distance from the nearest robot \footnote{Or the nearest robot from a certain ribbon}, the moving robot $u_n$ need to know
% \begin{itemize}
%     \item its coordinates at $t_r[n]$, if $u_n$ is a rearranged robot; or
%     \item its coordinates at $t_r[n]$ and the merging point of $R$, if $u_n$ is a recycled robot.
% \end{itemize}

\begin{lemma}
\label{prop-recycledRobotLocalizable}
    At epoch $T = n = N + k$ and $u_n$ being a recycled robot, suppose $s_r[k]$ belongs to ribbon $R$ and the first vertex of $R$ is $x$ and the merging point of $R$ is $m(R)$, then $u_n$ is localizable at $s_r[k]$, $x$, $m(R)$ and $t_r[k]$.
\end{lemma}
\begin{proof}
    Notice that $s_r[k]$ is adjacent to a seed robot or a robot belonging to $R^{[p]}$. By the definition of the reactivation sequence, the parent ribbon $R^{[p]}$ is filled. As a result, $u_n$ is localizable at $s_r[k]$ by \ref{prop-3robotsForLocalization}. Similarly, $u_n$ is localizable at $x$ and $m(R)$.
    
    By the definition of recycle sequence, $t_r[k]$ is adjacent to a seed robot or a robot of a filled idle ribbon. By the same reasoning as in the proof of proposition \ref{prop-3robotsForLocalization}, robot $u_n$ is localizable at $t_r[k]$.
\end{proof}

% \begin{theorem}
%     \label{thm-subtractiveStageCorrect}
%     For epoch $N+1 \leq T \leq 2N$, robot $u_{N+k}$ executing the edge-following movement stops at $t_r[k]$.
% \end{theorem}
% \begin{proof}
%     By lemma \ref{prop-subtractiveStageCorrectness}, lemma \ref{prop-rearrangeLocalizable} and lemma \ref{prop-recycledRobotLocalizable}, the theorem follows.
% \end{proof}

\subsection{Conclusion}
\label{sect-finalShape}

The activation sequence $s$, the reactivation sequence $s_r$, the assembly sequence $t$ and the re-assembly sequence $t_r$ uniquely defined the movement sequence of tuples $(s_1,t_1,1), \ldots, (s_N,t_N,N), ({s_r}_1,{t_r}_1, N+1), \ldots, ({s_r}_N,{t_r}_N, 2N)$. In this section, we use the results from previous sections to show that the movement sequences, together with the add-subtractive algorithm, satisfy the problem statement:

    \begin{theorem}
        \label{thm-final}
        Given sufficient robots (with an initial set of robots $B$ whose position coordinates are known in advance) and user-defined shapes $S \setminus D$. The add-subtract algorithm with the movement sequence of tuples $(a_1,b_1,1), \ldots, (a_{2N}, b_{2N}, 2N) := (s_1,t_1,1), \ldots, (s_N,t_N,N), ({s_r}_1,{t_r}_1, N+1), \ldots, ({s_r}_N,{t_r}_N, 2N)$ satisfies:
        \begin{enumerate}
            % \item before the movement of any robot, the robots centered at points in $B$ can locate their position.
            \item the final collection of robots forms the desired shape in the workspace $H_1$ in the sense that $\{\{b_1,\ldots, b_{2N}\} \circleddash \{a_1,\ldots, a_{2N}\}\} \sqcap H_1  = S \setminus D$.
            \item For each prefix subsequence $(a_1, b_1, 1), \ldots, (a_{k}, b_{k},k)$, $1 \leq k \leq 2N$, let $A_{k} = \{a_1, \ldots, a_{k-1}\}$, $B_{k} = \{b_1, \ldots, b_{k-1}\}$, and $\overline{B_{k}} = B_{k}$ since the add-subtract algorithm enforce at most one moving robot at any time. Then the starting position of the $k$th robot $a_k \in I \oplus \overline{B_{k}} \circleddash A_{k}$ and the robot can only relay on robots centered at $I \oplus \overline{B_{k}} \circleddash A_{k} \circleddash \{a_k\}$ to locate its position before stopping at $b_k$. 
        \end{enumerate}
    \end{theorem}
    \begin{proof}
        
        First, notice the following:
        \begin{enumerate}
            \item for the additive stage, $\{a_1,\ldots, a_N\} = \{s_1,\ldots,s_N\} \in I$ and $\{b_1, \ldots, b_N\}= \{t_1,\ldots,t_N\} = S$;
            \item for the subtractive stage, $\{a_{N+1}, \ldots, a_{2N}\}=\{{s_r}_1,\ldots,{s_r}_N\} = S$ and $\{b_{N+1}, \ldots, b_{2N}\} \sqcap H_1 = \{{t_r}_1,\ldots,{t_r}_N\} \sqcap H_1 = S\setminus D$.
        \end{enumerate}
        As a result, $\{\{b_1,\ldots, b_{2N}\} \circleddash \{a_1,\ldots, a_{2N}\}\} \sqcap H_1  = \{S \oplus \{b_{N+1}, \ldots, b_{2N}\} \circleddash I \circleddash S\} \sqcap H_1 = S \setminus D$. This shows theorem \ref{thm-final}.1.
        For $1 \leq k \leq N$, $a_k \in I \circleddash \{a_1,\ldots, a_{k-1}\} \subseteq I \oplus \overline{B_{k}} \circleddash A_{k}$. For $N+1 \leq k \leq 2N$, $a_k \in S \circleddash \{{t_r}_1,\ldots,{t_r}_{k-N}\} \subseteq  I \oplus \overline{B_{k}} \circleddash A_{k}$. Furthermore, by lemma \ref{thm-localizable}, lemma \ref{prop-rearrangeLocalizable} and lemma \ref{prop-recycledRobotLocalizable}, the active robots are localizable whenever necessary. By lemma \ref{thm-unobstructedPath} and lemma \ref{prop-subtractiveStageCorrectness}, the path for each moving robot is unobstructed. Theorem \ref{thm-final}.2 follows.

        % Finally, by theorem \ref{thm-additiveStageCorrect} and theorem \ref{thm-subtractiveStageCorrect}, the robot starts from lattice position $a_k$ stops at the stopping lattice position $b_k$ by executing the add-subtract algorithm. This completes the proof.
    \end{proof}

\end{appendices}
%\newpage

%display all citations
%\nocite{*}

%%Avoid print the page number
%\renewcommand{\thepage}{}
% \bibliographystyle{ieeetr}
% \bibliography{references}
\bibliographystyle{ieeetr}
\bibliography{references}
\end{document}